%% file: main.tex
\documentclass[accepted]{uai2026}

\PassOptionsToPackage{numbers, compress}{natbib}

\usepackage{natbib}  %

\usepackage[dvipsnames]{xcolor}

\definecolor{myblue}{RGB}{31, 119, 180}
\definecolor{myorange}{RGB}{255, 127, 14}
\definecolor{mygreen}{RGB}{44, 160, 44}
\definecolor{myred}{RGB}{214, 39, 40}

\usepackage{hyperref}

\usepackage[utf8]{inputenc} %
\usepackage[T1]{fontenc}    %
\usepackage{url}            %
\usepackage{booktabs}       %
\usepackage{amsfonts}       %
\usepackage{nicefrac}       %
\usepackage{microtype}      %
\usepackage{xspace}

\usepackage{subcaption}

\usepackage{tabularx}

\usepackage[noend]{algpseudocode}
\usepackage{algorithm}

\usepackage{flushend}

\usepackage{cite}
\usepackage{comment}
\usepackage{prettyref}
\usepackage{amsfonts}
\usepackage{makecell}
\usepackage{adjustbox}
\usepackage{multicol}

\usepackage{amsmath}
\usepackage{amssymb}
\usepackage{amsmath,mathtools,thmtools}
\usepackage{enumitem}

\renewcommand{\thesubfigure}{(\alph{subfigure})}

\usepackage{thm-restate}

\usepackage[textsize=footnotesize]{todonotes}
\usepackage{xspace}

\usepackage[most]{tcolorbox}
\definecolor{colorcomment}{RGB}{160, 190, 210}%
\makeatletter
\algnewcommand{\LineComment}[1]{\Statex \hskip\ALG@thistlm \(\triangleright\) 
{\color{colorcomment}#1}}
\makeatother

\makeatletter
\algnewcommand{\IndentLineComment}[1]{\Statex \hskip\ALG@tlm \(\triangleright\) {\color{colorcomment}#1}}
\makeatother

\input{macros}

\title{Advantage-Aligned Active Online Reinforcement Learning with Offline Data}

\author[1,4]{Xuefeng Liu}
\author[1]{Hung T. C. Le}
\author[2]{Siyu Chen}
\author[1]{Rick  Stevens}
\author[2]{Zhuoran Yang}
\author[3]{Matthew Walter}
\author[1]{Yuxin Chen}

\affil[1]{%
Department of Computer Science, University of Chicago, Chicago, IL, USA
}

\affil[2]{%
Department of Statistics and Data Science, Yale University, New Haven, CT, USA
}

\affil[3]{%
Toyota Technological Institute at Chicago, Chicago, IL, USA
}
\affil[4]{%
 Correspondence to: xuefeng@uchicago.edu
 }

\begin{document}

\maketitle

\input{abstract}
\input{introduction}

\input{related_works}

\input{background_problemStatement}

\input{algorithm}

\input{analysis}

\input{experiments}

\bibliographystyle{plainnat}
\bibliography{reference}
\flushcolsend %

\clearpage
\onecolumn

\appendix
\clearpage

\input{supp_additional_theory}

\end{document}

%% file: macros.tex
\newcommand\loss{\ell}

\newcommand\policy{\ensuremath{\pi}}

\newcommand\state{s}

\newcommand\stateDist{d}

\newcommand\transDynamics{\mathcal{P}}
\newcommand\reward{r}
\newcommand\RFunc{R}

\newcommand\VFunc{\ensuremath{\ensuremath{V}}}

\newcommand\QFunc{\ensuremath{\ensuremath{Q}}}
\newcommand\AFunc{\ensuremath{\ensuremath{A}}}

\newcommand\discFactor{\gamma}

\newcommand\stateSpace{\ensuremath{\mathcal{S}}}
\newcommand\action{\ensuremath{a}}

\newcommand\actionSpace{\ensuremath{\mathcal{A}}}

\newcommand\advTemperature{\ensuremath{\mathcal{\xi}}} %
\newcommand\densityTemp{\ensuremath{\mathcal{\zeta}}} %

\newcommand\horizon{\ensuremath{T}}

\newcommand\GNum{\ensuremath{G}} %

\newcommand\EnsembleNum{\ensuremath{E}} %
\newcommand\batch{\ensuremath{N}} %

\newcommand\density{\ensuremath{{w}}}

\newcommand\ReplayBuffer{\ensuremath{\mathcal{R}}}
\newcommand\OfflineBuffer{\ensuremath{\mathcal{D}}}

\newcommand{\expctover}[2]{\mathbb{E}_{#1}\!\left[#2\right]}

\newcommand{\abs}[1]{\left\vert#1\right\vert}
\def \argmax {\mathop{\rm arg\,max}}
\def \argmin {\mathop{\rm arg\,min}}
\newcommand{\kl}{{\mathrm{KL}}}

\newcommand{\algname}{\textsc{$A^3$\text{RL}}\xspace}

\newif\iffinal
\finaltrue

\iffinal
    \newcommand{\fix}[1]{#1}
    \newcommand{\YC}[1]{}
    \newcommand{\YCinline}[1]{}
    \newcommand{\XL}[1]{}
    \newcommand{\XLinline}[1]{}
    \newcommand{\MW}[1]{}
    \newcommand{\MWinline}[1]{}
    \newcommand{\note}[1]{}
    \newcommand{\pref}[1]{}
    
\else
    \newcommand{\fix}[1]{{\color{red} #1}}
    \newcommand{\YC}[1]{\todo[fancyline,color=NavyBlue!40]{YC: #1}\xspace}
    \newcommand{\YCinline}[1]{\textcolor{NavyBlue}{[YC: #1]}}
    \newcommand{\XL}[1]{\todo[fancyline,color=green!40]{XL: #1}\xspace}
    \newcommand{\XLinline}[1]{\textcolor{NavyBlue}{[XL: #1]}}
    \newcommand{\MW}[1]{\todo[fancyline,color=blue!40]{MW: #1}\xspace}
    \newcommand{\MWinline}[1]{\textcolor{blue!90}{[MW: #1]}}
    \newcommand{\note}[1]{{\color{purple}[XL: #1]}}
    
    \newcommand{\pref}[1]{{\color{blue}(\ref{#1})}}
\fi

\newcommand{\tabref}[1]{Table~\ref{#1}}
\newcommand{\figref}[1]{Fig.~\ref{#1}}

\newcommand{\secref}[1]{\S\ref{#1}}
\newcommand{\appref}[1]{Appendix~\ref{#1}}

\newcommand{\paren} [1] {\ensuremath{ \left( {#1} \right) }}

\newcommand{\bracket}[1]{\left[#1\right]}
\newcommand{\tuple}[1]{\ensuremath{\left\langle #1 \right\rangle}}
\newcommand{\curlybracket}[1]{\ensuremath{\left\{#1\right\}}}

\newenvironment{proof}{\emph{Proof:}}{\hfill$\square$}

\let\hat\widehat
\let\tilde\widetilde

\def\given{{\,|\,}}

\newcommand{\cA}{\mathcal{A}}

\newcommand{\EE}{\mathbb{E}}

\newcommand{\PP}{\mathbb{P}}

\RequirePackage{dsfont}

\DeclareMathOperator{\ind}{\mathds{1}}  %

%% file: abstract.tex
\begin{abstract}
Online reinforcement learning (RL) enhances policies through direct interactions with the environment, but faces challenges related to sample efficiency. {In contrast, offline RL leverages extensive pre-collected data to learn policies but often yields suboptimal results due to limited data coverage and redundancy within the dataset.} Recent efforts integrate offline and online RL in order to harness the advantages of both approaches. However, effectively combining online and offline RL remains challenging due to issues that include %
catastrophic forgetting, lack of robustness { to data quality} and {limited sample efficiency in data utilization}. 
{To address these challenges, we introduce \algname, which incorporates a novel confidence-aware Active Advantage-Aligned ($A^3$) sampling strategy. Our approach dynamically prioritizes samples from both online and offline data sources by aligning the sampling distribution with the current policy improvement direction, thereby enabling more efficient and effective policy optimization. Moreover, we provide theoretical insights into the effectiveness of our active sampling strategy and conduct diverse empirical experiments and ablation studies, demonstrating that our method outperforms competing online RL techniques that leverage offline data.}

\end{abstract} 

%% file: introduction.tex
\section{Introduction}\label{sec:intro}

Reinforcement learning~(RL) has achieved notable success in many domains, such as robotics~\citep{kober2011reinforcement,kober2013reinforcement},  game play~\citep{mnih2013playing,silver2017mastering}, drug discovery~\citep{liu2023drugimprover, liu2024entropy}, 
and reasoning with Large Language Models (LLMs)~\citep{havrilla2024teaching}.
Online RL algorithms such as Q-learning~\citep{watkins1989learning}, SARSA~\citep{rummery1994line}, and PPO~\citep{schulman2017proximal} learn and make decisions in an online, sequential manner, whereby an agent interacts with an environment and learns from its experience. However, due to the need for exploration that is fundamental to RL, online RL tends to be highly sample inefficient in high-dimensional or sparse reward environments. 
A complementary approach {to improve the sample efficiency} is imitation learning (IL) ~\citep{ross2010efficient,ross2014reinforcement}, where an agent learns a policy by leveraging expert demonstrations~\citep{cheng2020policy,liu2023blending, liu2023active}.

However, in many cases, we do not have access to a live expert to query, but do have access to an abundance of logged data collected from experts. One approach to make use of this data is through offline reinforcement learning. Offline RL~\citep{levine2020offline,prudencio2023survey}  learns a policy solely from such a fixed dataset of pre-collected experiences, without the need to directly interact with the environment. 
Despite its advantages, offline RL often results in a suboptimal policy due to dataset limitations. This has motivated recent work that combines offline and online RL, whereby %
learning begins from a logged dataset before transitioning to online interactions for further improvement. %
While beneficial, contemporary offline-to-online RL methods suffer from catastrophic forgetting, where previously learned knowledge is overwritten during online fine-tuning, leading to significant performance degradation~\citep{luo2023finetuning,zheng2023adaptive}.

{More recently,} methods that integrate online RL with offline datasets utilize off-policy techniques to incorporate offline data while learning online~\citep{ball2023efficient,song2022hybrid}{, mitigating catastrophic performance drops.} 
These techniques do not require any preliminary offline RL training or incorporate specific imitation clauses that prioritize pre-existing offline data. Notably, RLPD~\citep{ball2023efficient} {exhibits strong empirical performance, however}
it employs a uniform random sampling strategy for both offline and online learning, ignoring that different transitions contribute differently to {the various stages of} policy improvement.
{Furthermore, this uniform sampling strategy may lead to data inefficiencies (e.g., sampling unhelpful transitions while missing out on valuable ones) and makes policy improvement highly sensitive to data quality.
}

\begin{figure}[t!]
    \centering
    \includegraphics[width=1\linewidth]{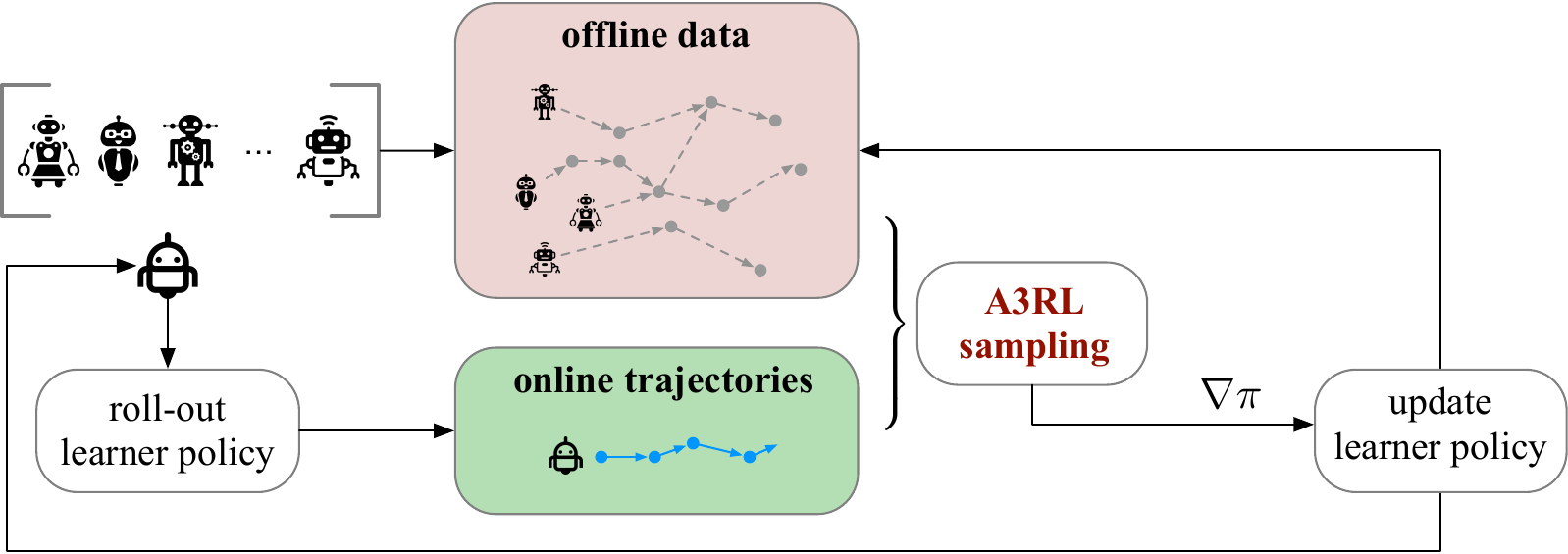}
    \vspace{-3mm}
    \caption{\algname combines online and offline RL using a priority-based sampling strategy that prioritizes samples from online roll-outs and offline data that align with directions for policy improvement.}\label{fig:main:method_overview}
    \vspace{-3mm}
\end{figure}

{\bf Our contributions.} 
In this work, we introduce \underline{A}ctive \underline{A}dvantage-\underline{A}ligned  \underline{R}einforcement \underline{L}earning (\algname),
a novel method that operates in the realm of online RL with an offline dataset, {as illustrated in \figref{fig:main:method_overview}}. 
{
Our approach dynamically prioritizes transitions with the highest potential to drive policy improvement, while also tackling the mixed offline-online nature by aligning the sampling distribution to the policy's data generation.
}
{More specifically, \algname considers not only the relevance of the data in facilitating the current policy's online exploration and exploitation, but also its estimated contribution to policy improvement via {confidence-aware} advantage-based prioritization.}
{\algname demonstrates robustness to data quality in a black-box manner and maintains resilience under varying environmental conditions. Notably, it also effectively accelerates policy improvement, even in a purely online environment.}

In short, our contributions are:
\begin{itemize}[noitemsep, wide,labelwidth=0pt, labelindent=0pt]
    \item We propose \algname, a novel algorithm for online RL with offline data. This algorithm surpasses current state-of-the-art (SOTA) methods through a priority-based sampling strategy integrating a conservative estimate of the advantage function %
    with the extent of online coverage of the offline dataset.
    \item {In contrast to RLPD {and other related works~\citep{lee2022offline, schaul2015prioritized}}, which lack theoretical support, 
    this study provides theoretical insights into our active advantage-aligned sampling strategy, demonstrating superiority and its minimum improvement gap over random sampling.}
    \item Through extensive empirical evaluations on the D4RL benchmark~\citep{fu2020d4rl}, we demonstrate that \algname achieves %
    consistent and significant improvements over prior SOTA models, especially on harder Adroit tasks.
    \item Given the black-box nature of offline datasets, we conduct comprehensive ablation studies across a range of dataset qualities and environmental settings, including the purely online scenario,  to evaluate the robustness of \algname. {These studies consistently confirm its stable performance across diverse conditions, regardless of environmental factors or data quality.}
\end{itemize}

%% file: related_works.tex
\section{Related Work}\label{sec:related}
\vspace{-1em}
{\bf Online RL with offline datasets.} Several methods exist that incorporate offline datasets in online RL to enhance sample efficiency. %
Many rely on high-quality expert demonstrations~\citep{ijspeert2002learning, kim2013learning, nair2018overcoming, rajeswaran2017learning, vecerik2017leveraging, zhu2019dexterous}. 
\citet{nair2020awac} introduced the Advantage Weighted Actor Critic (AWAC), which utilizes regulated policy updates to maintain the policy's proximity to the observed data during both offline and online phases. On the other hand, \citet{lee2022offline} proposed an initially pessimistic approach to avoid over-optimism and bootstrap errors in the early online phase, gradually reducing the level of pessimism as more online data become available.
Most relevant to our work is RLPD~\citep{ball2023efficient}, which %
adopts a sample-efficient off-policy approach to learning that does not require pre-training. 
Unlike RLPD, which utilizes symmetric sampling to {randomly} draw from both online and offline datasets for policy improvement, \algname adopts a Prioritized Experience Replay \citep{schaul2015prioritized} (PER)-style  method, whereby it selectively uses data from both datasets to enhance policy performance.

{\bf Prioritized experience replay.}
Experience replay~\citep{lin1992self} enhances data efficiency in online RL by reusing past experiences.
PER ~\citep{schaul2015prioritized} introduces prioritization based on temporal difference (TD) error to ensure that miscalibrated $Q$-values are corrected more frequently, and has proven effective in a variety of settings~\citep{hessel2018rainbow, jaderberg2016reinforcement, nair2018overcoming, oh2021model, saglam2022actor, tian2023learning, wang2016dueling, yue2023offline}. 
Alternative prioritization strategies have been explored, such as prioritizing transitions based on expected returns~\citep{isele2018selective} or adjusting sample importance based on recency~\citep{fedus2020revisiting}. 
{\citet{eysenbach2020off} apply a density ratio to the reward instead of weighting the samples.}
These works predominantly focus on either purely online or purely offline applications of PER. In contrast, our approach uniquely integrates the strengths of both online and offline data.
The most relevant studies to ours include \citet{sinha2022experience} that uses the density ratio between off-policy and near-on-policy state-action distributions as an importance weight for policy evaluation, and \citet{lee2022offline} that employs density ratios to select relevant, on-policy–like samples from offline datasets. 
However, existing methods fail to account for a critical factor—the potential of data (e.g., its advantageousness and informativeness) that fundamentally shapes policy improvement.

{\bf Active learning in RL.}
Active learning has been explored in RL for data-efficient exploration %
\citep{epshteyn2008active, fang2017learning, krueger2020active,liu2022cost,liu2023blending, liu2023active,lopes2009active}. %
Unlike previous approaches that focus on %
oracle selection~\citep{liu2023blending, liu2023active}, state exploration~\citep{epshteyn2008active,liu2023active} or reward estimation~\citep{lopes2009active},
\algname introduces an active sampling mechanism tailored to online RL with offline data, prioritizing transitions that maximize policy improvement. We defer more details of related work to the 
\appref{app:relatedWork}.

%% file: background_problemStatement.tex
\section{Preliminaries and Problem Statement}\label{sec:Preliminaries}

We consider a discounted Markov decision process (MDP) environment~\citep{bellman1957markovian} characterized by a tuple $M = (\stateSpace, \actionSpace, \transDynamics, \RFunc, \discFactor, \stateDist_0)$, where $\stateSpace$ represents a potentially infinite state space, $\actionSpace$ is the action space,
$\transDynamics: \stateSpace \times \actionSpace \rightarrow \Delta(\stateSpace)$ is the unknown transition kernel, %
$\RFunc: \stateSpace \times \actionSpace \rightarrow \mathbb{R}$ is the reward function, $\discFactor \in (0,1)$ is the discount factor and $\stateDist_0\paren{\state}$ is the initial state distribution. 
The learner's objective is to solve for the policy $\policy:\stateSpace \rightarrow \Delta(\actionSpace)$ that maximizes the %
expected sum of discounted future rewards $\expctover{\policy}{\sum_{t=1}^{\infty} \gamma^t\reward\paren{\state_t,\action_t}}$,
where the expectation is taken over the trajectory sampled from $\policy$.

{\bf Maximum entropy RL.} 
In this work, we adopt off-policy soft actor-critic (SAC)~\citep{haarnoja2018soft} RL to train an agent with samples generated by any behavior policy. 
{
We use a general maximum entropy objective~\citep{ball2023efficient,haarnoja2018soft,ziebart2010modeling} as follows:
\begin{equation}\label{eq:maxentRL}
    \max_{\policy} \expctover{\state\sim{\ReplayBuffer},{\action\sim\policy}}{\sum_{t=0}^{\infty} \gamma^t\paren{\reward_t+\alpha \mathcal{H}\paren{\policy\paren{\action|\state}}}},
\end{equation}
where {$\ReplayBuffer$ is replay buffer and $\alpha$ is a temperature parameter.} This involves optimizing reward while encouraging exploration, making the learned policy more robust. 
}

\textbf{$\QFunc$-value and advantage function.} 
The $\QFunc$-value function measures the expected return of executing action $\action$ in state $\state$ under policy $\policy$, and satisfies the Bellman self-consistency condition:
$\QFunc^{\policy}\paren{\state,\action} = \mathcal{B}^{\policy}\QFunc^{\policy}\paren{\state,\action}$, 
where $\mathcal{B}^{\policy}$ is the Bellman operator:
    $\mathcal{B}^{\policy}\QFunc\paren{\state,\action} := \reward\paren{\state,\action} + \gamma \expctover{\state' \sim P\paren{\cdot|\state,\action}}{\VFunc^{\policy}\paren{\state'}}$.
The soft state value function is defined as:
$\VFunc^{\policy}\paren{\state}:= \expctover{\action\sim\policy\paren{\cdot|\state}}{\QFunc^{\policy}\paren{\state,\action}- {\log \policy\paren{\action|\state}}}$. 
For a policy $\policy$, the advantage function \citep{sutton1999policy} %
quantifies the relative benefit of selecting $\action$ over the policy's default behavior:
\begin{align}\label{eq:advantage}
   \AFunc^{\policy}\paren{\state,\action} = \QFunc^{\policy}\paren{\state,\action} - \VFunc^{\policy}\paren{\state}.
\end{align}
Specifically, SAC learns a soft Q-Function, denoted as $\QFunc_{\theta}\paren{\state,\action}$ with parameters $\theta$, 
and a stochastic policy $\policy_{\phi}$ parameterized by $\phi$. The SAC method involves alternating between updates for the critic and the actor by minimizing their respective objectives~\citep{lee2022offline} as follows
\begin{align*}
&\mathcal{L}_{\text{critic}}^{\text{SAC}}\paren{\theta}=\mathbb{E}_{\paren{\state_t,\action_t,\state_{t+1}}\sim\ReplayBuffer}[(\QFunc_{\theta}\paren{\state_t,\action_t}-\reward\paren{\state_t,\action_t}\\
&-\gamma{
    \expctover{\action'_{t+1} \sim \policy_{\phi}}{\QFunc_{\theta'}\paren{\state_{t+1},\action'_{t+1}}-\alpha\log\policy_{\phi}\paren{\action'_{t+1}|\state_{t+1}}}})^2], \\
&\mathcal{L}_{\text{actor}}^{\text{SAC}}\paren{\phi}=\expctover{\state_t\sim\ReplayBuffer,\action_t\sim\policy_{\phi}}{\alpha\log\policy_{\phi}\paren{\action_t|\state_t}-\QFunc_{\theta}\paren{\state_t,\action_t}},
\end{align*}
Here, $\ReplayBuffer$
is an experience replay buffer of either on-policy experience~\citep{sutton1999policy} or through off-policy experience~\citep{munos2016safe, precup2000eligibility}, 
and $\overline{\theta}$ denotes the delayed target parameters.

\textbf{Prioritized experience replay.}
PER~\citep{schaul2015prioritized} serves as the basis of our sampling techniques, providing a framework for prioritizing experience replay based on transition importance. Instead of sampling uniformly from the replay buffer $\ReplayBuffer$,  PER assigns higher probability to transitions with miscalibrated $Q$-values, leading to improved sample efficiency~\citep{hessel2018rainbow} as these transitions are seen and corrected more frequently. Each transition $\ReplayBuffer_i = (\state_{i},\action_{i},\reward_{i},\state_{i+1})$ is assigned a priority $P_i$, typically based on the absolute of the TD-error $\delta=\reward+\gamma Q_{{\theta'}}(s_{t+1}, a'_{t+1}) - Q_\theta(s_t, a_t)$ where $a'_{t+1} \sim \pi(\cdot | s_{t+1})$ \citep{brittain2019prioritized,hessel2018rainbow,oh2021model, schaul2015prioritized, van2019use}.
Subsequently, the sampling approach of PER involves sampling an index set $\mathcal{I}$ within the range $\bracket{|\ReplayBuffer|}$ based on the probabilities $p_i$ assigned by the priority set as follows:
    $p_i=\frac{P_i^{\alpha}}{\sum_{k\in\bracket{|\ReplayBuffer|}}P_k^{\alpha}}$, 
with a hyperparameter $\alpha>0$.
To correct for sampling bias, PER applies importance sampling weights:
\begin{equation}
\label{eq:weights}
u_i \propto \bigl ( {1}/ (|\ReplayBuffer|\cdot p_i) \bigr)^{\beta_{priority}},
\end{equation}
where $\beta_{priority}$ anneals from $\beta_0 \in (0,1)$ to $1$ during training to counteract bias in the learning updates, and the importance sampling weights are normalized to have maximum weight of $1$ for stability. While standard PER prioritizes TD-error, our method extends this framework to prioritize transitions based on onlineness and estimated advantage.

\textbf{Online RL with offline datasets.}
In this work, we study online RL, with offline datasets denoted as $\OfflineBuffer$ \citep{ball2023efficient}. These datasets consist of a set of tuples $(\state, \action, \reward, \state')$ generated from a specific MDP. 
A key characteristic of offline datasets is that they typically offer only partial coverage of state-action pairs. In other words, the set of states and actions in the dataset, denoted as {$\curlybracket{\paren{\state,\action}\in \OfflineBuffer}$}, represents a limited subset of the entire state space and action space, $\stateSpace \times \actionSpace$. 
Moreover, learning on the  data with incomplete coverage of state-action pairs potentially results in excessive value extrapolation during the learning process for methods using function approximation~\citep{fujimoto2019off}.
Our model, based on SAC \citep{haarnoja2018soft}, incorporates several effective strategies for RL with offline data, as outlined in RLPD \citep{ball2023efficient}, with the most crucial one being \emph{Clipped Double Q-Learning}. The maximization objective of Q-learning and the estimation uncertainty from value-based function approximation often  leads to value overestimation \citep{van2016deep}. To address this problem, \citet{fujimoto2018addressing} introduced Clipped Double Q-Learning (CDQ) as a means of mitigation. CDQ involves taking the minimum from an ensemble of two Q-functions for computing TD-backups. The targets for updating the critics are given by the equation $y=\reward\paren{\state,\action} + \gamma \min_{i=1,2} \QFunc_{\theta'_i}\paren{\state',\action'}$, where $\action'\sim \policy\paren{\cdot|\state'}$. See \appref{app:preliminaries} for more strategies to stabilize training.

%% file: algorithm.tex
\vspace{-0.3cm}
\section{Algorithm}
Here, we describe our active advantage-aligned sampling algorithm that is theoretically grounded in the performance difference lemma \secref{sec:theoretical-analysis}. At a high level, we do \textit{active} learning in that we select relevant offline transitions that align better with the data distribution generated by the current policy through a \textit{density} term, while also incorporating the \textit{advantage}, which measures the potential impact of the transition on policy improvement.
\vspace{-0.3cm}
\paragraph{Active density term.} To evaluate the onlineness/on-policyness of an offline transition, we estimate the density ratio 
\begin{align}\label{eq:densityratio}
\density\paren{\state,\action}:= {\stateDist^{\text{on}}\paren{\state,\action}}/ {\stateDist^{\text{off}}\paren{\state,\action}}
\end{align}
where $\stateDist^{\text{on}}\paren{\state,\action}$ denotes the state-action density of online samples in the online buffer $\ReplayBuffer$ and the $\stateDist^{\text{off}}\paren{\state,\action}$ denotes that of the offline dataset $\mathcal{D}$. Then, by identifying a transition with a high density ratio $\density\paren{\state,\action}$, we can effectively select a near-on-policy transition $\paren{\state,\action,\state'}$ from the offline dataset $\mathcal{D}$. 
This directly leverages the abundance of the offline data compared to online data, greatly improving the amount of transitions and diversity of coverage for policy improvement in each step.

Estimating the likelihoods $\stateDist^{\text{off}}\paren{\state,\action}$ and $\stateDist^{\text{on}}\paren{\state,\action}$ poses a challenge, as they could represent stationary distributions from mixture of complex policy.
To address this issue, we employ a method studied by \citet{lee2022offline} and \citet{sinha2022experience} for density ratio estimation that does not rely on likelihoods, namely by approximating $\density\paren{\state,\action}$ with a neural network $\density_{\psi}\paren{\state,\action}$.
To this end, we use variational representation of f-divergences~\citep{nguyen2007estimating} in the following way. Let $P$ and $Q$ be probability measures on a measurable space $\mathcal{X}$, with $P$ being absolutely continuous w.r.t. $Q$. Let $f\paren{y}:=y \log\frac{2y}{y+1}+\log\frac{2}{y+1}$, then the Jensen-Shannon (JS) divergence between $P$ and $Q$ can be written as $D_\textrm{JS}\paren{P||Q} = \int_{\mathcal{X}}f\paren{ (dP/dQ)(x)}dQ\paren{x}$. It then has the variational lower bound 
\begin{equation}\label{eq:lower_bound}
    \mathcal{L}^{\text{DR}}\paren{\psi}= \expctover{x\sim P}{f'\paren{\density_{\psi}\paren{x}}}- \expctover{x\sim \QFunc}{f^*\paren{f'\paren{\density_{\psi}\paren{x}}}},
\end{equation}
where $f^*$ denotes the convex conjugate of $f$, that is maximized at $w_\psi(x) = (dP/dQ)(x)$. This enables us to use a parametric model $\density_{\psi}\paren{x}$ to approximate the density ratio $\frac{dP}{dQ}$ via maximizing this $\mathcal{L}^{DR}$ lower bound of $D_\textrm{JS}\paren{P||Q}$.

\begin{algorithm*}[!t]
\caption{\algname}
\label{alg:a3rl}
\begin{algorithmic}[1]
\State Select ensemble size $\EnsembleNum$, UTD ratio $\GNum$, discount $\gamma$,  temperature $\alpha$, critic polyak weight $\rho$, batch size $\batch$.
\State Randomly initialize critic $\left\{\theta_i, i\in \bracket{\EnsembleNum}\right\}$ and actor $\phi$ parameters. Set targets $\theta'_i =\theta_i$.
\State Initialize offline priority buffer $\mathcal{B}^{\text{off}}$ with offline dataset $\mathcal{D}$, {online priority buffer $\mathcal{B}^{\text{on}} = \ReplayBuffer \leftarrow \emptyset$ }.
     \For{each environment step $t$}
       \State Take action $\action_t \sim \policy_{\phi}\paren{\cdot|\state_t}$, update online buffer $\mathcal{B}^{\text{on}} \leftarrow \mathcal{B}^{\text{on}} \cup \curlybracket{\paren{\state_t,\action_t,\reward_t,\state_{t+1}}}$.
       
       \For{$g=1, \ldots, G$}
            \State Uniformly sample $\frac{N}{2}$ transitions from $\mathcal{B}^{\text{off}}$ and similarly $\frac{N}{2}$ from $\mathcal{B}^{\text{on}}$. \textcolor{NavyBlue}{Update $w_\psi$ to maximize \eqref{eq:lower_bound}}.
            \State \textcolor{NavyBlue}{Sample with priority $\frac{N}{2}$ transitions from $\mathcal{B}^{\text{off}}$ and similarly $\frac{N}{2}$ from $\mathcal{B}^{on}$ to form batch $b$ of size $N$.}
            \State 

            \State With $b$, randomly subsample $\mathcal{Z} \subset [E]$ with $\abs{\mathcal{Z}} = 2$, set 
            \begin{equation*}
            y=\reward+\discFactor\Big(\min_{i\in \mathcal{Z}} \QFunc_{\theta'_i}\paren{\state',\action'} + \alpha \log \policy_{\phi}\paren{\action'|\state'}\Big), \quad \action'\sim \policy_{\phi}\paren{\cdot|\state'}.
             \end{equation*}

            \State \textcolor{Maroon}{Within each buffer, calculate importance weights $u$ via \eqref{eq:weights}.}
            \State \textcolor{Maroon}{Scale weights so that sum of offline weights = sum of online weights = $1/2$.}
            \State Perform SAC critic updates with target $y$ and target critic polyak updates.
       \EndFor
       \State Perform SAC actor updates and \textcolor{NavyBlue}{update priorities for $b$ according to \eqref{eq:adv_priority}.}
     \EndFor

\end{algorithmic}
\end{algorithm*}

In practice, since $\density_{\psi}\paren{x}$ approximates $\frac{dP}{dQ} \geq 0$, it has to be constrained to be non-negative (through the use of appropriate activation functions). Furthermore, since we're estimating $d^{\text{on}}/{d^{\text{off}}}$, for $\mathcal{L}^{\text{DR}}$ we sample from $\ReplayBuffer$ for $x \sim P$ and from $\mathcal{D}$ for $x \sim Q$.
To get more stable density ratio estimates $\underline{w}$, similar to \citet{lee2022offline}, we self-normalize estimated density ratio $w$ within $\mathcal{D}$ with a temperature hyperparameter $\zeta$:
\begin{equation*}
    \underline{w}_\psi(x) = {w_\psi(x)^\zeta}/{\mathbb{E}_{y \sim \mathcal{D}}\left[w_\psi(y)^\zeta\right]}.
\end{equation*}
\paragraph{Advantage term.}
{For the advantage term, to enhance robustness, we use the pessimistic CDQ $Q$ estimation through random subsampling of an ensemble $\left\{Q_{\theta_i}, Q_{\theta'_i}: i \in [E]\right\}$ of $Q$ and target $Q$ networks where $E$ is the ensemble size, similar to \citet{ball2023efficient}, while also incorporating uncertainty estimation for the advantage in the following way. Since $\AFunc^\policy(s, a) = \QFunc^\policy(s, a) - \VFunc^\policy(s) = \AFunc^\policy(s, a) = \QFunc^\policy(s, a) - \mathbb{E}_{a' \sim \policy(s)} \left[\QFunc^\policy(s, a')\right]$, we can get $E$ noisy estimates of the advantage by simply taking $A_i(s, a) = Q_{\theta_i}(s, a) - Q_{\theta_i}(s, a')$ where $a' \sim \policy$. The estimated advantage we then use is the advantage lower confidence bound (LCB) $\underline{A}(s, a)$, calculated as
\begin{equation*}
    \underline{A}(s, a) = \hat{A}(s, a) - \beta \hat{\sigma}(s, a),
\end{equation*}
where $\hat{A}(s, a)$ and $\hat{\sigma}(s, a)$ are respectively the sample mean and standard deviation of $\left\{A_i(s, a): i \in [E]\right\}$, and $\beta$ is a hyperparameter confidence level derived from the Student's $t$-distribution. The immediate upside of this calculation is that quantities like $Q_{\theta_i}(s, a)$ and $Q_{\theta_i}(s, a')$ are already computed during regular SAC updates, making this calculation computationally inexpensive.
\vspace{-0.3cm}
\paragraph{{Confidence-aware} active advantage-aligned sampling.} Here, we make the argument that relying solely on either of the two aforementioned terms would yield inevitable insufficiencies. Relying solely on the density ratio is insufficient; even if a transition appears to be relevant in the online context, it may still fail to contribute meaningfully to policy improvement, e.g., it might simply represent a "bad" transition in the value sense. Consider a transition $\paren{\state,\action,\state'}$. If the policy has previously encountered this state and taken the same action, or if the action performed in this state could potentially lead to a negative reward, such a transition would not be that helpful in contributing to policy improvement, regardless of how closely it aligns with on-policy data. On the other hand, as lengthily discussed in prior works (e.g. \citet{lee2022offline}), not addressing the offline-online distributional shift could also yield catastrophic results.

Therefore, we incorporate both estimated advantage and density ratio into our sampling strategy, ensuring that transitions are selected not only based on relevance, but also on their contribution to policy improvement. Concretely, this is done via setting a transition's priority to 
\begin{equation*}
p(s, a) = \begin{cases}
    \underline{\density}(s, a) \exp(\xi \underline{A}(s, a)) &\text{ if } (s, a) \in \mathcal{D} \\
    \exp(\xi \underline{A}(s, a)) &\text{ if } (s, a) \in \mathcal{R}
\end{cases},
\end{equation*}
or more succinctly, 
\begin{equation*}
p(s, a) = (\mathbb{I}^{\text{off}}\underline{\density}(s, a) + \mathbb{I}^{\text{on}})\exp(\xi \underline{A}(s, a)),
\end{equation*}
where $\mathbb{I}^{\text{off}}$ and $\mathbb{I}^{\text{on}}$ are indicator functions of the offline dataset and online buffer respectively.

First, we use the estimated advantage through the exponential term $\exp(\xi \underline{A}(s, a))$. The natural consequence is that the more advantage a transition has (and thus the greater its impact on learning is), the more priority it is assigned, making our sampling mechanism both adaptive and optimization-aware. Since there may be uncertainty and significant variance in estimating the advantage, we adopt the LCB \citep{liu2023blending} a conservative estimate, as previously detailed.

Second, we assign an additional weight of $\underline{w}(s, a)$ for transitions from the offline dataset. In other words, we prioritize offline samples based on both its "onlineness" and advantage value to retrieve near-on-policy samples that also provide policy improvement benefits. For the online transitions, they are themselves samples from the online data distribution, and thus we prioritize them purely based on advantage values under the current policy.
}

Third, we found that the implementation of this prioritized sampling scheme is most straightforward when done as a variation of RLPD \citep{ball2023efficient}, where we similarly give equal total weights to the offline dataset and the online replay buffer, but distribute them non-uniformly. This reaps the benefits of RLPD over previous offline-to-online RL methods, namely avoiding the phase separation between offline pretraining and online finetuning, and directly utilizing the offline dataset in conjunction with online interactions instead.

Last but not least, this advantage-aligned sampling strategy is not a heuristic-based approach. It is theoretically derived from the performance difference lemma~\citep{kakade2002approximately}, 
{which provides insights into its effectiveness and superiority over the random sampling approach. For this, see section \secref{sec:theoretical-analysis}.}

In summary, our approach broadens the distribution of samples used during online updates, centering around on-policy examples, thereby facilitating immediate value. The active advantage-aligned priority is as follows:
\begin{subequations} \label{eq:adv_priority}
    \begin{align}
    p(s, a) &= (\mathbb{I}^{\text{off}}\underline{\density}(s, a) + \mathbb{I}^{\text{on}})\exp(\xi \underline{A}(s, a)),\\
        \underline{A}(s, a) &= \hat{A}(s, a) - \beta \hat{\sigma}(s, a), \\
        \underline{w}_\psi(s, a) &= \frac{w_\psi(s, a)^\zeta}{\mathbb{E}_{(s', a') \sim \mathcal{D}}\left[w_\psi(s', a')^\zeta\right]}, 
    \end{align}
\end{subequations}

with hyperparameter $\beta$ associated with the LCB advantage estimated, $\zeta$ as the density self-normalization temperature and $\xi$ as the advantage temperature. Note that the standard PER temperature can be absorbed into these hyperparameters, and thus there is no need to define a separate one.

This active sampling process in our algorithm is highlighted in \textcolor{NavyBlue}{blue} in Algorithm~\ref{alg:a3rl}, while our approach to addressing sampling bias is highlighted in \textcolor{Maroon}{red}. We defer the details of the algorithm to \appref{app:alg:a3rl_details}.

%% file: analysis.tex
\section{Theoretical Analysis}\label{sec:theoretical-analysis}
In this section, we derive the priority term theoretically from the performance difference lemma~\citep{kakade2002approximately} and show that our active advantage-aligned sampling strategy leads to improved policy performance. Furthermore, we establish a theoretical lower bound on the performance improvement gap under our sampling scheme.
\begin{restatable}[]{thm}{perfdiff}\label{thm:perfdiff} 
Suppose the Q-function class is uniformly bounded, and for any Q-function, the corresponding optimal policy lies within the policy function class.  
Let $\epsilon^t$ denote the $\ell_2$ error of the Q-function in the critic update step.  
Let $\policy^t$ be the policy at iteration $t$ in \algname, updated using priority-weighted sampling with $w(\state,\action)\exp(\xi\cdot\AFunc(\state,\action))$. Then, the following lower bound holds:
\begin{equation*}
    J_{\alpha}^{\pi^{t+1}} - J_{\alpha}^{\pi^t} \geq J_{\alpha}^{\pi^\star} - J_{\alpha}^{\pi^t} - C \sqrt{\epsilon^t} \sup_{s, a} \left| R^t(s, a;\xi) \right|,
\end{equation*}
where $J_\alpha^\pi = \mathbb{E}_{\state\sim\rho^{\policy},\action\sim\policy} \left[\sum_{t=0}^{\infty} \gamma^t \left(\reward_t+\alpha \mathcal{H}(\policy(\action|\state))\right) \right]$ is the entropy-regularized objective,  $J_{\alpha}^{\pi^\star} - J_{\alpha}^{\pi^t}$ represents the maximum possible improvement if the true Q-function were known, and the function $R^t(s, a;\xi)$ is given by:
\begin{align*}
    R^t(s, a;\xi) &= \left(\frac{{\pi^{t+1}}(a\given s)}{d^{\text{on}}(a\given s)}\right)^{1-\xi}\\ 
    &\cdot  
    \frac{\sum_{s', a'}d^{\text{on}}(a', s') \pi^{t+1}(a'\given s')^\xi}{d^{\text{on}}(a\given s)^\xi} \cdot \frac{d^{\pi^{t+1}}(s)}{d^{\text{on}}(s)}.
\end{align*}
\end{restatable}
The proof is provided in the 
Appendix~\ref{app:theory_proof}.
We note that the coefficient $R^t(s, a;\xi)$ is not a  %
tight bound, since it is based on the supremum norm and therefore can be dominated by a single $(s, a)$ pair. A sharper result could be obtained by measuring distribution shift in the $\ell_2$ norm (or some other weaker norm). We nevertheless adopt the simpler supremum‐norm bound here for clarity and to highlight the core intuition behind why advantage reweighting yields improvement, as will be detailed in the following.

\renewcommand{\thesubfigure}{\alph{subfigure}}  %
\begin{figure*}[t!]
   \centering
    \begin{subfigure}{.24\textwidth}
        \centering
        \includegraphics[width=\textwidth,  clip={0,0,0,0}]{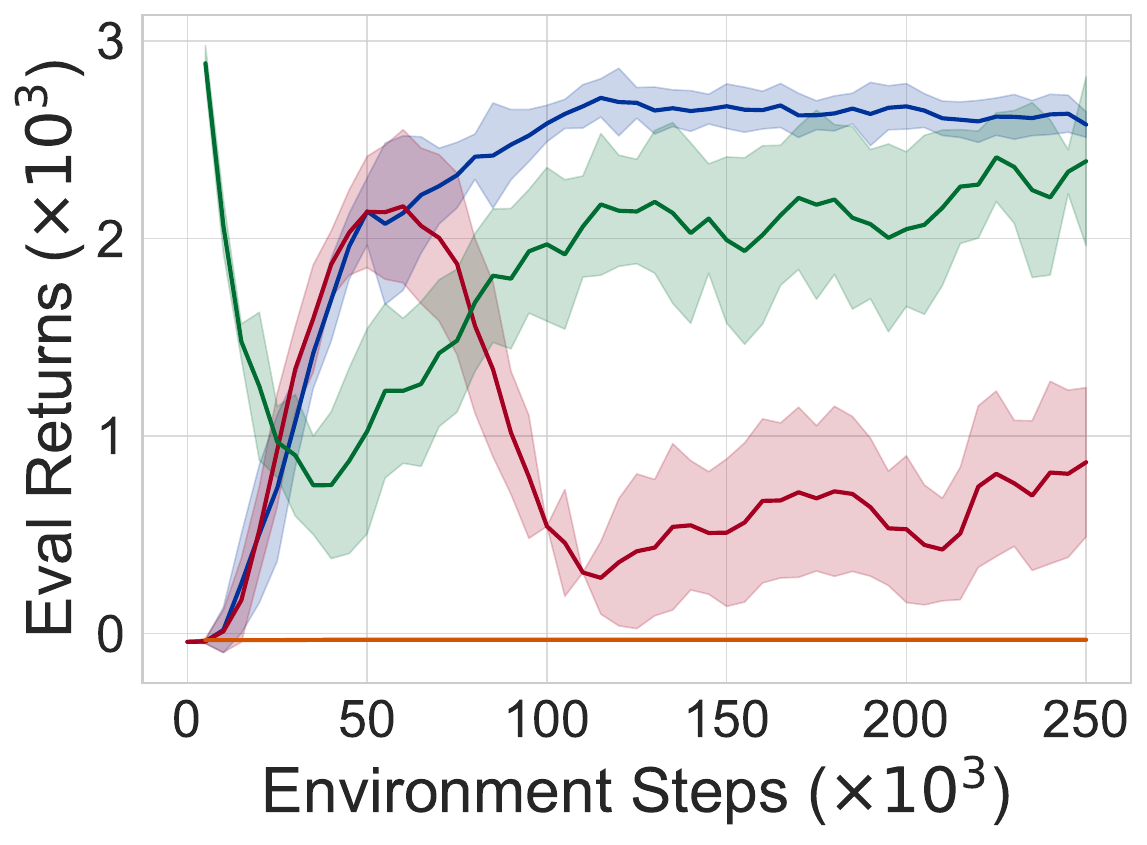}
        \caption{door-expert} \label{fig:exp:main:door-expert}
    \end{subfigure}\hfil
    \begin{subfigure}{.24\textwidth}
        \centering
        \includegraphics[width=\textwidth,  clip={0,0,0,0}]{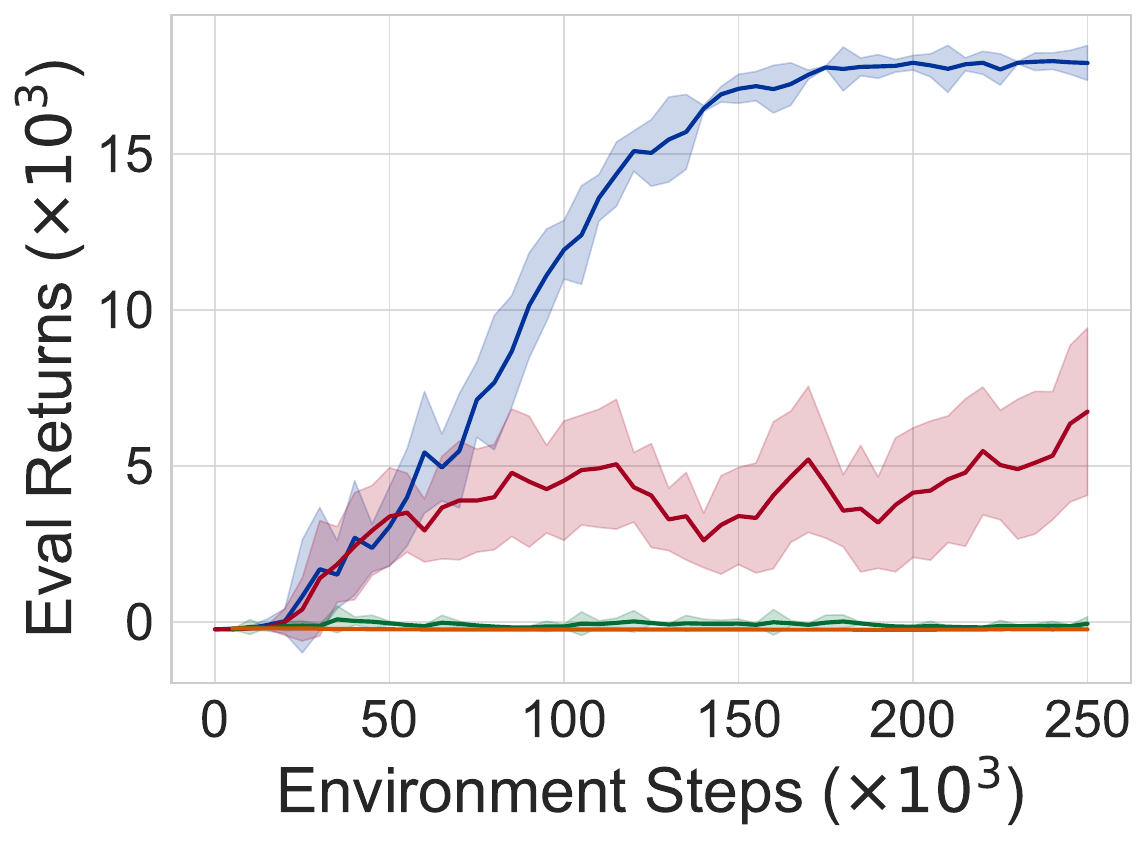}
        \caption{hammer-cloned} \label{fig:exp:main:hammer-cloned}
    \end{subfigure}\hfil
    \begin{subfigure}{.24\textwidth}
        \centering
        \includegraphics[width=\textwidth,  clip={0,0,0,0}]{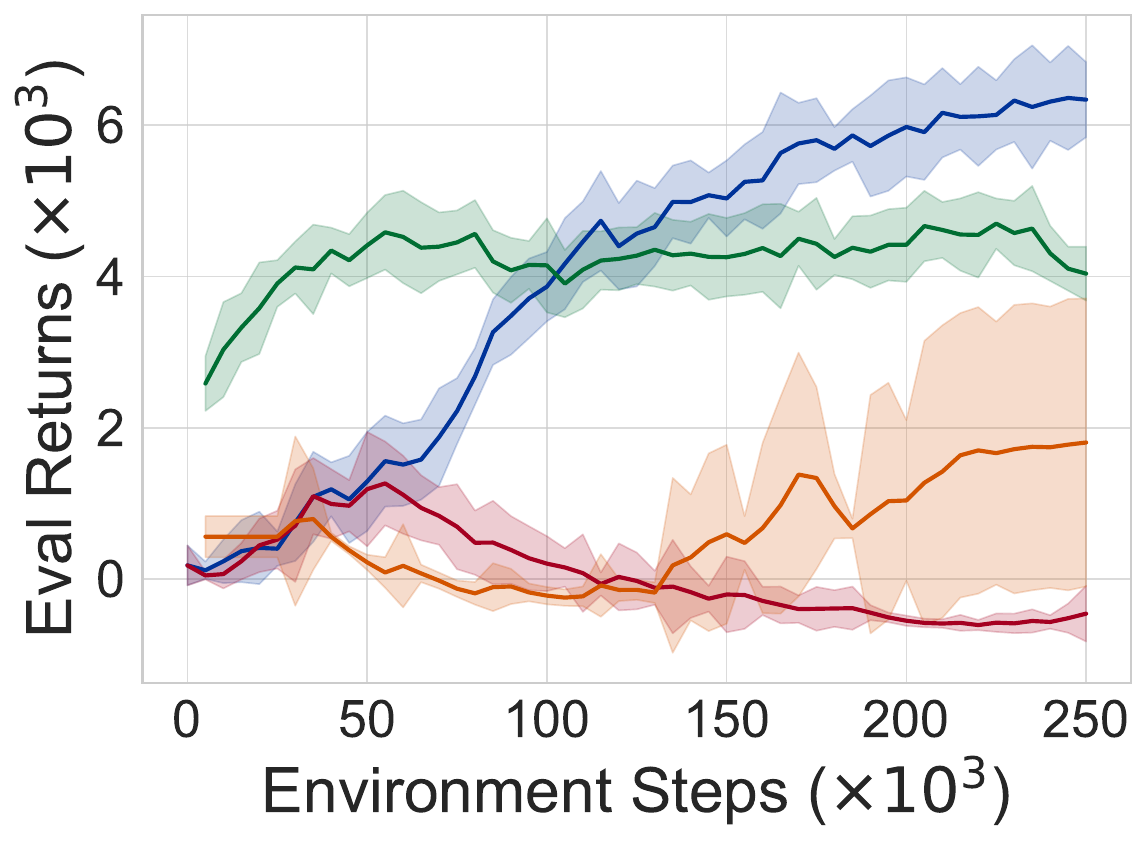}
        \caption{pen-cloned} \label{fig:exp:main:pen-cloned}
    \end{subfigure}\hfil
    \begin{subfigure}{.24\textwidth}
        \centering
        \includegraphics[width=\textwidth,  clip={0,0,0,0}]{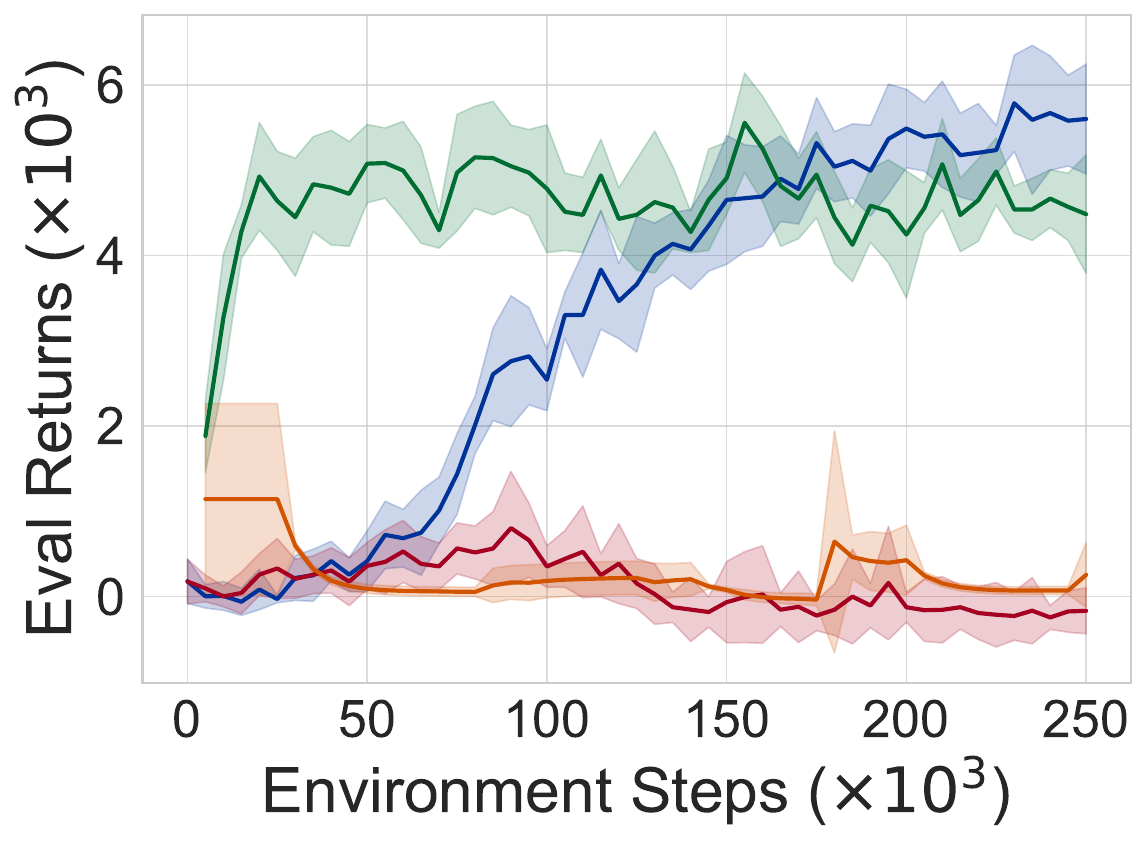}
        \caption{pen-human} \label{fig:exp:main:pen-human}
    \end{subfigure}\hfil
    \begin{subfigure}{.24\textwidth}
        \centering
        \includegraphics[width=\textwidth,  clip={0,0,0,0}]{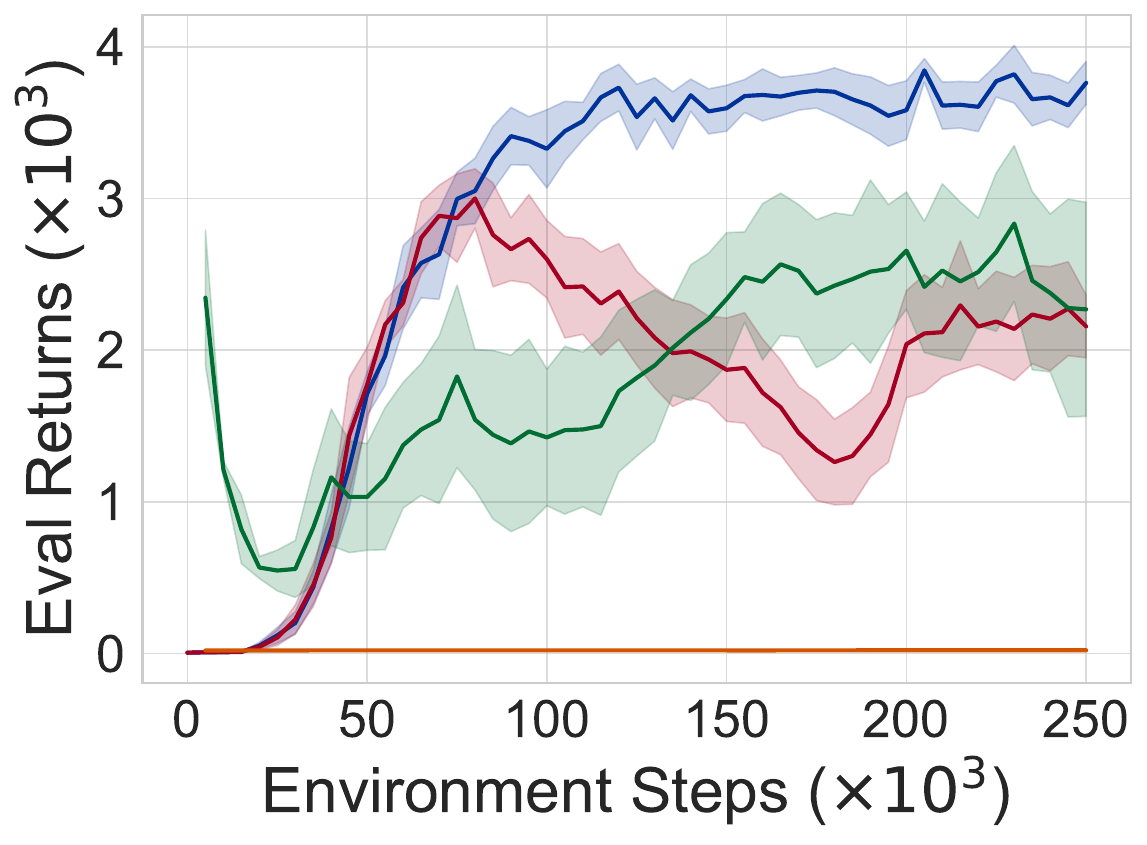}
        \caption{relocate-expert} \label{fig:exp:main:relocate-expert}
    \end{subfigure}\hfil
    \begin{subfigure}{.24\textwidth}
        \centering
        \includegraphics[width=\textwidth,  clip={0,0,0,0}]{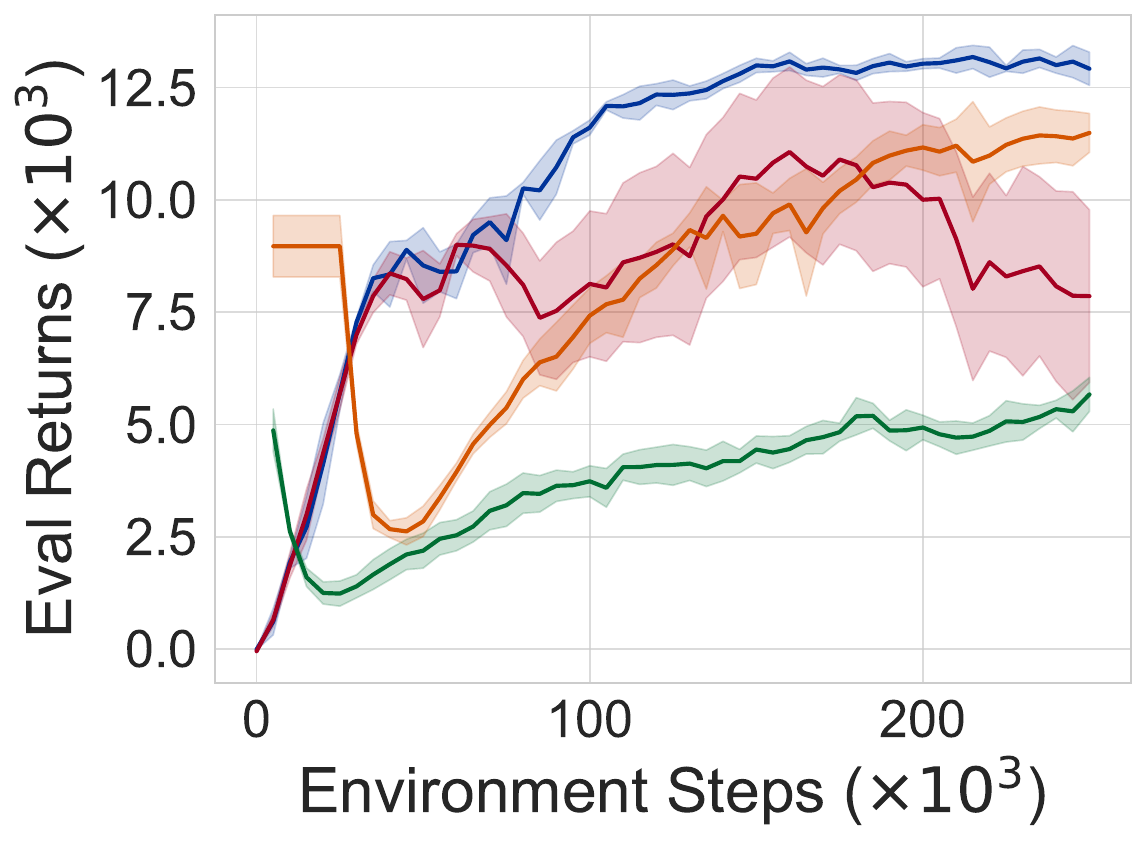}
        \caption{halfcheetah-medium} \label{fig:exp:main:halfcheetah-medium}
    \end{subfigure}\hfil
    \begin{subfigure}{.24\textwidth}
        \centering
        \includegraphics[width=\textwidth,  clip={0,0,0,0}]{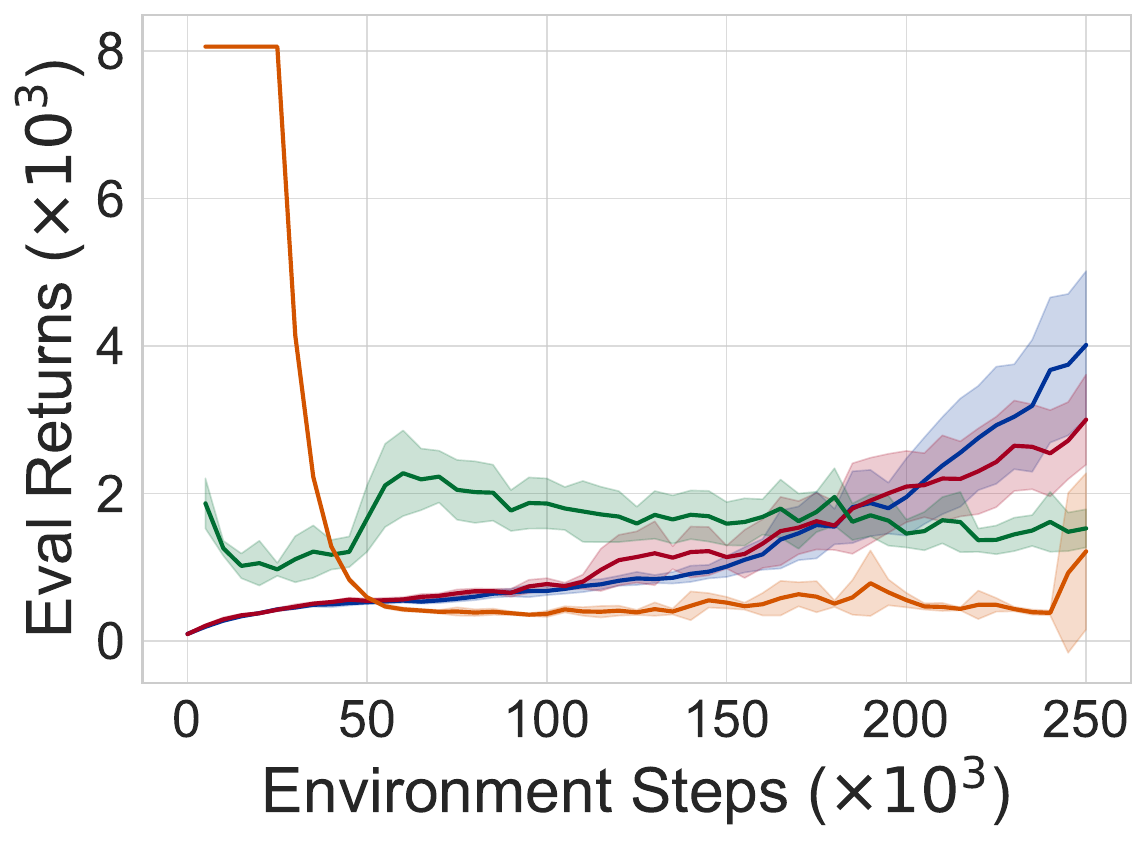}
        \caption{humanoid-medium} \label{fig:exp:main:humanoid-medium}
    \end{subfigure}\hfil
    \begin{subfigure}{.24\textwidth}
        \centering
        \includegraphics[width=\textwidth,  clip={0,0,0,0}]{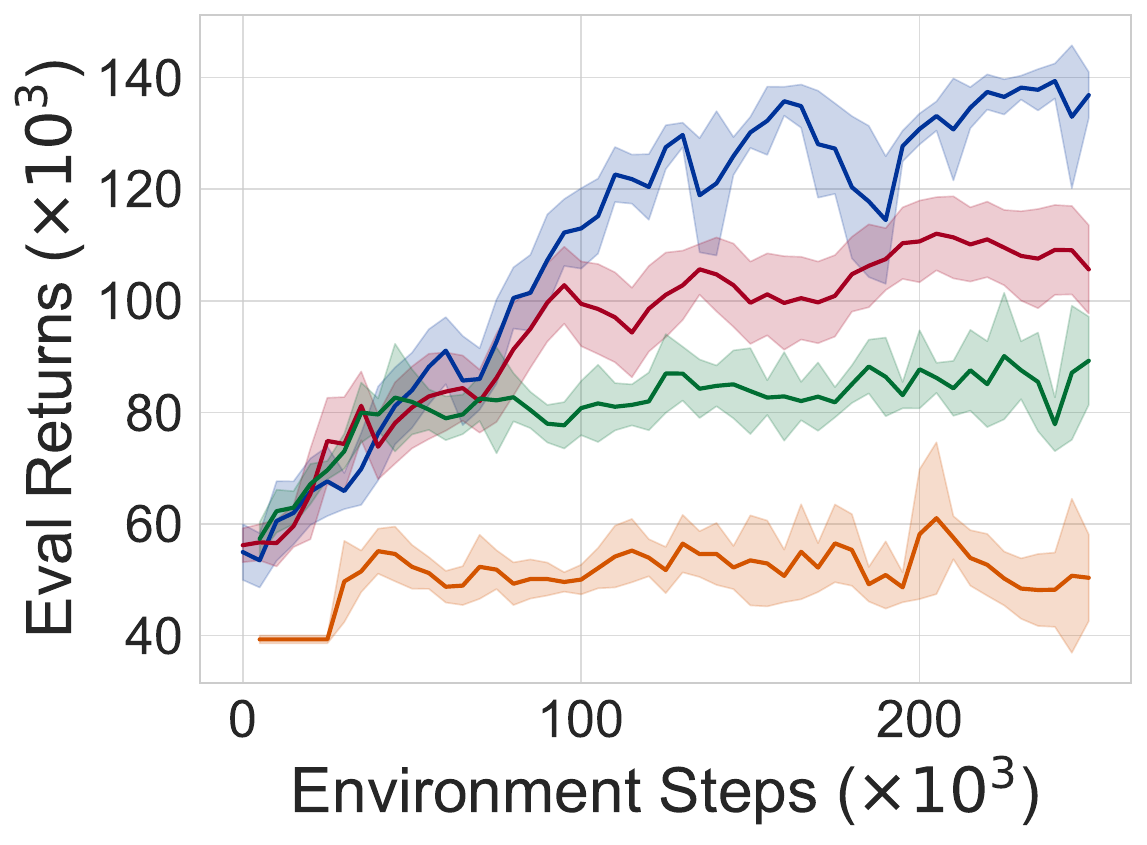}
        \caption{humanoidstandup-medium} \label{fig:exp:main:humanoidstandup-medium}
    \end{subfigure}\hfil

    \caption{\textbf{Main results.} A comparison between \textcolor{myblue}{\algname} in \textcolor{myblue}{blue}, the SOTA baseline \textcolor{myred}{RLPD} \citep{ball2023efficient} in \textcolor{myred}{red}, \textcolor{mygreen}{PEX} \citep{PEX} in \textcolor{mygreen}{green} and \textcolor{myorange}{BOORL} \citep{hu2024bayesiandesignprinciplesofflinetoonline} in \textcolor{myorange}{orange} on various D4RL benchmark tasks.
     \textcolor{myblue}{\algname} scores the best in all benchmarks, and the gap is especially large for D4RL Adroit tasks (door, hammer, pen, relocate), which are harder due to their larger action dimensionality. Both \textcolor{mygreen}{PEX} and \textcolor{myorange}{BOORL} require an offline pretraining process of 1M gradient steps each, and only the online finetuning phase with an initial pretrained jumpstart is shown here. In this view, \textcolor{myblue}{\algname} is much more computationally efficient in achieving the same level of performance.
    }\label{fig:exp:main}
        \vspace{-3mm}
\end{figure*}

\renewcommand{\thesubfigure}{\alph{subfigure}}  %

{\bf Comparison to random sampling.} The fundamental concept behind proving that our sampling technique surpasses random sampling and contributes to positive policy improvement involves initially applying the performance difference lemma. 
This approach yields the performance differential term $J\paren{\policy^{t+1}}-J\paren{\policy^t}$ between the updated policy and the current policy.
Our goal is to demonstrate that this term is non-negative under our sampling priority. To do this, we prove that by a shift of distribution, this term is no less than the gap
\begin{align}
J^{\pi^\star} - J^{\pi^t} - C \sqrt{\epsilon^t} \sup_{s, a}  |  {d^{\pi^{t+1}}(s, a)} / {\rho(s, a)} |.
\end{align}
When looking at the distribution shift
\begin{align*}
    \frac{d^{\pi^{t+1}}(s, a)}{\rho(s, a)} =&\left(\frac{{\pi^{t+1}}(a\given s)}{d^{\text{on}}(a\given s)}\right)^{1-\xi}\\ 
     &\cdot \frac{\sum_{s', a'}d^{\text{on}}(a', s') \pi^{t+1}(a'\given s')^\xi}{d^{\text{on}}(a\given s)^\xi} \cdot \frac{d^{\pi^{t+1}}(s)}{d^{\text{on}}(s)},
\end{align*} 
we notice the shift between online/offline dataset is canceled, and the remaining terms comprise a shift term $ {d^{\pi^{t+1}}(s)}/{d^{\text{on}}(s)}$ that characterizes how well the online data cover the visitation measure induced by the next policy, and another term that characterizes the shift in policy. In the sequel, we will see through an example why using some proper $\xi$ helps reduce the shift in policy.

{\bf Why does weighting by advantage help?}
We show that under certain conditions, the ratio $R^t(s, a)$ can decrease for increased value of $\xi$. Since $\xi$ does not influence the ratio between the state distribution, let us just consider the bandit case with ratio 
\begin{align*}
    R^t(a;\xi) = \left(\frac{{\pi^{t+1}}(a)}{d^{\text{on}}(a)}\right)^{1-\xi}  \cdot \frac{\sum_{a'}d^{\text{on}}(a') \pi^{t+1}(a')^\xi}{d^{\text{on}}(a)^\xi}.
\end{align*}
We illustrate the results of Theorem \ref{thm:perfdiff} on the bandit setting because its visitation measure reduces directly to the policy distribution—eliminating any dependence on a transition kernel—and note that the same argument carries over to MDPs with deterministic transitions.
Moreover, the argument will still provide sufficient insight.
Suppose the online data distribution $d^{\text{on}}(a)\propto \exp(\beta_1 r(a))$ for some parameter $\beta$ while the policy $\pi^{t+1}(a) \propto \exp(\beta_2 r(a))$ for some parameter $\beta_2 > \beta_1$. This is reasonable since the policy converges faster than the online buffer to the optimal policy. 
Then we have the following lemma.
\begin{restatable}[]{lem}{bandit}\label{lem:bandit}
For the bandit case with $d^{\text{on}}(a)\propto \exp(\beta_1 r(a))$ and $\pi^{t+1}(a) \propto \exp(\beta_2 r(a))$ for $\beta_2 > \beta_1$, the coefficient $\sup_{a}R^t(a;\xi)$ decreases as $\xi$ increases within the range $\xi\in(0, 1 - \beta_1/\beta_2)$. 
\end{restatable}
This lemma justifies that within a proper range of $\xi$, adding more advantage weighting would benefit learning by reducing the distributional shift.

%% file: experiments.tex
\definecolor{myblue}{RGB}{0, 51, 153}      %
\definecolor{myred}{RGB}{165, 0, 33}       %
\definecolor{mygreen}{RGB}{0, 110, 51}     %
\definecolor{myorange}{RGB}{211, 84, 0}    %
\definecolor{myplum}{RGB}{110, 0, 140}     %
\definecolor{myteal}{RGB}{0, 120, 130}     %
\definecolor{mygold}{RGB}{180, 130, 0}     %
\definecolor{myslate}{RGB}{50, 50, 120}    %
\definecolor{myrose}{RGB}{160, 0, 80}      %
\definecolor{myolive}{RGB}{90, 100, 0}     %
\definecolor{myrust}{RGB}{160, 50, 0}      %
\definecolor{mynavy}{RGB}{0, 30, 100}      %
\definecolor{myburgundy}{RGB}{130, 0, 50}  %
\definecolor{mypine}{RGB}{0, 80, 60}       %
\definecolor{myindigo}{RGB}{55, 0, 160}    %
\definecolor{mybronze}{RGB}{140, 90, 10}   %
\newcommand{\redtext}[1]{\textcolor{red}{#1}}
\section{Experiments}\label{experiments}

{\bf Environments.}
We evaluate \algname on tasks from the D4RL benchmark~\citep{fu2020d4rl} with Minari offline datasets \citep{minari}. Each environment has offline datasets composed of trajectories of a wide range of competence, ranging from simple to to expert on some, and cloned to expert on others. We defer additional details on the environments to the Appendix~\ref{app:experiment}.

{\bf Setup.}
We employ the basic setup of the SAC networks as recommended by \citep{ball2023efficient}. In particular, for critic networks and target critic networks, we have an ensemble of size $E = 10$ each, as well as entropy regularization. We use simpler multi-layer perceptrons of 2 layers of size 256 each to see if the agent is able to learn with less model capacity. Runs are compared over 250K online environment steps, with checkpoint evaluation at every 5K steps. Shaded areas represent one standard error of the mean based on ten seeds.

{\bf Baseline methods.}
For our main results, we compare our \textcolor{myblue}{\algname} algorithm against three baselines:
(1) \textcolor{myred}{RLPD}~\citep{ball2023efficient}, regarded as the SOTA baseline for addressing online RL with offline datasets, specifically having attained SOTA performance in these tasks;
(2) \textcolor{mygreen}{PEX} \citep{PEX}; and (3) \textcolor{myorange}{BOORL} \citep{hu2024bayesiandesignprinciplesofflinetoonline}. Both \textcolor{mygreen}{PEX} and \textcolor{myorange}{BOORL} require an offline pretraining process, and thus only results from their online finetuning phase with an initial jumpstart are shown in our plots. For the training of \textcolor{myblue}{\algname}, we found it beneficial to simply run \textcolor{myred}{RLPD} for an initial phase of $1/4$ the total number of environment steps and the full \textcolor{myblue}{\algname} for the remaining $3/4$. This enables the $Q$ and density networks to be stabilized for more accurate density and advantage estimations in the absence of an explicit offline pretraining phase.

We then conduct several ablation studies to investigate the sensitivity and effect of each part of our algorithm: (4) ablation on the density term $(\zeta = 0)$ yields a variant of \algname that only uses advantage estimations; (5) ablation on the advantage term $(\xi = 0)$ yields a density-only-based implementation that is similar to Off2On~\citep{lee2022offline}, an offline-to-online RL method, which also does density-based prioritized sampling but requires an offline pretraining phase; (6) ablation on the use of the LCB over the ensemble mean $(\beta = 0)$ yields a non-confidence-aware variant of \algname. We also experiment with (7) removing the offline dataset altogether to yield an online version of \algname, which naturally no longer involves the density term; and (8) comparing our prioritization scheme against the commonly used PER scheme. Please refer to Appendix~\ref{app:experiment} for our full list of runs on more diverse environments and datasets.

\vspace{-0.3cm}
\subsection{Main Results}
\vspace{-0.2cm}

\begin{figure*}[!t]
    \centering
    \begin{subfigure}{.24\textwidth}
        \centering
        \includegraphics[width=\textwidth,  clip={0,0,0,0}]{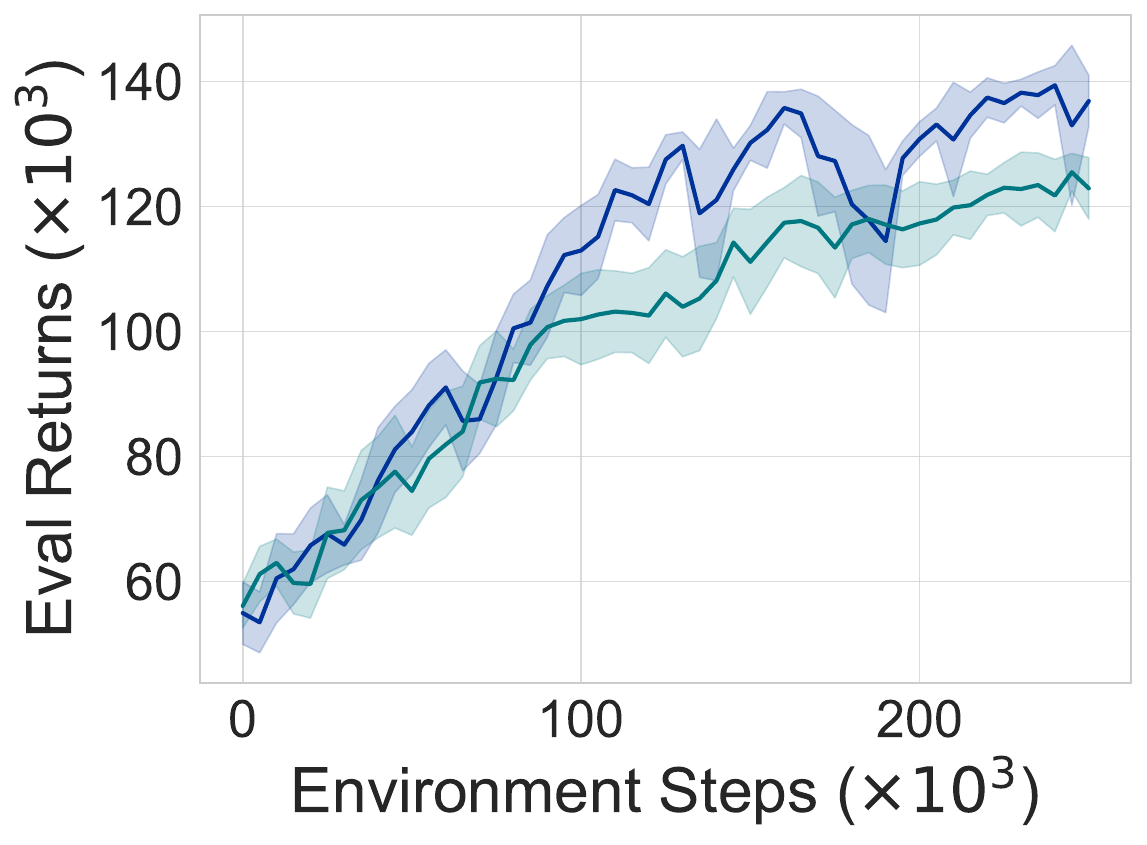}
        \caption{Density, $\zeta = 0$}\label{fig:ablation:density}
    \end{subfigure}\hfil
    \begin{subfigure}{.24\textwidth}
        \centering
        \includegraphics[width=\textwidth,  clip={0,0,0,0}]{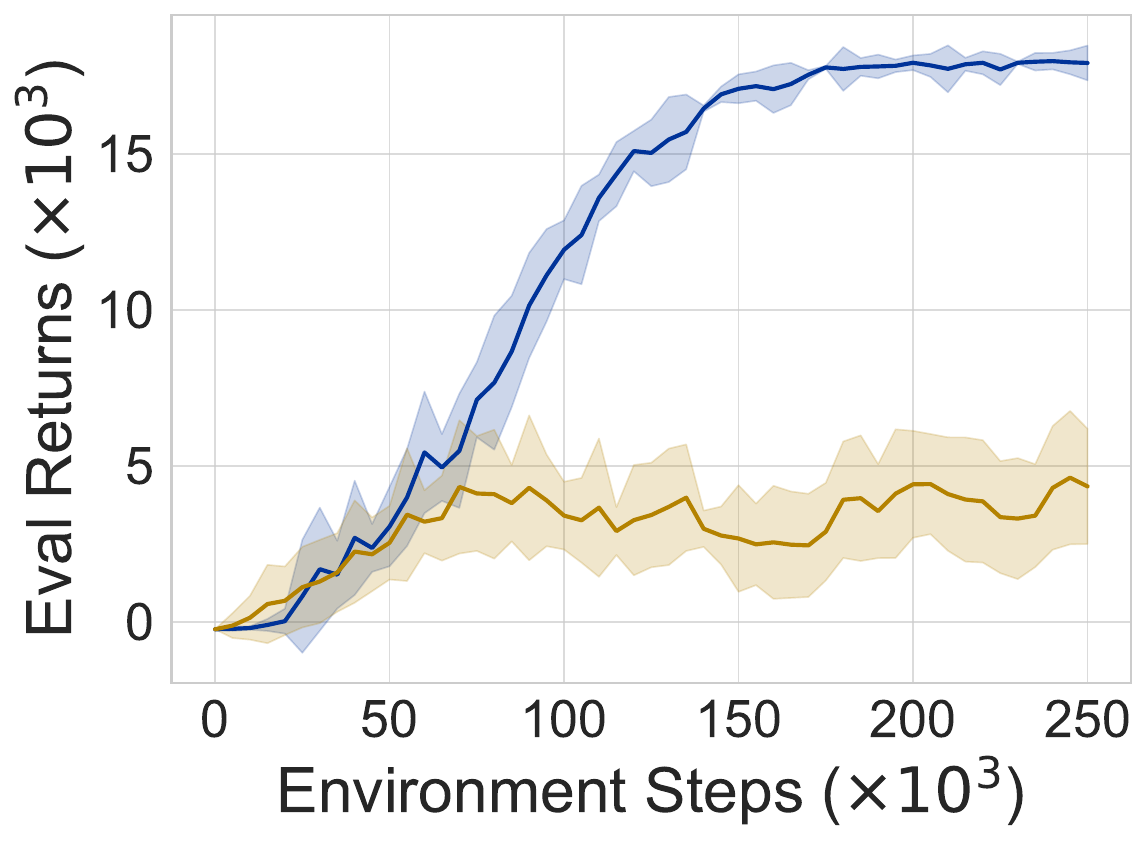}
        \caption{Advantage, $\xi = 0$}\label{fig:ablation:advantage}
    \end{subfigure}\hfil
    \begin{subfigure}{.24\textwidth}
        \centering
        \includegraphics[width=\textwidth,  clip={0,0,0,0}]{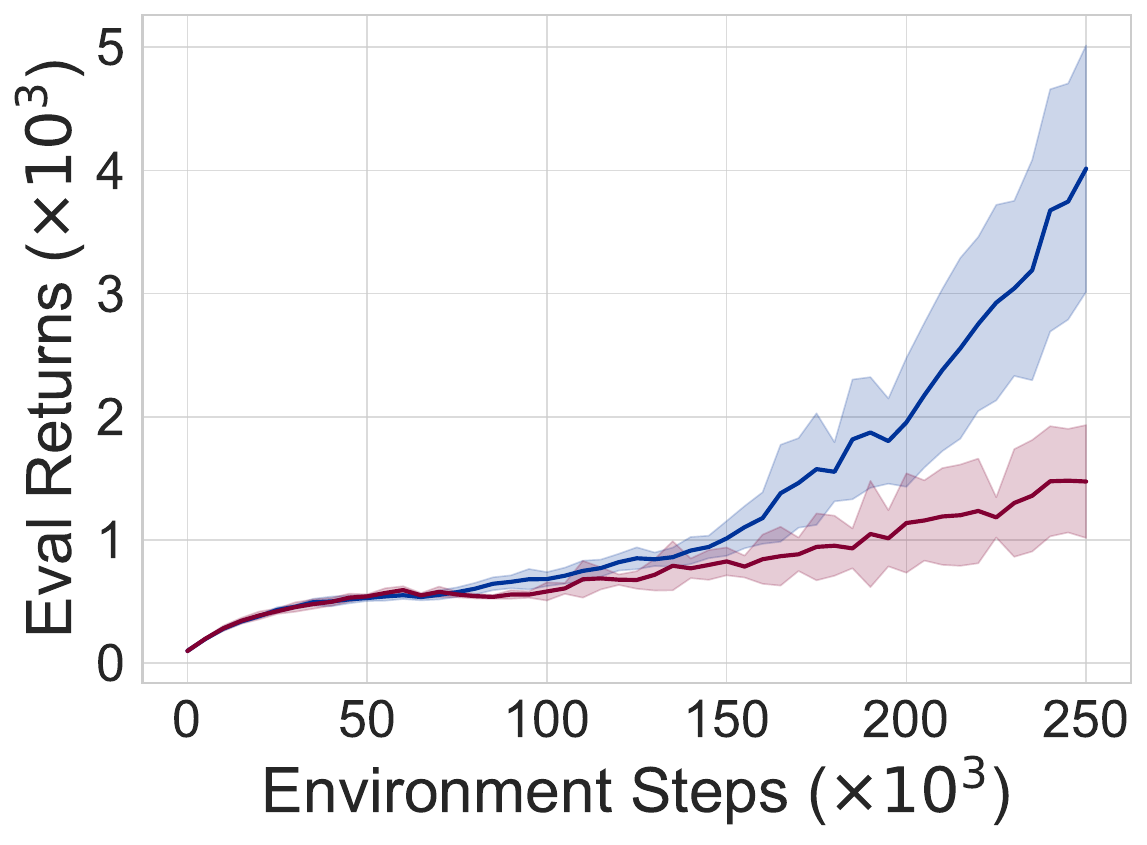}
        \caption{LCB, $\beta = 0$}\label{fig:ablation:lcb}
    \end{subfigure}\hfil
    \begin{subfigure}{.24\textwidth}
        \centering
        \includegraphics[width=\textwidth,  clip={0,0,0,0}]{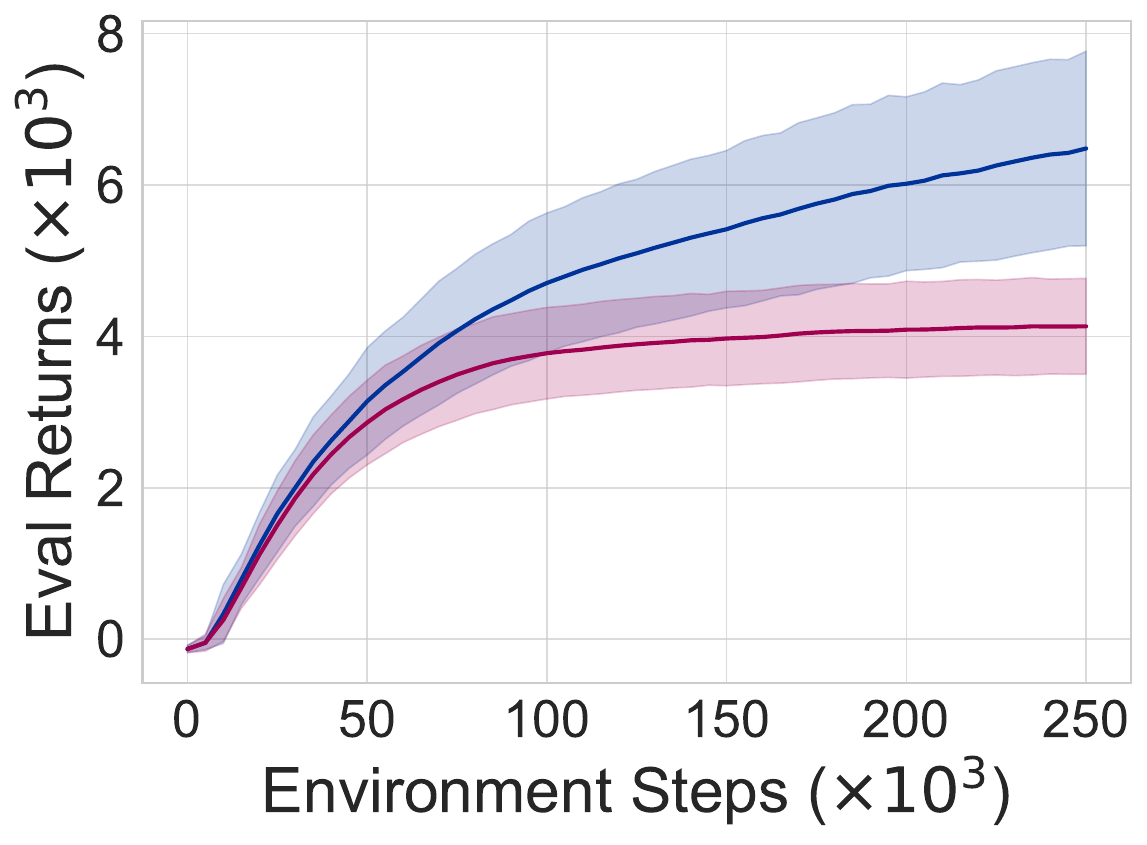}
        \caption{Online}\label{fig:ablation:online}
    \end{subfigure}\hfil
    \begin{subfigure}{.24\textwidth}
        \centering
        \includegraphics[width=\textwidth,  clip={0,0,0,0}]{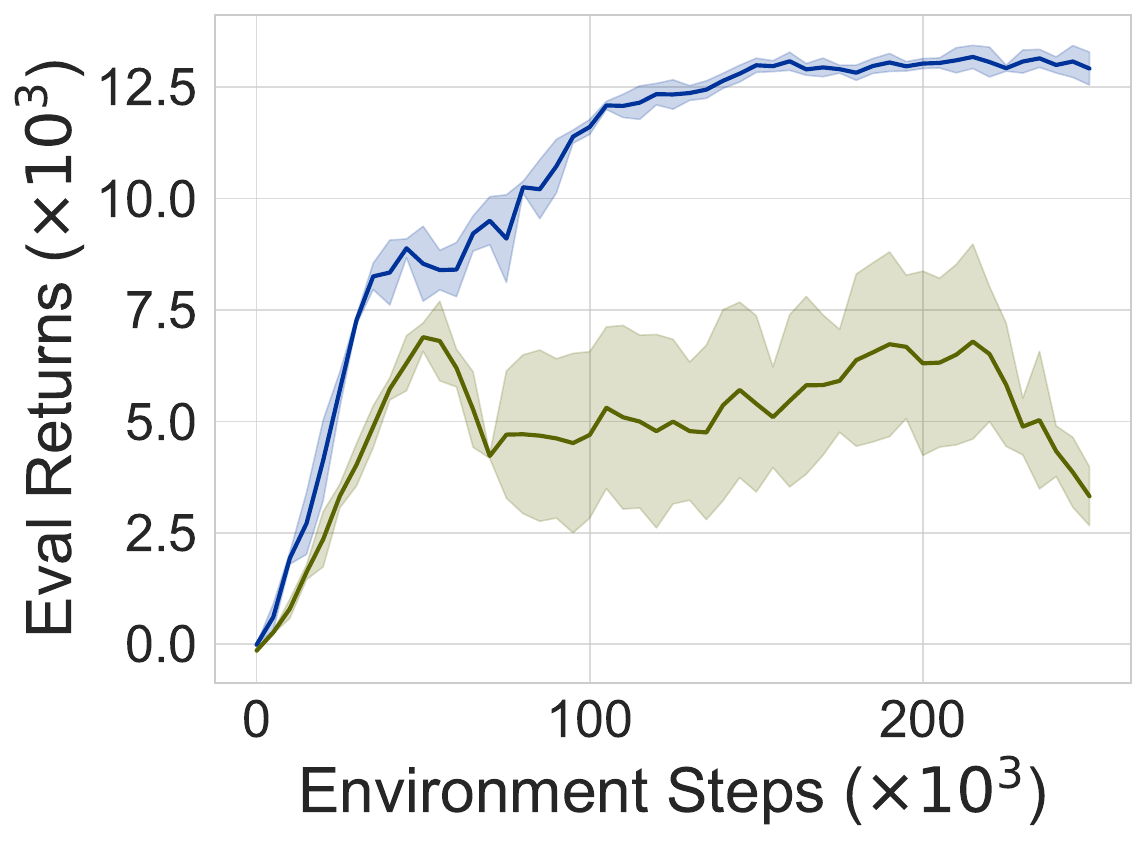}
        \caption{TD}\label{fig:ablation:td}
    \end{subfigure}\hfil
    \begin{subfigure}{.24\textwidth}
        \centering
        \includegraphics[width=\textwidth,  clip={0,0,0,0}]{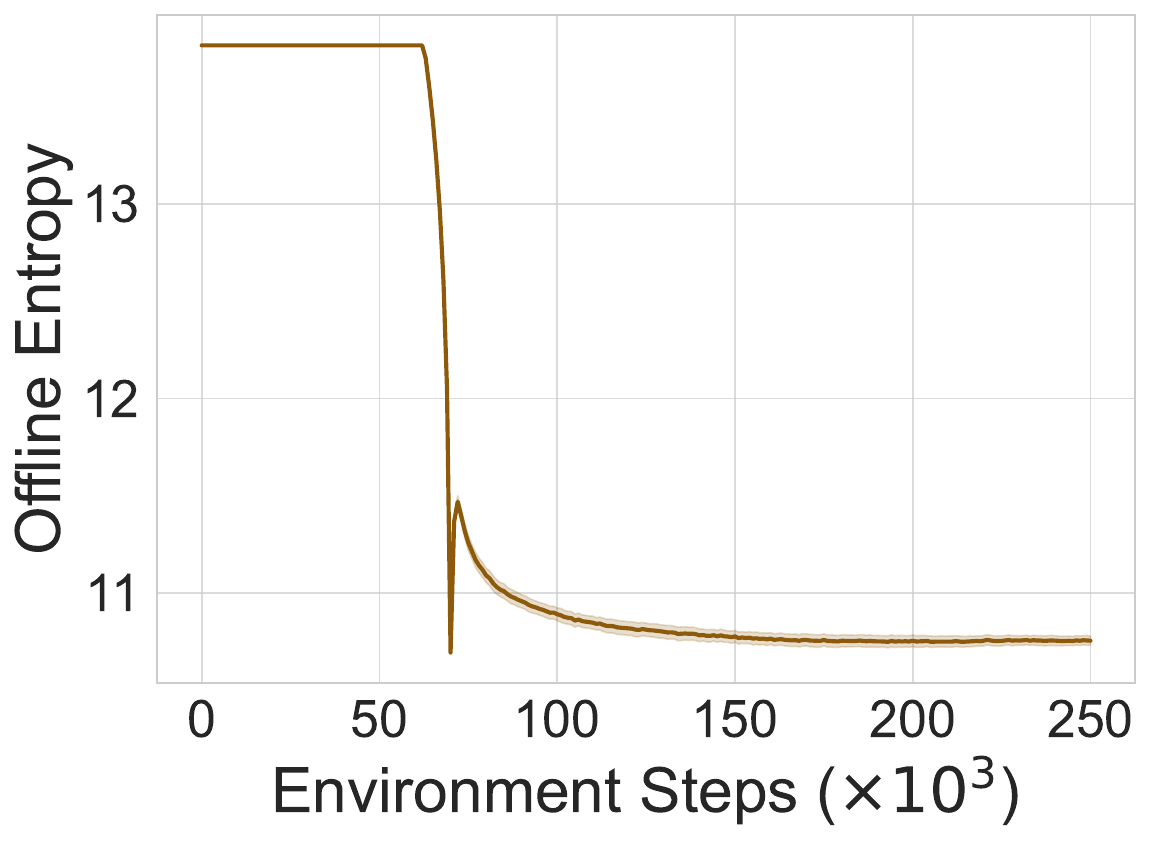}
        \caption{Offline entropy}\label{fig:ablation:entropy}
    \end{subfigure}\hfil
    \begin{subfigure}{.24\textwidth}
        \centering
        \includegraphics[width=\textwidth,  clip={0,0,0,0}]{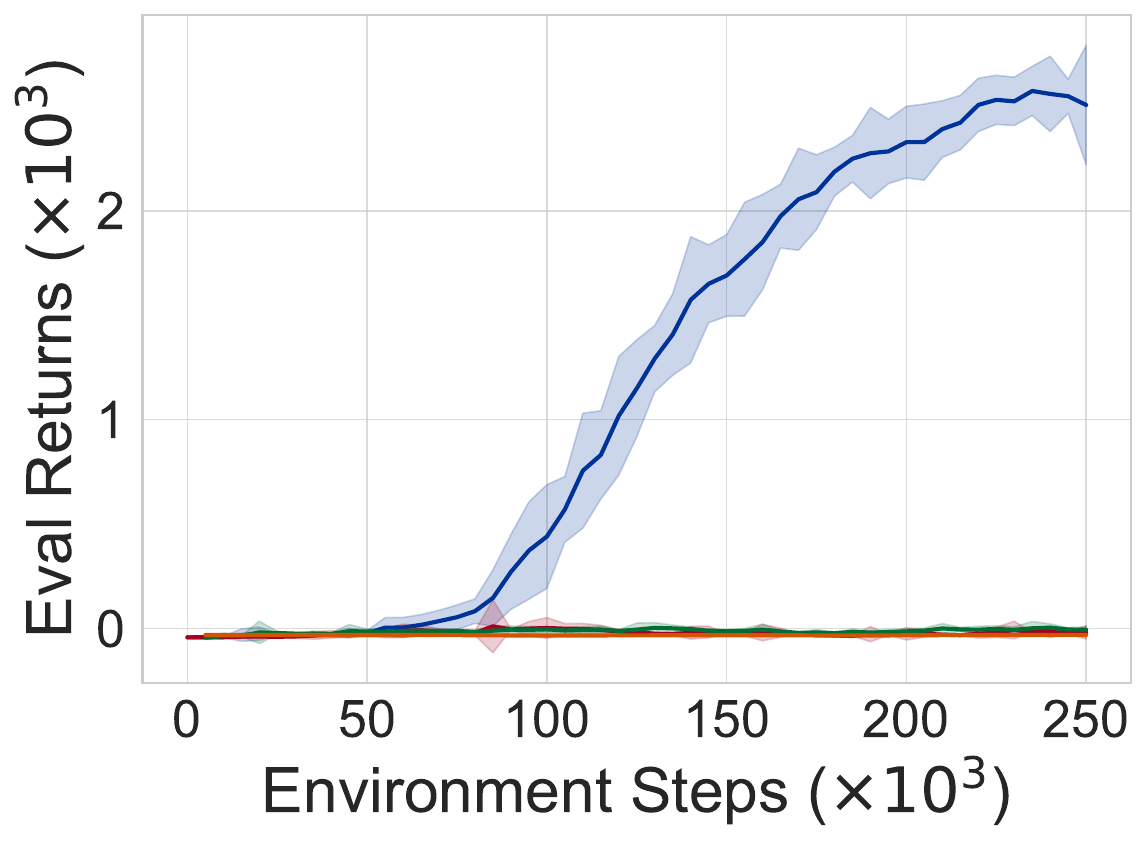}
        \caption{door-cloned}\label{fig:ablation:doorcloned}
    \end{subfigure}\hfil
    \begin{subfigure}{.24\textwidth}
        \centering
        \includegraphics[width=\textwidth,  clip={0,0,0,0}]{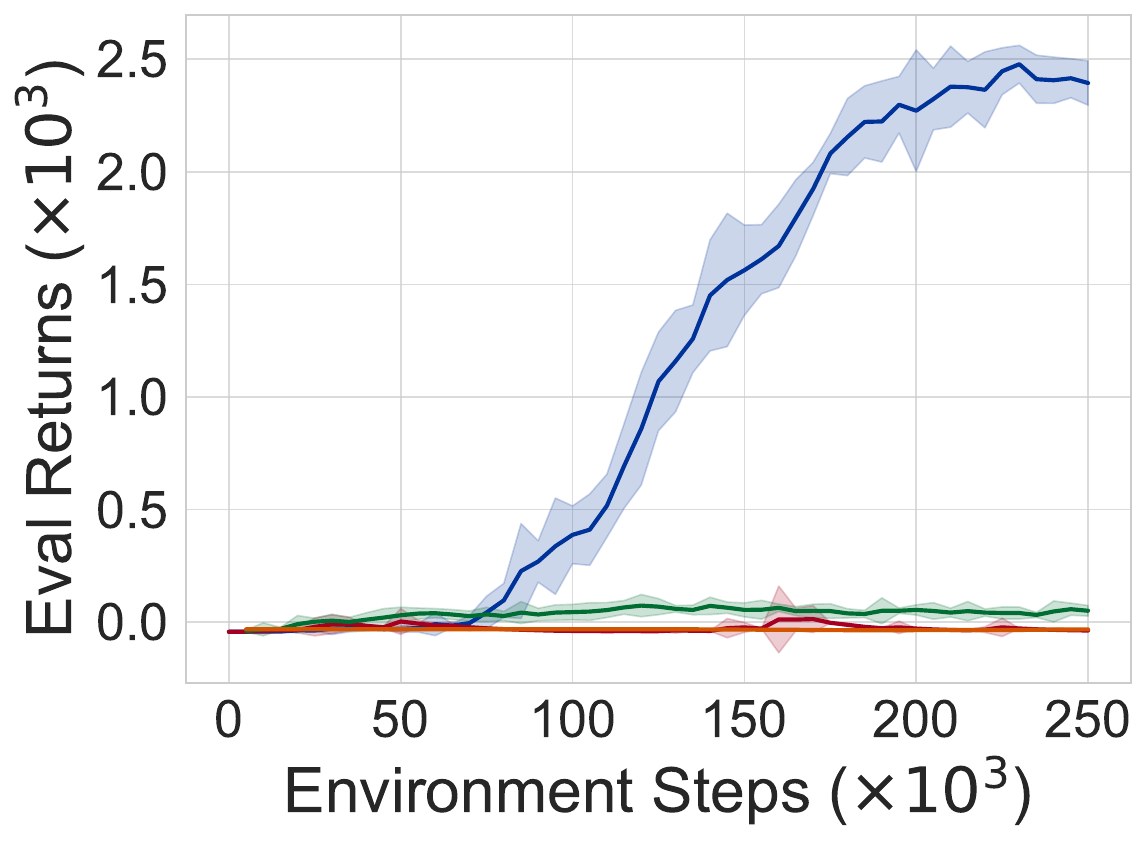}
        \caption{door-human}\label{fig:ablation:doorhuman}
    \end{subfigure}\hfil

    \vspace{-3pt}
    \caption{ \textbf{Ablation studies and effects of \algname}. The unablated version of \textcolor{myblue}{\algname} is in \textcolor{myblue}{blue} throughout. (a)(b)(c) Representative ablation studies of the density term, the advantage term and the LCB estimation. (d) Comparison of the online variant of \textcolor{myblue}{\algname} against purely using SAC online. (e) Comparison of \textcolor{myblue}{\algname} against using TD as the priority term. (f) The typical evolution of the entropy of the prioritized offline buffer. (g)(h) Comparison of \textcolor{myblue}{\algname} against \textcolor{myred}{RLPD},  \textcolor{mygreen}{PEX}, \textcolor{myorange}{BOORL} in the lower-quality/data-starved regime.}
  \label{fig:ablation:without_advantage}
  \vspace{-10pt}
\end{figure*}

\figref{fig:exp:main} presents a comparative evaluation of \textcolor{myblue}{\algname} against the baselines \textcolor{mygreen}{PEX}, \textcolor{myorange}{BOORL}, and the current SOTA method \textcolor{myred}{RLPD}. The results illustrate that \textcolor{myblue}{\algname} consistently outperforms these methods across all evaluated domains, achieving superior performance, robustness, and training stability.

As consistent in all subplots in \figref{fig:exp:main}, the performance of \textcolor{myblue}{\algname} is essentially identical to that of \textcolor{myred}{RLPD} in the first $1/4$ of steps. Of course, this is expected because we train \textcolor{myblue}{\algname} with an initial phase of \textcolor{myred}{RLPD}. However, after this point, their performances quickly diverge. The active sampling scheme of \textcolor{myblue}{\algname} enables it to stay robust in harder D4RL Adroit tasks, such as door-expert and relocate-expert, while \textcolor{myred}{RLPD} is less sample-efficient and effective. This performance advantage can be attributed to the fundamental difference in the two sampling strategies: while \textcolor{myred}{RLPD} relies on uniform random sampling, \textcolor{myblue}{\algname} employs an active sampling approach.

Meanwhile, although the other two baselines, \textcolor{mygreen}{PEX} and \textcolor{myorange}{BOORL}, promised to address the issues related to catastrophic forgetting and offline-online distribution shifts that come hand-in-hand with an explicit offline pretraining phase, these guarantees are not fulfilled. Oftentimes, the offline pretrained models do not offer a competitive initial jumpstart in these harder tasks (e.g., \figref{fig:exp:main:door-expert}, \figref{fig:exp:main:hammer-cloned}), which defeat the goals of an explicit, computationally expensive offline pretraining phase. Furthermore, even when they do (e.g., \figref{fig:exp:main:relocate-expert}, \figref{fig:exp:main:halfcheetah-medium}, \figref{fig:exp:main:humanoid-medium}), the effects of both catastrophic forgetting and distributional shift are ever present. It is in this light that the benefits of an ``end-to-end'' approach like \textcolor{myblue}{\algname} and \textcolor{myred}{RLPD} are clear over computationally expensive and non-robust methods with an offline pretraining phase, at least when we are concerned with sufficiently many online environment steps.

\vspace{-1em}
\subsection{Ablation Studies and Effects of \algname}
\vspace{-1.5em}
\textbf{Ablation on density term.} \figref{fig:ablation:density} compares the performance of \algname to \algname with only the advantage term in the sampling priority ($p=\exp\paren{\advTemperature\cdot\underline{A}}$), without the density term, on humanoidstandup-medium. This amounts to setting $\zeta = 0$. Our ablation results consistently show that \algname, which incorporates onlineness through the density term $w = d^{\text{on}}/d^{\text{off}}$, outperforms the version that does not. Removing the density term leaves the agent to potentially focus on transitions that are unlikely to occur during online interactions, which can yield sample inefficiencies.

{\bf Ablation on advantage term.} \figref{fig:ablation:advantage} compares the performance of \algname to \algname with only the advantage term in the sampling priority ($p=\mathbb{I}^{\text{off}}\underline{\density}\paren{\state,\action}+\mathbb{I}^{\text{on}}$), without the density term, on hammer-cloned. This amounts to setting $\xi = 0$. The results show that \algname with the advantage term surpasses its counterpart that only considers onlineness in prioritizing samples. This superiority is attributed to the advantage term, which effectively screens out transitions that are either non-informative or harmful. For example, even if a transition is near on-policy, it may not necessarily be advantageous for the policy to learn.

{\bf Ablation on LCB.} \figref{fig:ablation:lcb} compares the performance of \algname to \algname without the LCB calibration for the advantage term on humanoid-medium. In other words, $\underline{A} = \hat{A}$, the ensemble mean. This amounts to setting $\beta = 0$. This demonstrates the importance of staying pessimistic in advantage estimations, as over-optimism can be compounded by the exponentiation, resulting in a miscalibrated, highly biased sampling scheme.

{\bf Pure online setting and TD.} \figref{fig:ablation:online} presents an ablation study comparing an online variant of \algname against a pure SAC agent, with neither having access to offline data. The effective prioritization scheme is then $p = \exp(\xi \cdot \underline{A})$, and we still observe the perks of this method, which prioritizes more advantageous transitions. This heuristic is also superior to using the PER TD-error heuristic, as shown in \figref{fig:ablation:td}.

{\bf Evolution of offline dataset entropy.} The easiest way to quantify the effect of the \algname sampling scheme is to measure the entropy corresponding to the sampling probability on the offline dataset, where the priority has both density and advantage contributions. As shown in \figref{fig:ablation:entropy}, after the initial phase of RLPD training, this entropy quickly drops off from 13.8 nats, which corresponds to $\approx$ 1M effective transitions, to 10.8 nats, which corresponds to $\approx$ 50K effective transitions. It is via this selection mechanism that \algname is able to achieve better sample efficiency.

{\bf Effect of dataset quality and quantity.} Even though baseline methods are somewhat comparable to \algname in the high-quality, high-quantity data regime, e.g., \figref{fig:exp:main:door-expert}, their performances drastically suffer when either quality or quantity is affected. \figref{fig:ablation:doorcloned} demonstrates the robustness of \algname when the dataset is not of expert quality, while \figref{fig:ablation:doorhuman} shows that it is still able to learn in a sample-efficient manner with a much smaller offline dataset. In both of these scenarios, all baseline methods are unable to learn, despite both PEX and BOORL having an extensive offline pretraining phase.

\section{Conclusion}\label{sec:con}
We present \algname, a novel algorithm for online RL with offline datasets through a {confidence-aware} active advantage-aligned sampling strategy. This algorithm is theoretically motivated by the objective of shifting the sampling distribution toward more beneficial transitions to maximize policy improvement, and its enhancement gap is quantified. Moreover, we conduct comprehensive experiments with various qualities of offline data, demonstrating that \algname outperforms existing baselines with significance. We also conduct multiple ablation studies and confirm the importance of each component within the active advantage-aligned formula,  as well as its superiority in the online setting and over existing prioritization schemes. While our approach primarily aims to enhance performance, it is in essence trading an offline pretraining phase for an extra density network, which could cause higher computational costs with large architectures. Though it is empirically seen that our approach is relatively computationally efficient, further reducing computational demands will be a focus of our future work.

\subsubsection*{Acknowledgements}
We thank Yicheng Luo for the initial discussion and helpful suggestions. 
This work is supported  by the RadBio-AI project (DE-AC02-06CH11357), U.S. Department of Energy Office of Science, Office of Biological and Environment Research, the Improve project under contract (75N91019F00134, 75N91019D00024, 89233218CNA000001, DE-AC02-06-CH11357, DE-AC52-07NA27344, DE-AC05-00OR22725), Exascale Computing Project (17-SC-20-SC), a collaborative effort of the U.S. Department of Energy Office of Science and the National Nuclear Security Administration, the AI-Assisted Hybrid Renewable Energy, Nutrient, and Water Recovery project (DOE DE-EE0009505), and the National Science Foundation under Grant No. IIS 2313131, IIS
2332475, DMS 2413243 and HDR TRIPODS (2216899).

%% file: supp_additional_theory.tex
\section{Theoretical Motivation}\label{app:theory_proof}
In this section, we show that the active advantage-aligned sampling strategy
helps mitigate the gap between offline data distribution, online data distribution and the current on-policy distribution, which serves as a main theoretical motivation for designing \algname.

\perfdiff*

\begin{proof} \textit{(Proof of Theorem \ref{thm:perfdiff}).}\quad 
Define visitation measures
\[  
d_h^\pi(s, a) = \EE_{a\sim \pi(\cdot| s)} \left[ \ind(s_h = s, a_h = a)\right], \quad d^\pi(s, a) = \frac{1}{1-\gamma}\sum_{h=1}^\infty \gamma^h d_h^\pi(s,a).
\]
Consider a sufficiently small one-step update in the policy network with step-size $\eta$.
Define $J_\alpha^\pi = \expctover{\state\sim\rho^{\policy},{\action\sim\policy}}{\sum_{t=0}^{\infty} \gamma^t\paren{\reward_t+\alpha \mathcal{H}\paren{\policy\paren{\action|\state}}}}$. 
Let $\tilde\pi$ be the policy from the last iteration.
In the following, we abbreviate $\EE_\pi[\cdot]$ as $\EE[\cdot]$. 
\begin{align*}
V^{\policy} - V^{\tilde\policy}&=\mathbb{E} \bracket{\tuple{\policy,\QFunc^{\policy}-\alpha\log\policy}-\langle{\tilde\policy,\QFunc^{\tilde\policy}-\alpha\log\tilde\policy} \rangle_{\cA}}\\
&=\mathbb{E}\bracket{\langle{\policy,\QFunc^{\policy}-\QFunc^{\tilde\policy}}\rangle_{\cA}+\langle{\policy-\tilde\policy,\QFunc^{\tilde\policy}}\rangle_{\cA}-\alpha \tuple{\policy,\log{\policy}}+\alpha\tuple{\tilde\policy,\log\tilde\policy}}\\
&=\mathbb{E}\bracket{\tuple{\policy,\reward+\gamma \mathbb{P}\VFunc^{\policy}-\reward+\gamma \mathbb{P}\VFunc^{\policy}}+\langle \policy-\tilde\policy,\QFunc^{\tilde\policy}\rangle_{\cA} -
\alpha\tuple{\policy,\log\policy} + \alpha\tuple{\tilde\policy,\log \tilde\policy}}\\
&=\mathbb{E}\bracket{ \gamma\bigl\langle\policy, \PP\bigl(\VFunc^{\policy}-\VFunc^{\tilde\policy}\bigr)\bigr\rangle_{\cA}+\bigl\langle\policy-\tilde\policy,\QFunc^{\tilde\policy}\bigr\rangle_{\cA}-\alpha\tuple{\policy,\log\policy}+\alpha\tuple{\tilde\policy,\log\tilde\policy}} , 
\end{align*}
Using this iterative form, we conclude that 
\begin{align*}
    J_\alpha^\pi - J_\alpha^{\tilde\pi} 
    &= \EE\left[
        \sum_{h=1}^\infty \gamma^i \left( \bigl\langle\policy_i-\tilde\policy_i,\QFunc_i^{\tilde\policy}\bigr\rangle_{\cA}-\alpha\tuple{\policy_i,\log\policy_i}+\alpha\tuple{\tilde\policy_i,\log\tilde\policy_i}  \right)
    \right] \\
    &= \EE_{d^\pi} \left[ \bigl\langle\policy-\tilde\policy,\QFunc^{\tilde\policy}\bigr\rangle_{\cA}-\alpha\tuple{\policy,\log\policy}+\alpha\tuple{\tilde\policy,\log\tilde\policy} \right].
\end{align*}
Recall our definition of $\sigma(s,a)$ that 
\begin{align}
    \sigma(s, a) = \exp( \xi \hat A^{\tilde \pi}(s,a))\cdot \frac{d^{\text{on}}(s, a)}{\mu(s,a)}, 
    \label{eq:sigma_theory}
\end{align}
where $\mu(\cdot, \cdot)$ is the distribution in the sampled batch and $d^{\text{on}}(\cdot, \cdot)$ is the online distribution. 
Note that the advantage function $\hat A^{\tilde\pi} (s, a) = \hat Q^{\tilde\pi} (s, a) - \alpha \log \sum_{a'}\exp(\alpha^{-1}\hat Q(s, a'))$ is calculated using policy $\tilde\pi$ and Q function $\hat Q^{\tilde\pi}$ obtained from the last iteration in the above formula.
Let us define $\pi_{\phi^\star}$ as the optimal policy under the current Q function $\tilde Q$:
\begin{align*}
    \pi^\star(\cdot\given s) &= \argmin_{\pi} \kl \left( \pi(\cdot \given s) \,\bigg\|\, \frac{\exp(\alpha^{-1} \QFunc^{\tilde\policy}(s, \cdot))}{\tilde Z_\alpha(s)} \right) \\
    &=\argmax_{\pi}  \bigl\langle{\policy\paren{\cdot|\state},\QFunc^{\tilde\policy}\paren{\state,\cdot}-\alpha\log\policy\paren{\cdot|\state}}\bigr\rangle_{\cA} \propto \exp(\alpha^{-1}A^{\tilde\policy}(s, \cdot)).
\end{align*}
where $\tilde Z_\alpha(s)$ is the normalization factor at state $s$ for the exponential of the $Q$ function, and $A^{\tilde\pi}(s, \cdot)$ is the advantage function under policy $\tilde\pi$.
Recall by policy optimization:
\begin{align*}
    \hat\pi = \argmax_\pi\expctover{\mu}{\sigma\paren{\state,\action} \bigl\langle{\policy\paren{\cdot|\state},\hat\QFunc^{\tilde\policy}\paren{\state,\cdot}-\alpha\log\policy\paren{\cdot|\state}}\bigr\rangle_{\cA}}, 
\end{align*}
where $\hat\QFunc^{\tilde\pi}$ is the estimated Q function at the current iteration.
In the above formula, $\mu$ is the sampled data distribution and $\sigma$ is the quantity calculated in \eqref{eq:sigma_theory}. 
Suppose we take some function class $\pi_\phi$ which contains the optimal one-step policy improvement $\pi^\star$ and also the optimization target $\hat\pi$.
Using a shift of distribution, we have 
\begin{align*}
    \mu(s, a) \sigma(s, a) 
        &= \mu(s, a) \cdot \frac{d^{\text{on}}(s, a)}{\mu(s, a)} \cdot \exp(\xi \hat A^{\tilde\pi}(s, a)) ={d^{\text{on}}(s, a)} \cdot \hat\pi(a\given s)^\xi \\
        &= d^{\hat\pi}(s, a) \cdot \frac{d^{\text{on}}(s)}{d^{\hat\pi}(s)} \cdot \frac{d^{\text{on}}(a\given s)}{\hat\pi(a\given s)^{1-\xi}} \propto \rho(s, a),
\end{align*}
where we define $\rho(s, a)$ as the probability density induced by the above distribution.
Here, the first ratio $d^{\text{on}}(s)/d^{\pi^\star}(s)$ is the state-drift between the online data and the next-step optimal policy.
Since the online batches are refreshing as the algorithm proceeds, the ratio will be close to $1$. The second ratio term characterizes the drift caused by a mismatch in the policy. Intuitively, as we know the policy $\tilde\pi$ from the last iteration, we can use this information to further boost the alignment between the online policy and the next-step policy.
Suppose the Q function is learned up to $\epsilon$ error, that is 
\begin{align*}
    \expctover{\rho}{(\QFunc^{\tilde\pi}(s, a) - \hat\QFunc^{\tilde\pi}(s, a))^2} \leq \epsilon.
\end{align*}
Then, we have performance difference lemma that 
\begin{align*}
    J_{\alpha}^{\hat\pi} -J_{\alpha}^{\pi^\star}
       & = \EE_{d^{\hat\pi}} \left[ \bigl\langle\hat\pi, \QFunc^{\tilde\pi}\bigr\rangle_{\cA}-\alpha\tuple{\hat\pi,\log\hat\pi} - \bigl(\bigl\langle\pi^\star, \QFunc^{\tilde\pi}\bigr\rangle_{\cA} - \alpha\tuple{\pi^\star,\log\pi^\star}\bigr) \right]\\
       &= %
       \EE_{d^{\hat\pi}} \left[ \bigl\langle\hat\pi, \QFunc^{\tilde\pi}\bigr\rangle_{\cA}-\alpha\tuple{\hat\pi,\log\hat\pi} - \bigl(\bigl\langle\hat\pi, \hat\QFunc^{\tilde\pi}\bigr\rangle_{\cA}-\alpha\tuple{\hat\pi,\log\hat\pi}\bigr)\right]\\
       &\qquad + \EE_{d^{\hat\pi}} \left[ \bigl\langle\hat\pi, \hat\QFunc^{\tilde\pi}\bigr\rangle_{\cA}-\alpha\tuple{\hat\pi,\log\hat\pi} - \bigl(\bigl\langle\pi^\star, \hat\QFunc^{\tilde\pi}\bigr\rangle_{\cA}-\alpha\tuple{\pi^\star,\log\pi^\star}\bigr)\right]\\
       &\qquad + \EE_{d^{\hat\pi}} \left[ \bigl\langle\pi^\star, \hat\QFunc^{\tilde\pi}\bigr\rangle_{\cA}-\alpha\tuple{\pi^\star,\log\pi^\star} - \bigl(\bigl\langle\pi^\star, \QFunc^{\tilde\pi}\bigr\rangle_{\cA}-\alpha\tuple{\pi^\star,\log\pi^\star}\bigr)\right]\\
       &\ge %
       \EE_{d^{\hat\pi}} \left[ \bigl\langle\hat\pi - \pi^\star, \QFunc^{\tilde\pi} - \hat\QFunc^{\tilde\pi}\bigr\rangle_{\cA} \right] \\
       &\ge - \sup_{s, a}\left|\frac{\pi^\star(a\given s)}{\hat\pi(a\given s)} - 1\right| \cdot \EE_{d^{\hat\pi}}[|\QFunc^{\tilde\pi} - \hat\QFunc^{\tilde\pi}|] \ge - C \cdot \EE_{d^{\hat\pi}}[|\QFunc^{\tilde\pi} - \hat\QFunc^{\tilde\pi}|]
\end{align*}
where $C$ is an absolute constant given that both $Q^{\tilde\pi}$ and $\hat Q^{\tilde\pi}$ are uniformly bounded.
Here, the first inequality holds by the policy optimization step where we upper bound the second term by zero, and the last inequality holds by the assumption that the $Q$ function class is uniformly bounded.
Now, by a shift of distribution 
\begin{align*}
    \EE_{d^{\hat\pi}}[|\QFunc^{\tilde\pi} - \hat\QFunc^{\tilde\pi}|] 
        &= \EE_{\rho}\left[|\QFunc^{\tilde\pi} - \hat\QFunc^{\tilde\pi}|  \cdot \frac{d^{\hat\pi}(s, a)}{\rho(s, a)}\right]  %
        \le \sqrt{\EE_{\rho}[(\QFunc^{\tilde\pi} - \hat\QFunc^{\tilde\pi})^2]} \cdot \sup_{s, a} \left| \frac{d^{\hat\pi}(s, a)}{\rho(s, a)} \right|.
\end{align*}
Let's look at the distribution ratio 
\begin{align*}
    \frac{d^{\hat\pi}(s, a)}{\rho(s, a)} &= \frac{{\hat\pi}(a\given s) }{\hat\pi(a\given s)^\xi \cdot d^{\text{on}}(a\given s)^{1-\xi} }  \cdot \frac{\sum_{s', a'}d^{\text{on}}(a', s') \hat\pi(a'\given s')^\xi}{d^{\text{on}}(a\given s)^\xi} \cdot \frac{d^{\hat\pi}(s)}{d^{\text{on}}(s)} \\ 
    &= \left(\frac{{\hat\pi}(a\given s)}{d^{\text{on}}(a\given s)}\right)^{1-\xi}  \cdot \frac{\sum_{s', a'}d^{\text{on}}(a', s') \hat\pi(a'\given s')^\xi}{d^{\text{on}}(a\given s)^\xi} \cdot \frac{d^{\hat\pi}(s)}{d^{\text{on}}(s)}.
\end{align*}
Therefore, the policy improvement is guaranteed by 
\begin{align*}
    J_{\alpha}^{\hat\pi} - J_{\alpha}^{\tilde\pi} = J_{\alpha}^{\hat\pi} - J_{\alpha}^{\pi^\star} + J_{\alpha}^{\pi^\star} - J_{\alpha}^{\tilde\pi} \ge J_{\alpha}^{\pi^\star} - J_{\alpha}^{\tilde\pi} - C \cdot \sqrt \epsilon \cdot \sup_{s, a} \left| \frac{d^{\hat\pi}(s, a)}{\rho(s, a)} \right|.
\end{align*}
This completes the proof.

\end{proof}

Now we give a formal proof for Lemma \ref{lem:bandit}.

\bandit*
\begin{proof}\textit{(Proof of Lemma \ref{lem:bandit})}
Under the reparameterization $d^{\text{on}}(a)\propto \exp(\beta_1 r(a))$ and $\pi^{t+1}(a) \propto \exp(\beta_2 r(a))$, we have for the coefficient $R^t(a;\xi)$ that 
\begin{align*}
    R^t(a;\xi) &\propto \exp\left(\bigl((1-\xi)(\beta_2 - \beta_1) - \xi \beta_1\bigr) \cdot r(a) \right) \\
    &= \exp\left(\bigl((1-\xi)\beta_2 - \beta_1\bigr)\cdot r(a)\right).
\end{align*}
Within the range $\xi\in(0, 1-\beta_1/\beta_2)$, we always have $(1-\xi)\beta_2 - \beta_1>0$. Hence, the largest coefficient always occurs on action $\tilde a = \argmax_{a'} r(a')$.
In addition, we consider the following ratio
\begin{align*}
    \log \left(\frac{R(a;\xi)}{R(a;0)}\right) 
    &= -\xi \log(\pi^{t+1}(a)) + \log \left(\sum_{a'}d^{\text{on}}(a') \pi^{t+1}(a')^\xi \right) \\
    &= -\xi \beta_2 r(a) + \log \left(\sum_{a'}\exp((\beta_1 + \beta_2 \xi) r(a')) \right).
\end{align*}
Consider the gradient of $\log \left(\sum_{a'}\exp((\beta_1 + \beta_2 \xi) r(a')) \right)$ with respect to $\xi$:
\begin{align*}
    &\frac{\partial}{\partial \xi} \log \left(\sum_{a'}\exp((\beta_1 + \beta_2 \xi) r(a')) \right) 
     = \frac{\sum_{a'} \beta_2 r(a') \exp((\beta_1 + \beta_2 \xi) r(a'))}{\sum_{a'}\exp((\beta_1 + \beta_2 \xi) r(a'))} - \beta_2 r(a).
\end{align*}
Note that the largest probability ratio happens for $\tilde a = \argmax_{a'} r(a')$. Since the softmax is strictly less than the argmax when $r$ has different values in each action, the above derivative for action $\tilde a$ is negative, meaning that by increasing $\xi$, the value of $R(\tilde a;\xi)$ will decrease. 
As $\sup_{a}R(a;\xi) = R(\tilde a;\xi)$ by our previous discussion, we complete the proof. 
\end{proof}

\section{Additional Preliminaries}\label{app:preliminaries}

\emph{Layer Normalization:} %
Off-policy RL algorithms often query the learned \QFunc--function with out-of-distribution actions, leading to overestimation errors due to function approximation. This can cause training instabilities and even divergence, particularly when the critic struggles to keep up with growing value estimates. To address this, prior research has employed Layer Normalization to ensure that the acquired functions do not extrapolate in an unconstrained manner. Layer Normalization acts to confine \QFunc-values within the boundaries set by the norm of the weight layer, even for actions beyond the dataset. As a result, the impact of inaccurately extrapolated actions is substantially reduced, as their associated \QFunc-values are unlikely to significantly exceed those already observed in the existing data. Consequently, Layer Normalization serves to alleviate issues such as critic divergence and the occurrence of catastrophic overestimation.

\emph{Update-to-Data:} Enhancing sample efficiency in Bellman backups can be accomplished by elevating the frequency of updates conducted per environment step. This approach, often referred to as the update-to-data (UTD) ratio, expedites the process of backing up offline data.

\emph{Maximum Entropy RL:} %
Incorporating entropy into the learning objective (as defined in \eqref{eq:maxentRL})  helps mitigate overconfidence in value estimates, particularly when training with offline datasets. In offline RL, policies may become overly conservative due to limited dataset coverage, leading to suboptimal exploration during fine-tuning. By preserving policy stochasticity, entropy regularization ensures that the agent remains adaptable when transitioning from offline training to online interactions. This controlled exploration has been shown to improve training stability and prevent premature convergence~\citep{ball2023efficient, chen2021randomized, haarnoja2018soft, hiraoka2021dropout}.

\section{Additional Related Work}\label{app:relatedWork}

\paragraph{Offline to online RL.}%

In an effort to mitigate the sample complexity of online RL~\citep{liu2023active},
offline RL utilizes fixed datasets to train policies without online interaction, however it can be prone to extrapolation errors that lead to overestimation of state-action values.
Recent off-policy actor-critic methods~\citep{fujimoto2019off, kostrikov2021offline, kumar2020conservative, wang2020critic} seek to mitigate these issues by limiting policy learning to the scope of the dataset, thereby minimizing extrapolation error. 
Strategies for reducing extrapolation error include value-constrained approaches~\citep{kumar2020conservative} that aim for conservative value estimates and policy-constrained techniques~\citep{nair2020awac} that ensure the policy remains close to the observed behavior in the data.
\fix{There are several works that leverage advantage estimation to guide policy improvement in purely offline RL, such as LAPO~\citep{chen2022lapo}, A2PR~\citep{liu2024adaptive}, and A2PO~\citep{qing2024a2po}. However, they are not well-suited for online settings because they fail to consider the importance of “onlineness,” measured by the density ratio, to align with the needs of online RL exploration and exploitation. Additionally, they do not account for uncertainty in advantage estimation.
}
\XL{Several works leverage advantage estimation to guide policy improvement in offline RL, such as LAPO, A2PR, and A2PO. Could you incorporate a discussion on these methods in your paper to provide a more comprehensive comparison?}

While offline RL methods can outperform the dataset’s behavior policy, they rely entirely on static data~\citep{levine2020offline}. When the dataset has comprehensive coverage, methods like FQI~\citep{antos2007fitted} or certainty-equivalence model learning~\citep{rathnam2023unintended} can efficiently find near-optimal policies. %
However, in practical scenarios with limited data coverage, policies tend to be suboptimal. 
One approach to addressing this suboptimality is to follow offline RL with online fine-tuning, however as discussed above, existing methods are prone to catastrophic forgetting and performance drops during fine-tuning~\citep{luo2023finetuning}. 
\fix{In contrast, \algname begins with online RL while incorporating offline data to enhance the policy, selectively leveraging offline data to facilitate online policy improvement.}

\section{Implementation Details of \algname}
We provide more extensive details of \algname in Algorithm \ref{alg:a3rl} as Algorithm \ref{app:alg:a3rl_details}.

\begin{algorithm*}[!t]
\caption{\algname}
\label{app:alg:a3rl_details}
\begin{algorithmic}[1]
\State Select LayerNorm, large ensemble Size $\EnsembleNum$, gradient steps $\GNum$, discount $\gamma$,  temperature $\alpha$.
\State Randomly initialize Critic $\theta_i$ (set targets $\theta'_i = \theta_i$) for $i=1,2,\ldots, \EnsembleNum$, Actor $\phi$ parameters. 
\State Select critic EMA weight $\rho$, batch size $\batch$, determine number of Critic targets to subset $Z\in\curlybracket{1,2}$
\State Initialize offline priority buffer $\mathcal{B}^{\text{off}}$ with offline dataset $\mathcal{D}$, {online priority buffer $\mathcal{B}^{\text{on}} = \ReplayBuffer \leftarrow \emptyset$ }.
\While {True}
\State Receive initial observation state $\state_0$
     \For{$t=0, \ldots, \horizon$}
       \State Take action $\action_t \sim \policy_{\phi}\paren{\cdot|\state_t}$, update online buffer $\mathcal{B}^{\text{on}} \leftarrow \mathcal{B}^{\text{on}} \cup \curlybracket{\paren{\state_t,\action_t,\reward_t,\state_{t+1}}}$.

       \For{$g=1, \ldots, G$}
                   \State Uniformly sample $\frac{N}{2}$ transitions from $\mathcal{B}^{\text{off}}$ and similarly $\frac{N}{2}$ from $\mathcal{B}^{\text{on}}$. \textcolor{NavyBlue}{Update $w_\psi$ to maximize \eqref{eq:lower_bound}}.
            \State \textcolor{NavyBlue}{Sample with priority $\frac{N}{2}$ transitions from $\mathcal{B}^{\text{off}}$ and similarly $\frac{N}{2}$ from $\mathcal{B}^{on}$ to form batch $b$ of size $N$.}
            \State Sample set $\mathcal{Z}$ of $Z$ indices from $\curlybracket{1,\ldots, \EnsembleNum}$ 
            \State With $b$, set $y=\reward+\discFactor\Big(\min_{i\in \mathcal{Z}} \QFunc_{\theta'_i}\paren{\state',\action'}+\Big.$ 
            $\Big.\alpha \log \policy_{\phi}\paren{\action'|\state'}\Big), \action'\sim \policy_{\phi}\paren{\cdot|\state'}$

            \For{$i=1,\ldots, \EnsembleNum$}
            \State \textcolor{Maroon}{Calculate importance weights $u_i$ via \eqref{eq:weights}.}
            \State \textcolor{Maroon}{Update $\theta_i$ minimizing loss:}
             \textcolor{Maroon}{$\loss= \sum_{i} u_i\cdot
            \paren{y-\QFunc_{\theta_i}\paren{\state,\action}}^2$}
            \EndFor
            \State Update target networks:
             $\theta'_{i}\leftarrow \rho \theta'_{i}+\paren{1-\rho}\theta_i$
       \EndFor
       \State With $b$, update $\phi$ maximizing objective:
       \State \quad\quad$\frac{1}{\EnsembleNum}\sum_{i=1}^{\EnsembleNum} \QFunc_{\theta_i}\paren{\state,\action} - \alpha \log \policy_{\phi} \paren{\action|\state},$
 where $\action\sim\policy_{\phi}\paren{\cdot|\state}$, {$\paren{\state,\action}\sim b$}.
    \State  \textcolor{NavyBlue}{Update priorities for $b$ according to \eqref{eq:adv_priority}.}
     \EndFor
\EndWhile

\end{algorithmic}
\end{algorithm*}

\section{Limitations of the Prior State-of-The-Art.} 
A drawback of RLPD, as discussed by~\citet{ball2023efficient}, lies in its symmetric random sampling method applied to both online and offline data, disregarding the significance of individual transitions \fix{for evolving quality of policy}. This predefined approach to sampling can potentially lead to less than optimal policy improvements due to the omission of vital data and inefficiencies arising from the use of redundant data. Such inefficiencies fail to offer any positive contribution towards enhancing policy. To address the limitation, our research presents an innovative active data sampling technique, specifically designed to optimize the use of both online and offline data in the process of policy improvement.

\section{Experimental Details}\label{app:experiment}
In order to ensure fair evaluation, all baselines and ablation studies are assessed using an equal number of environment interaction steps. We average results over 10 seeds to obtain the final result. One standard error of the mean is shaded for each graph.

\subsection{Additional experiments}

\begin{figure*}[h]
    \centering
    \begin{subfigure}{.24\textwidth}
        \centering
        \includegraphics[width=\textwidth,  clip={0,0,0,0}]{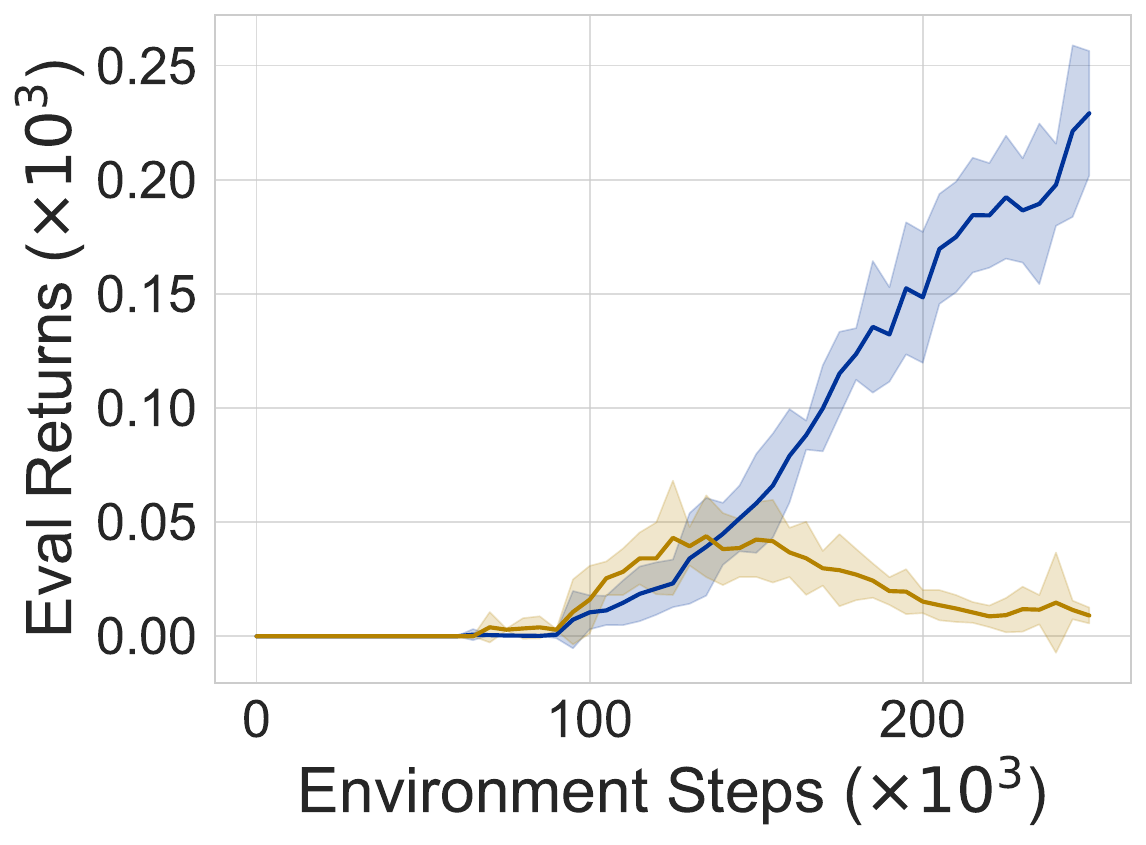}
        \caption{antmaze-umaze}
    \end{subfigure}\hfil
    \begin{subfigure}{.24\textwidth}
        \centering
        \includegraphics[width=\textwidth,  clip={0,0,0,0}]{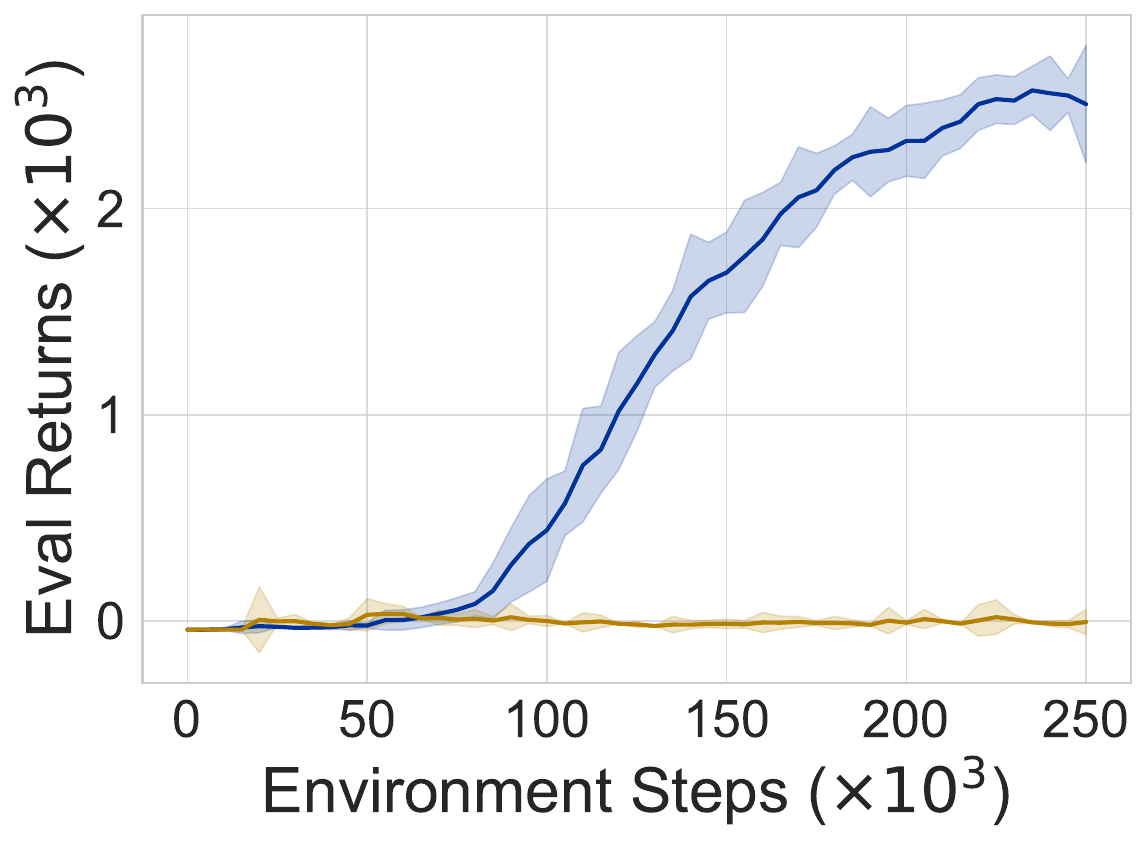}
        \caption{door-cloned}
    \end{subfigure}\hfil
    \begin{subfigure}{.24\textwidth}
        \centering
        \includegraphics[width=\textwidth,  clip={0,0,0,0}]{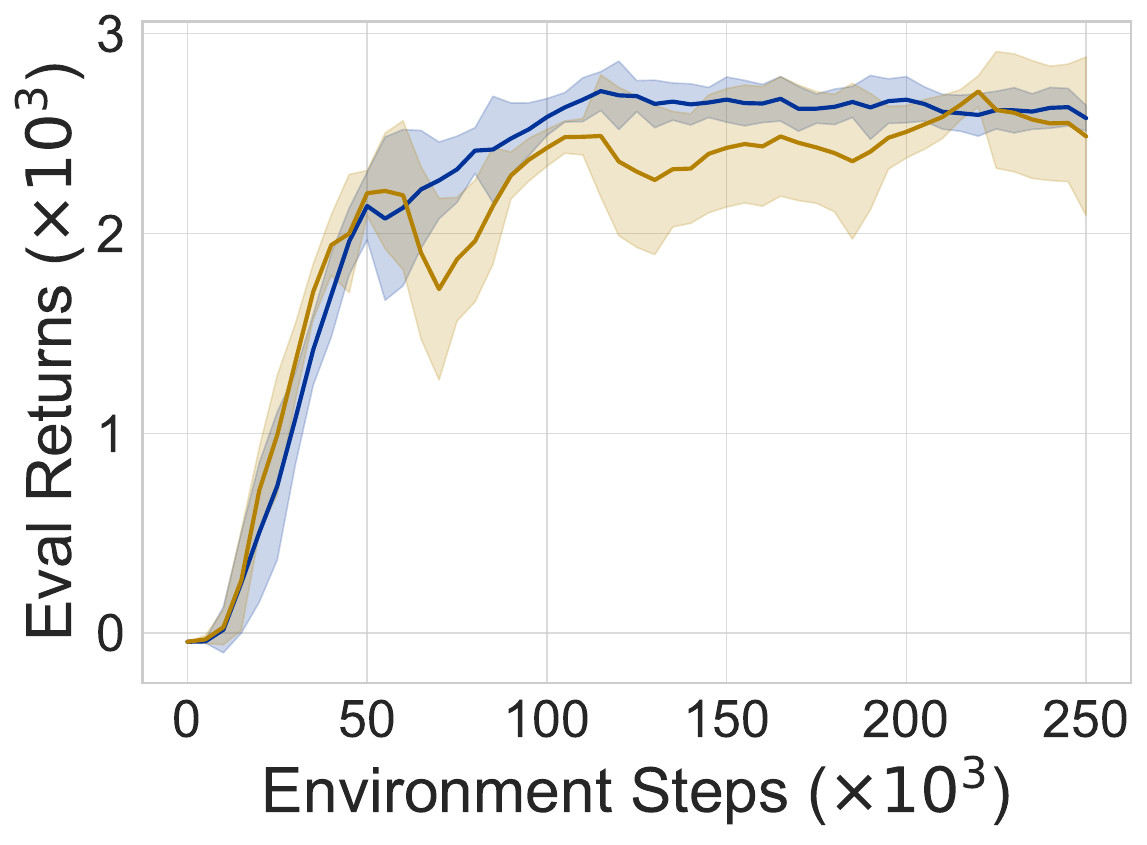}
        \caption{door-expert}
    \end{subfigure}\hfil
    \begin{subfigure}{.24\textwidth}
        \centering
        \includegraphics[width=\textwidth,  clip={0,0,0,0}]{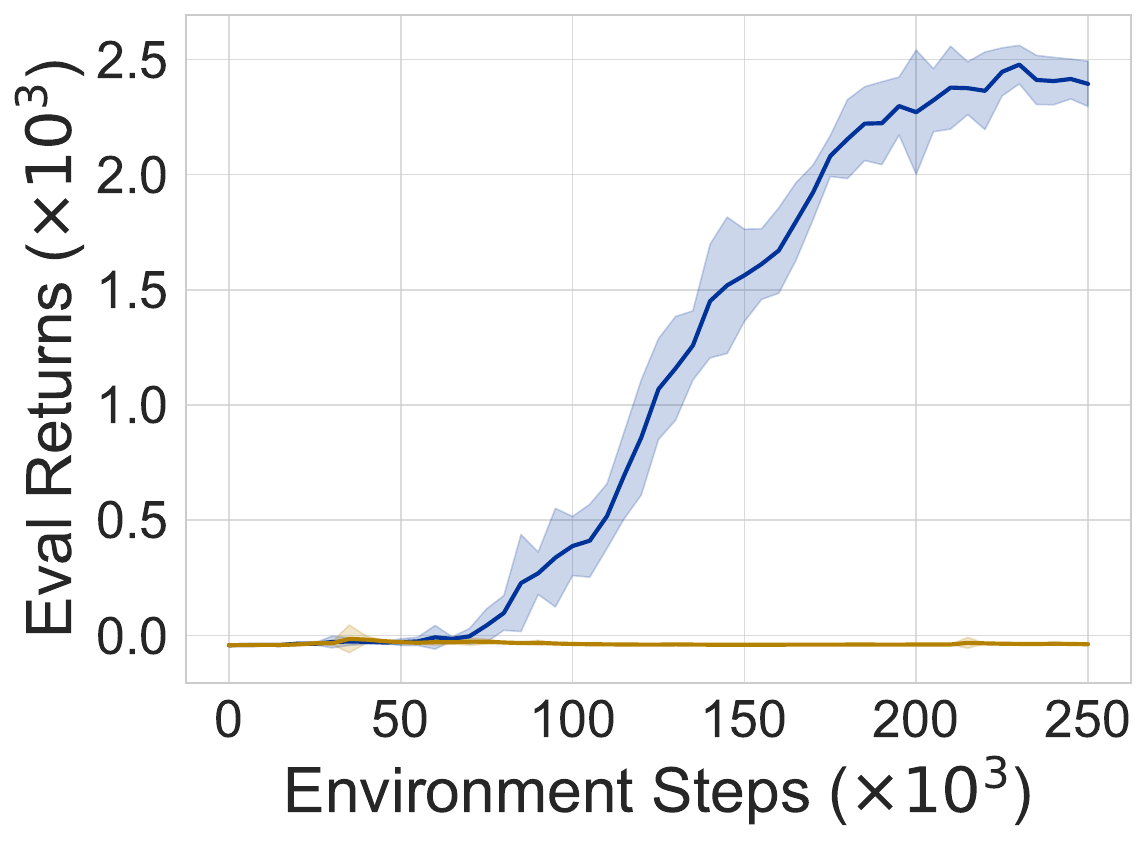}
        \caption{door-human}
    \end{subfigure}\hfil
    \begin{subfigure}{.24\textwidth}
        \centering
        \includegraphics[width=\textwidth,  clip={0,0,0,0}]{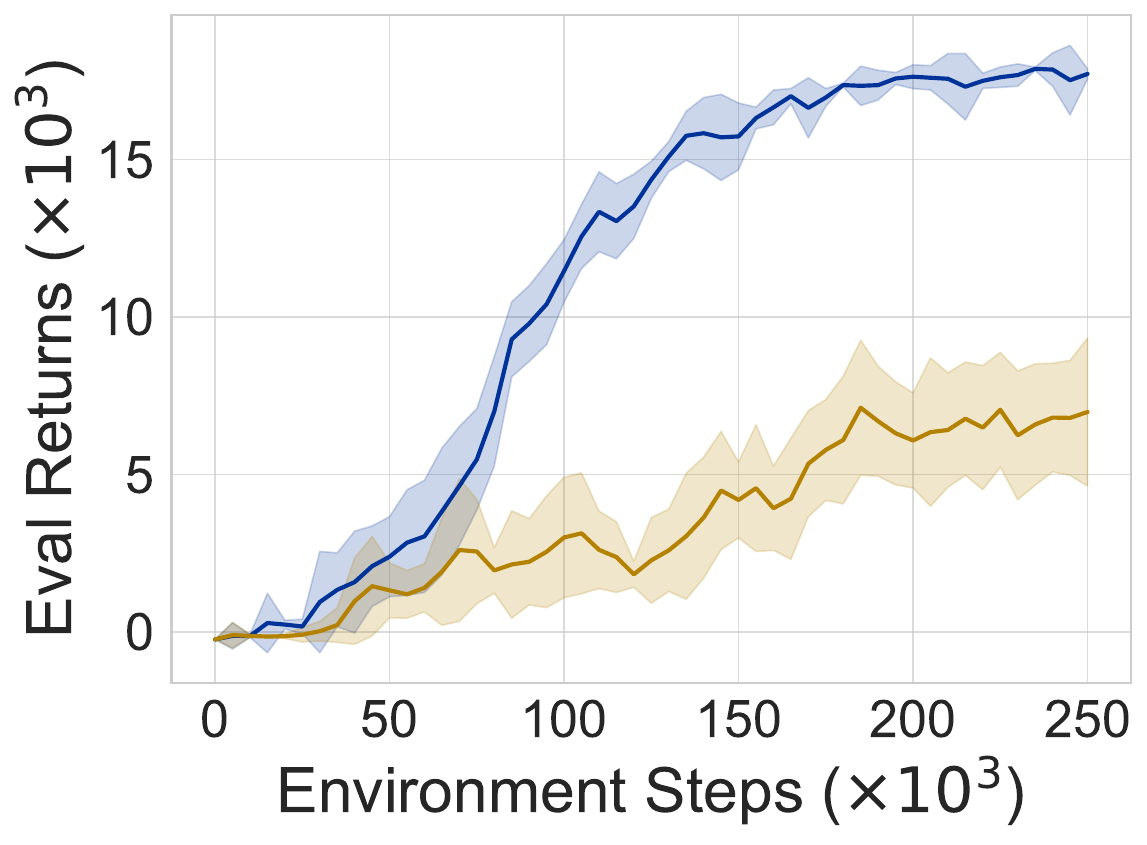}
        \caption{hammer-human}
    \end{subfigure}\hfil
    \begin{subfigure}{.24\textwidth}
        \centering
        \includegraphics[width=\textwidth,  clip={0,0,0,0}]{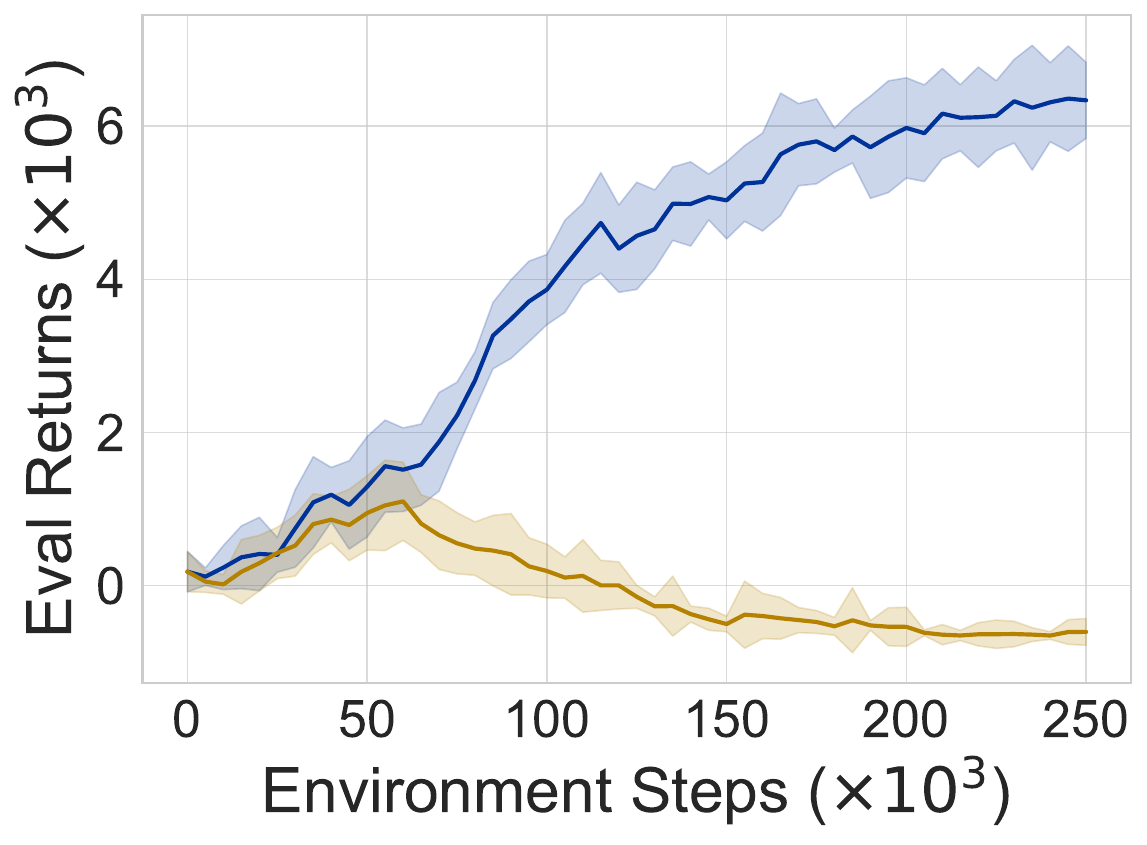}
        \caption{pen-cloned}
    \end{subfigure}\hfil
    \begin{subfigure}{.24\textwidth}
        \centering
        \includegraphics[width=\textwidth,  clip={0,0,0,0}]{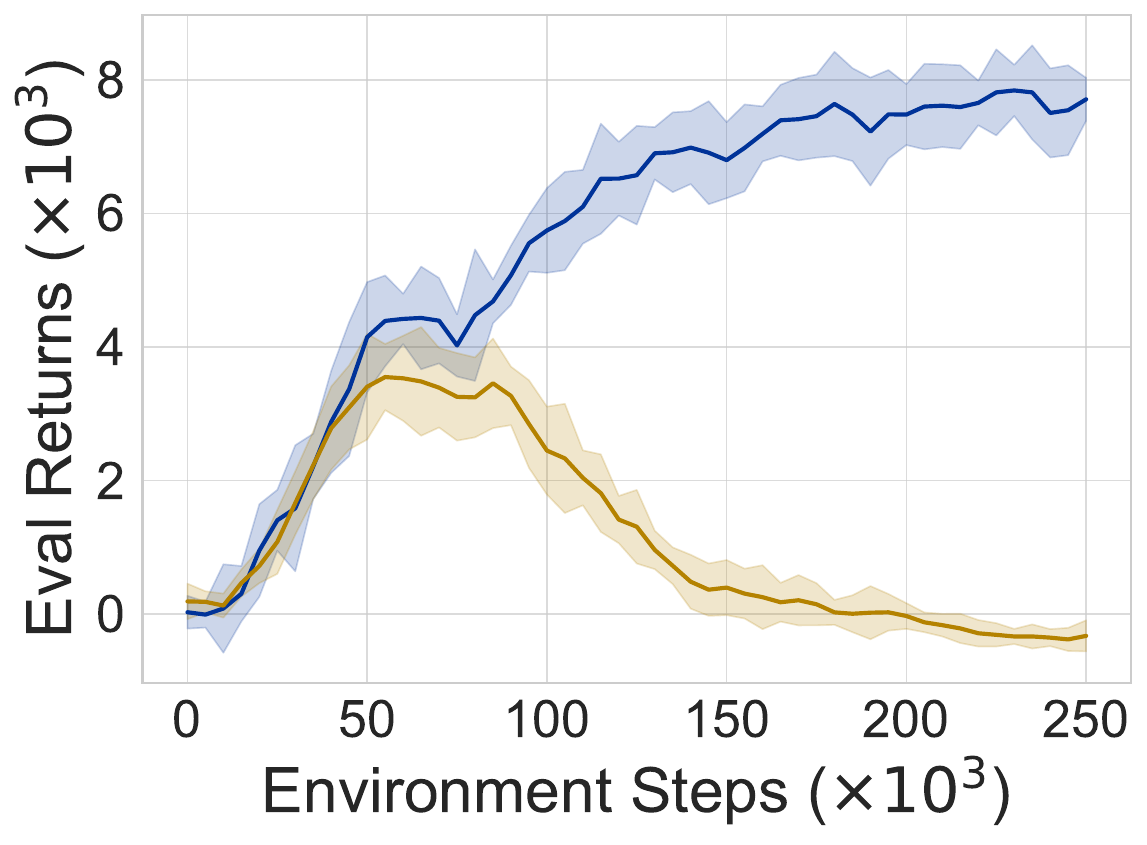}
        \caption{pen-expert}
    \end{subfigure}\hfil
    \begin{subfigure}{.24\textwidth}
        \centering
        \includegraphics[width=\textwidth,  clip={0,0,0,0}]{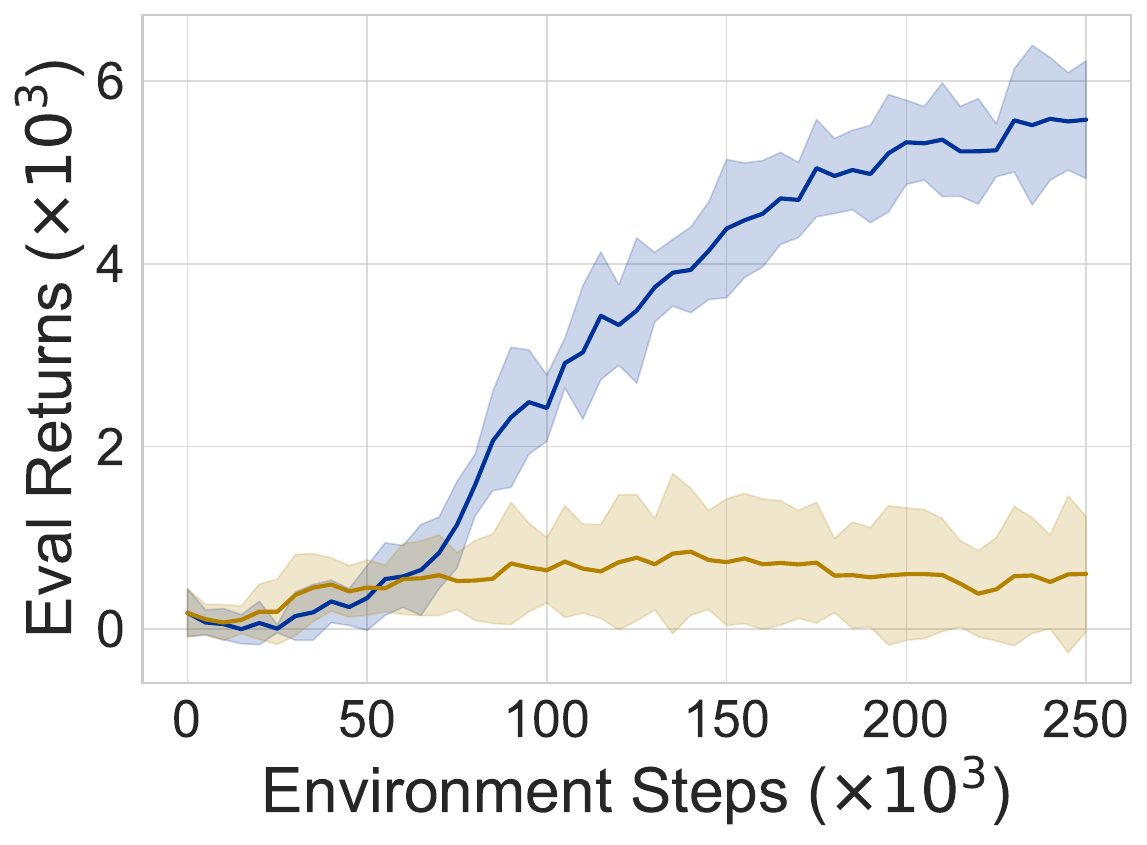}
        \caption{pen-human}
    \end{subfigure}\hfil
    \begin{subfigure}{.24\textwidth}
        \centering
        \includegraphics[width=\textwidth,  clip={0,0,0,0}]{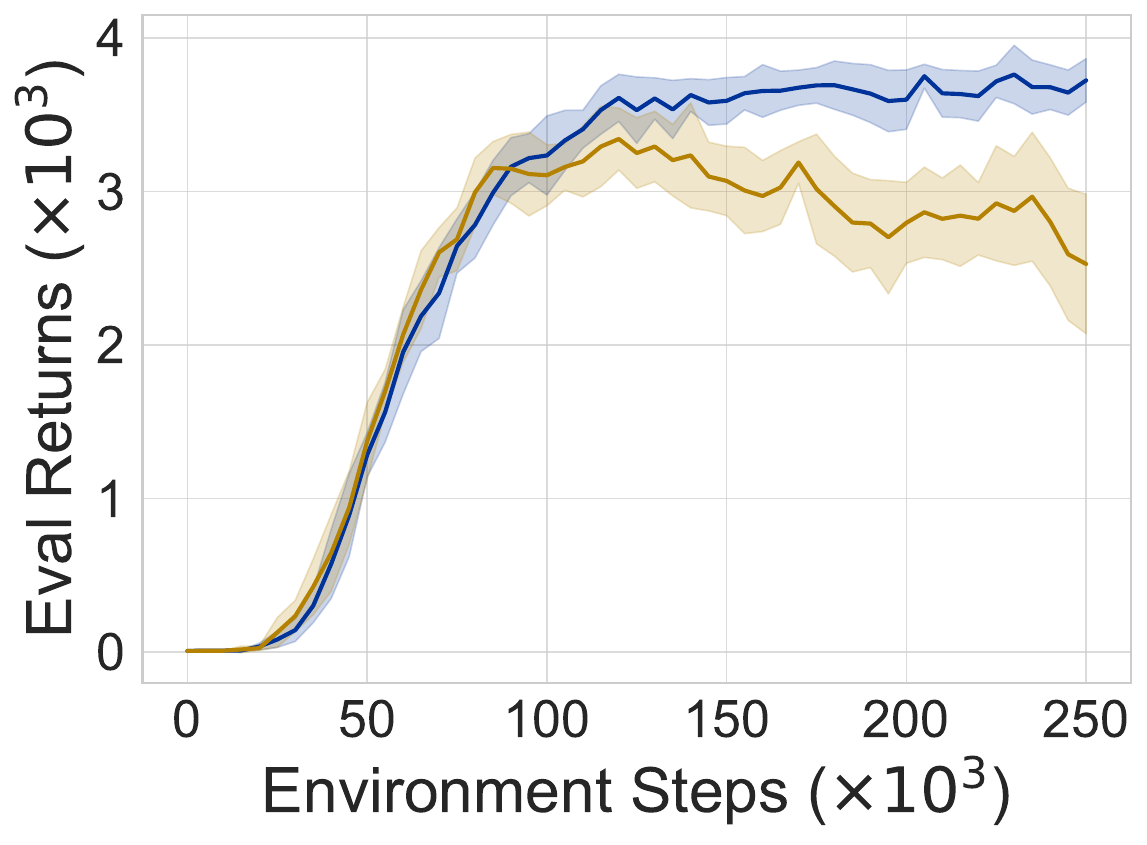}
        \caption{relocate-expert}
    \end{subfigure}\hfil
    \begin{subfigure}{.24\textwidth}
        \centering
        \includegraphics[width=\textwidth,  clip={0,0,0,0}]{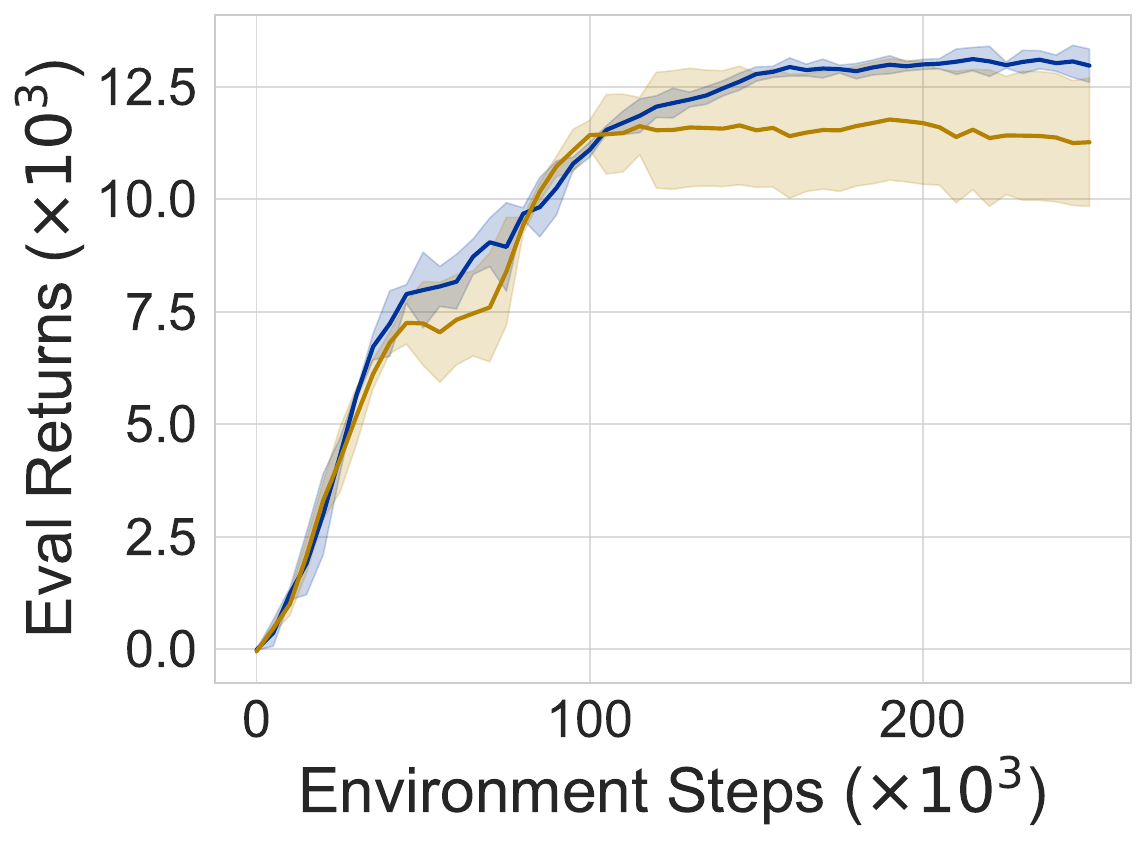}
        \caption{halfcheetah-medium}
    \end{subfigure}\hfil
    \begin{subfigure}{.24\textwidth}
        \centering
        \includegraphics[width=\textwidth,  clip={0,0,0,0}]{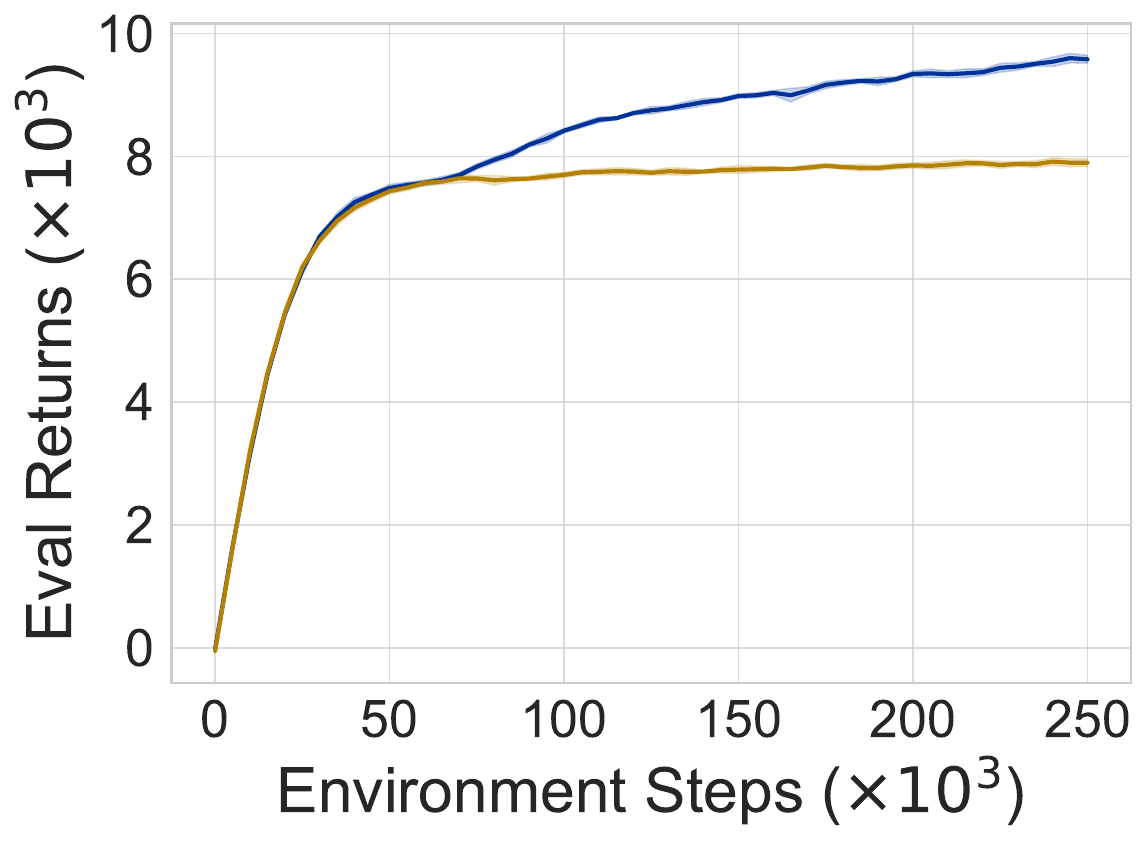}
        \caption{halfcheetah-simple}
    \end{subfigure}\hfil
    \begin{subfigure}{.24\textwidth}
        \centering
        \includegraphics[width=\textwidth,  clip={0,0,0,0}]{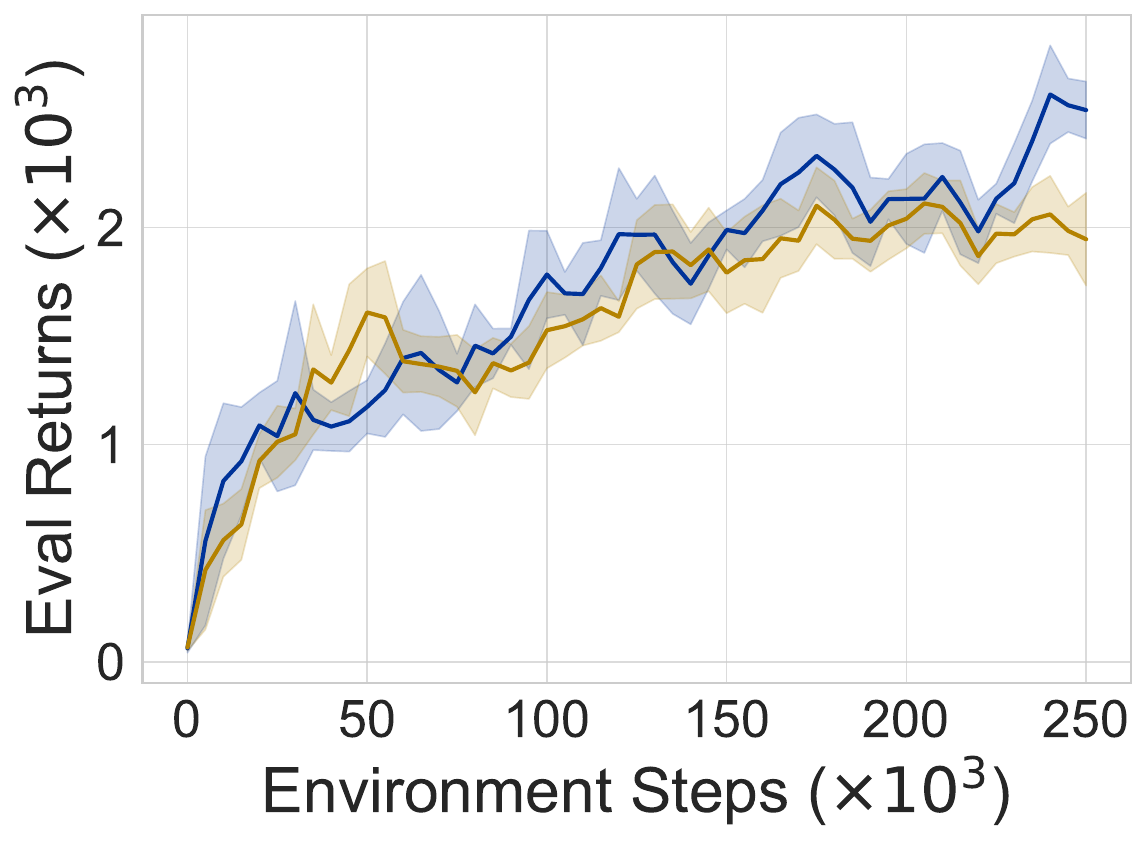}
        \caption{hopper-expert}
    \end{subfigure}\hfil
    \begin{subfigure}{.24\textwidth}
        \centering
        \includegraphics[width=\textwidth,  clip={0,0,0,0}]{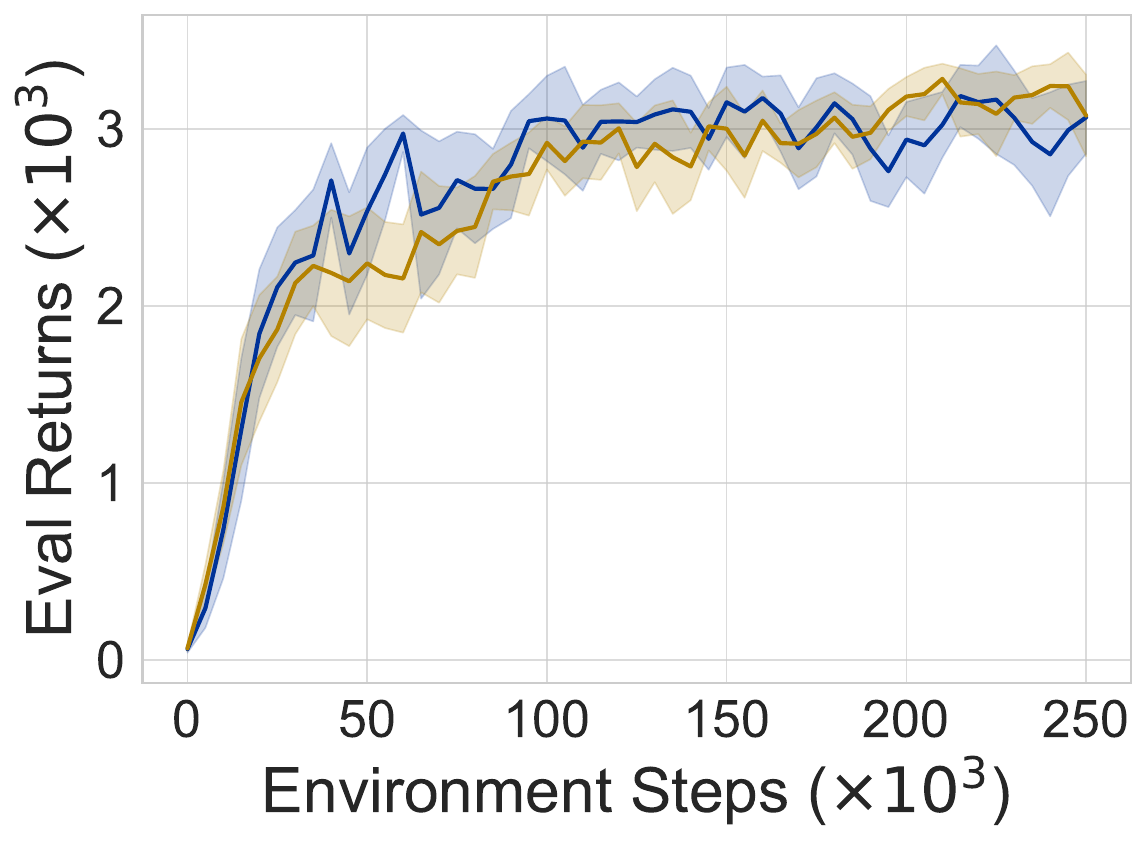}
        \caption{hopper-medium}
    \end{subfigure}\hfil
    \begin{subfigure}{.24\textwidth}
        \centering
        \includegraphics[width=\textwidth,  clip={0,0,0,0}]{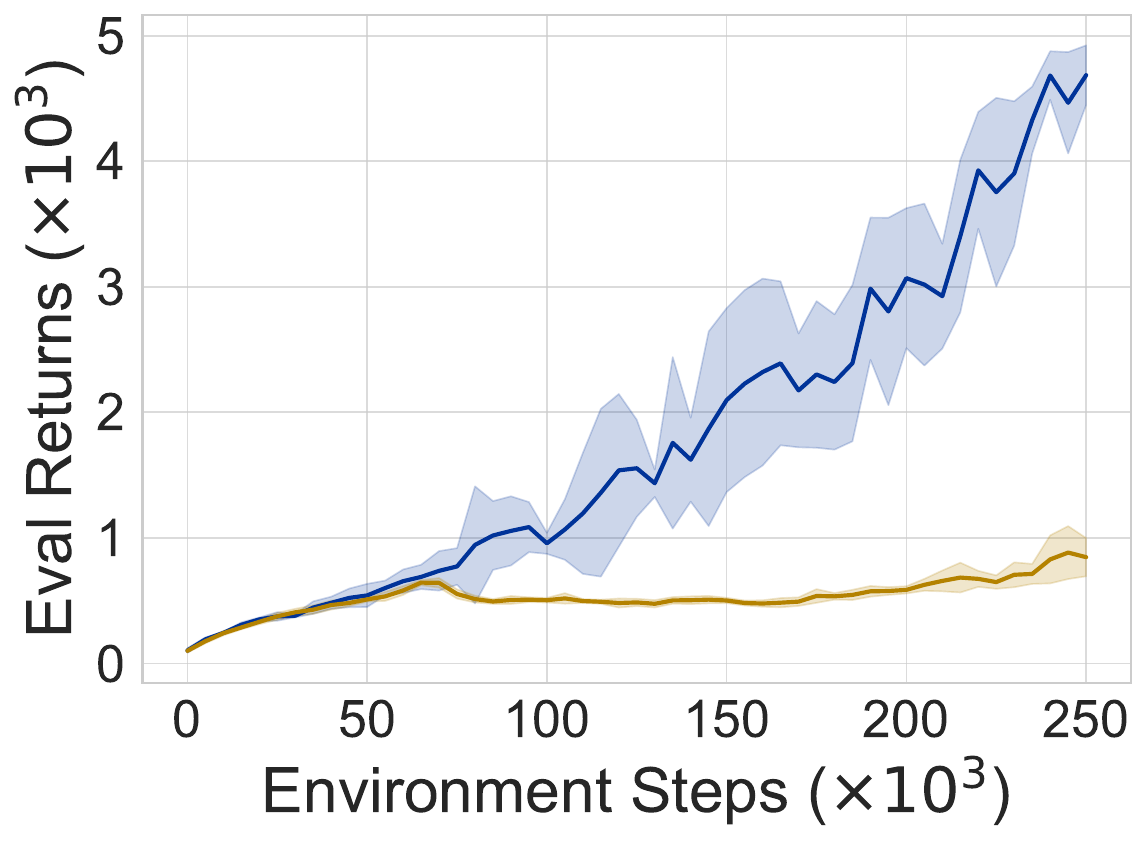}
        \caption{humanoid-expert}
    \end{subfigure}\hfil
    \begin{subfigure}{.24\textwidth}
        \centering
        \includegraphics[width=\textwidth,  clip={0,0,0,0}]{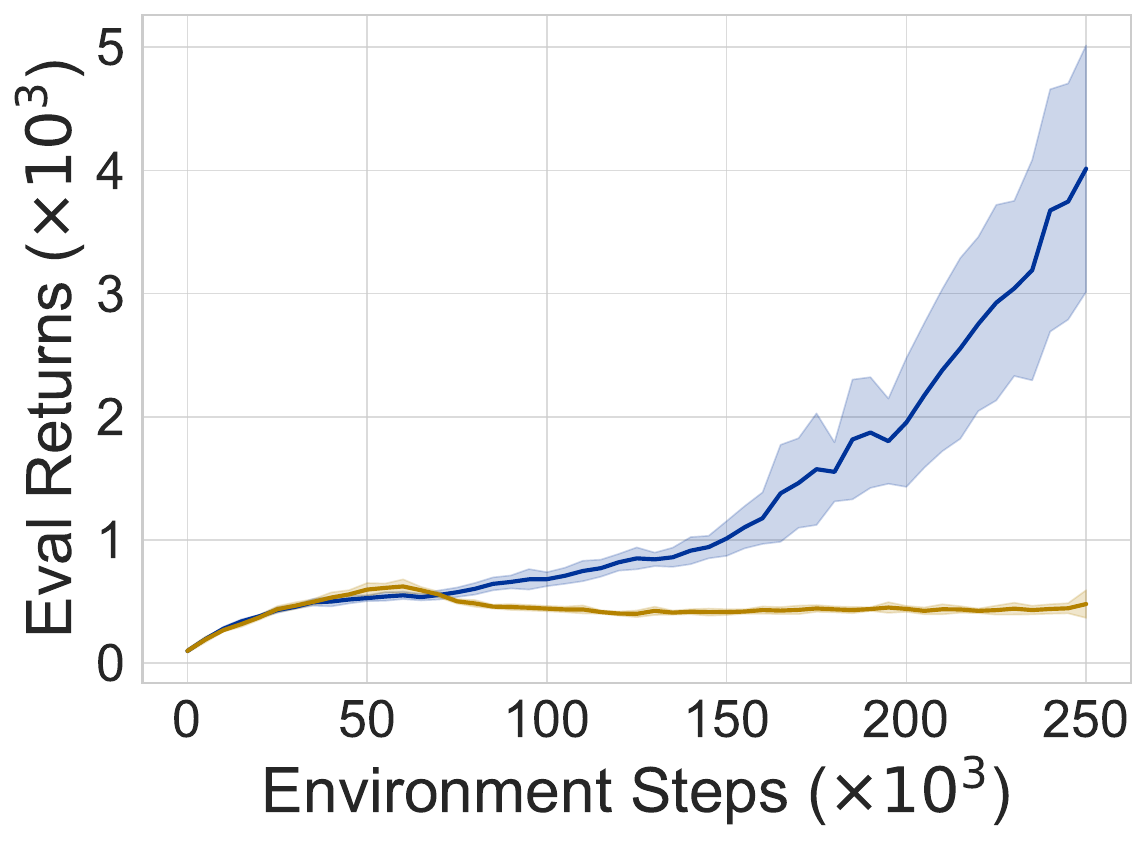}
        \caption{humanoid-medium}
    \end{subfigure}\hfil
    \begin{subfigure}{.24\textwidth}
        \centering
        \includegraphics[width=\textwidth,  clip={0,0,0,0}]{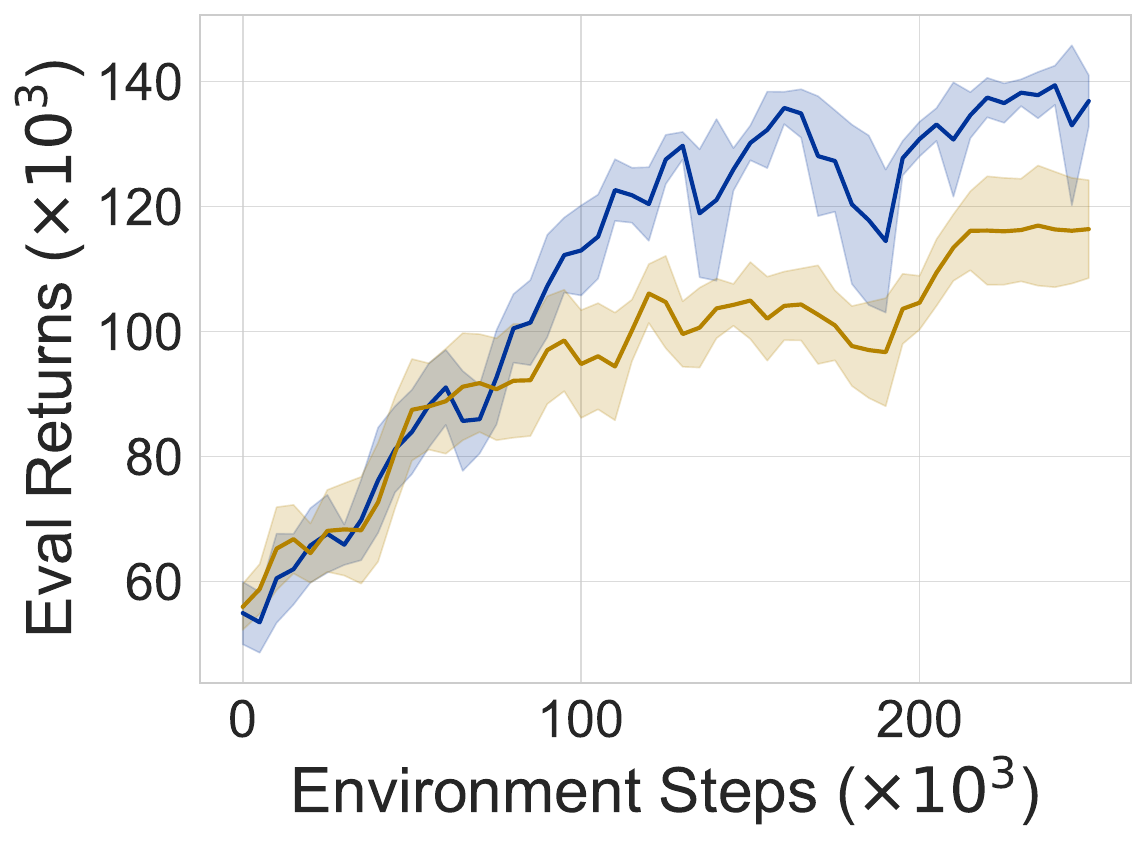}
        \caption{humanoidstandup-medium}
    \end{subfigure}\hfil

    \caption{Ablation on advantage, $\xi = 0$.}
\end{figure*}

\begin{figure*}[h]
    \centering
    \begin{subfigure}{.24\textwidth}
        \centering
        \includegraphics[width=\textwidth,  clip={0,0,0,0}]{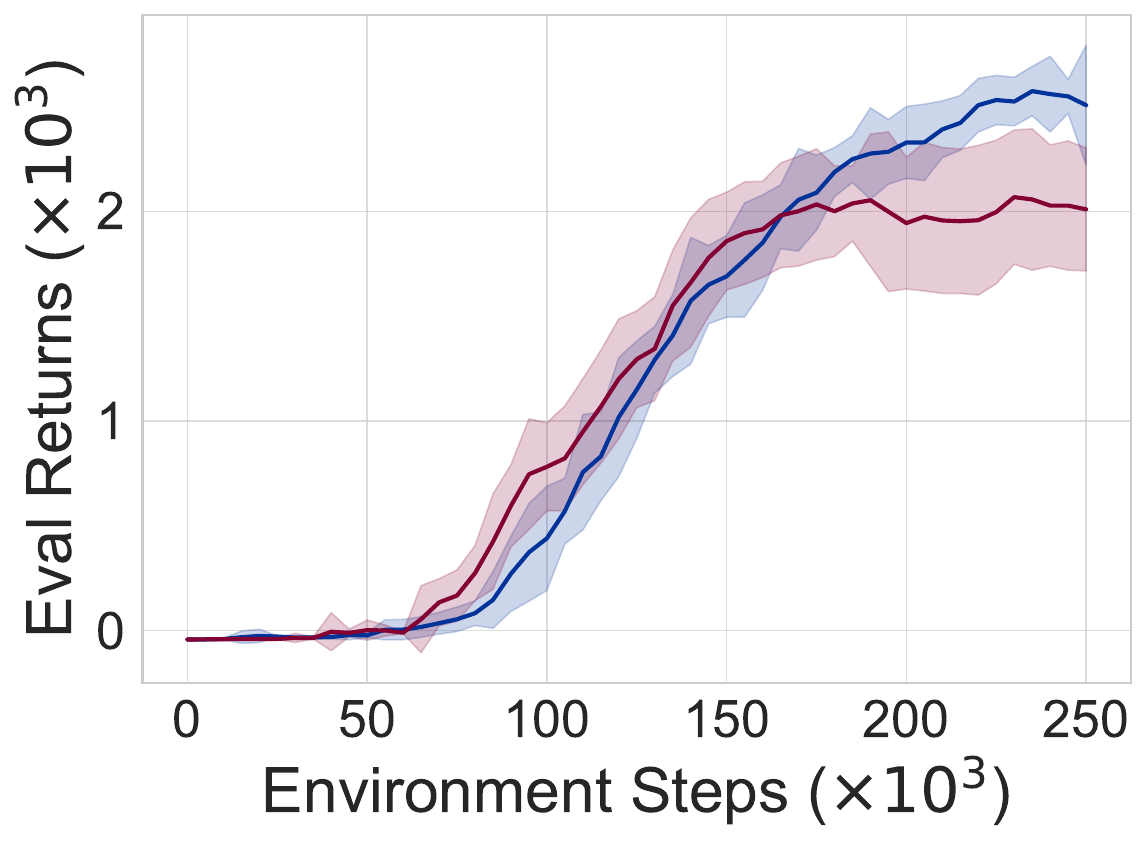}
        \caption{door-cloned}
    \end{subfigure}\hfil
    \begin{subfigure}{.24\textwidth}
        \centering
        \includegraphics[width=\textwidth,  clip={0,0,0,0}]{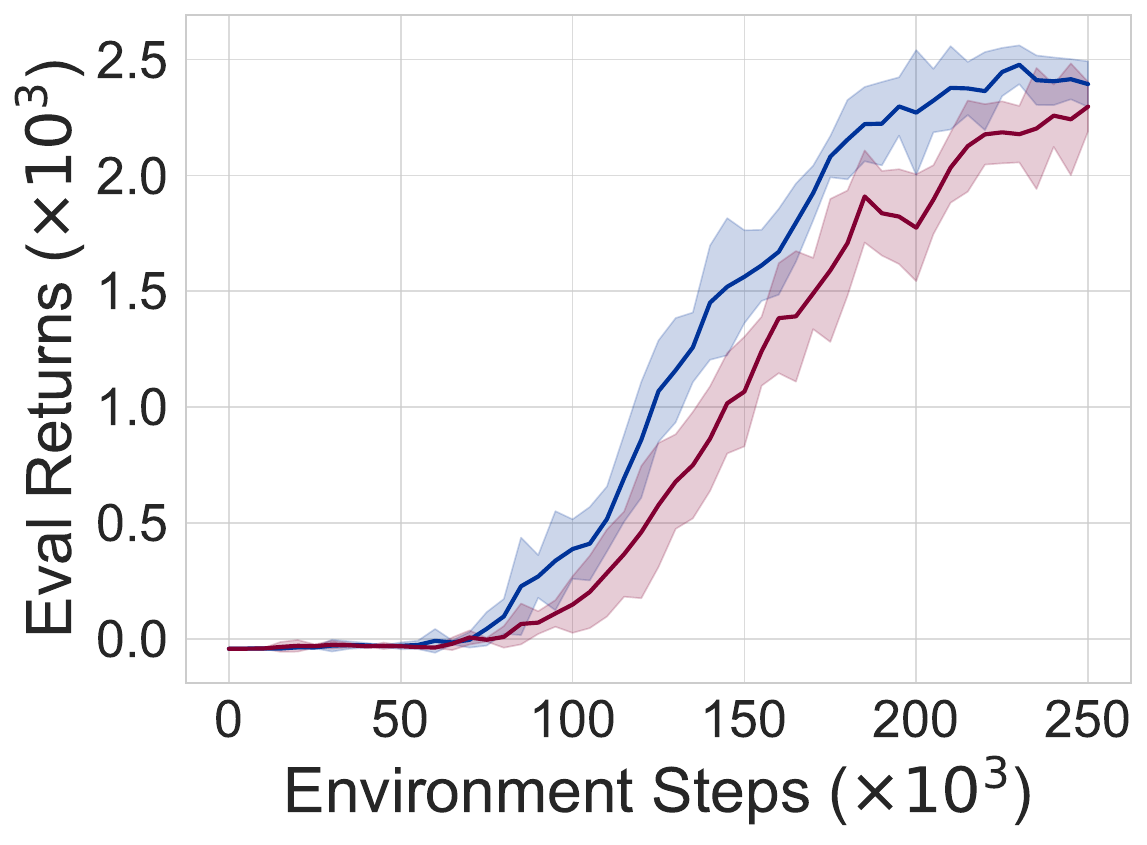}
        \caption{door-human}
    \end{subfigure}\hfil
    \begin{subfigure}{.24\textwidth}
        \centering
        \includegraphics[width=\textwidth,  clip={0,0,0,0}]{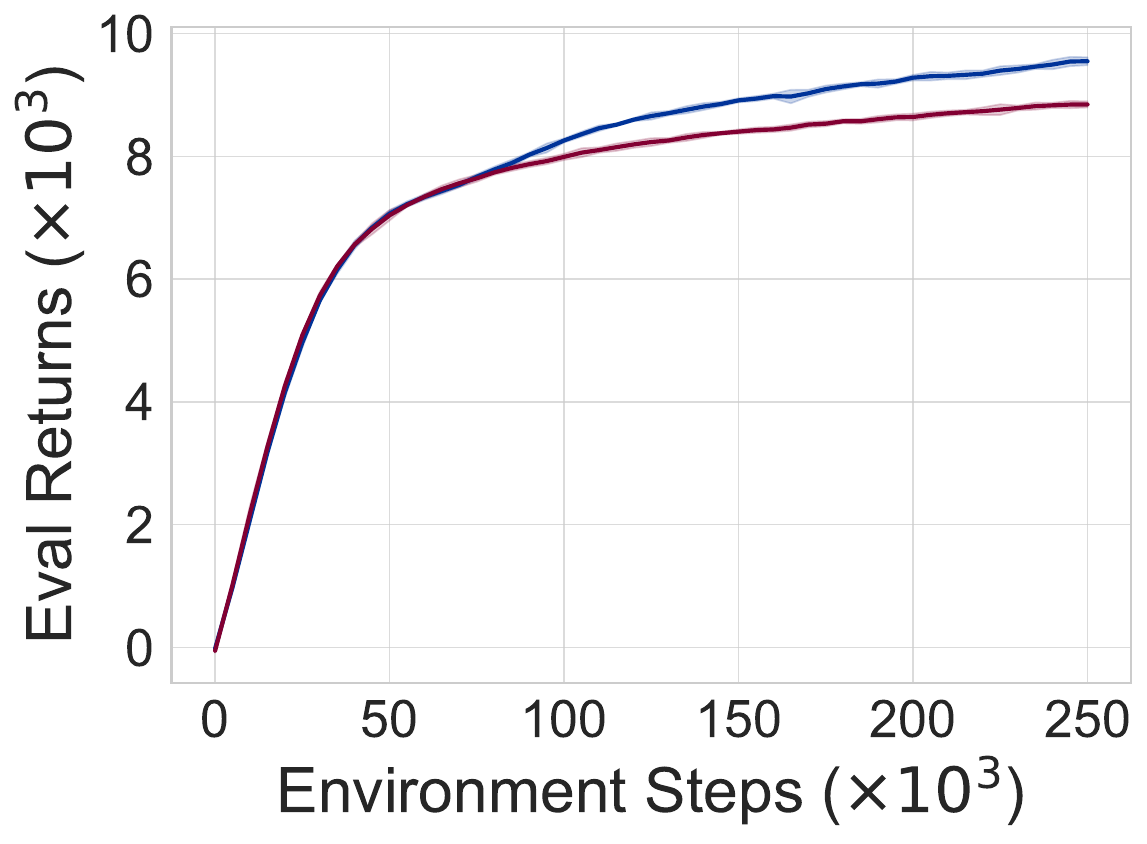}
        \caption{halfcheetah-simple}
    \end{subfigure}\hfil
    \begin{subfigure}{.24\textwidth}
        \centering
        \includegraphics[width=\textwidth,  clip={0,0,0,0}]{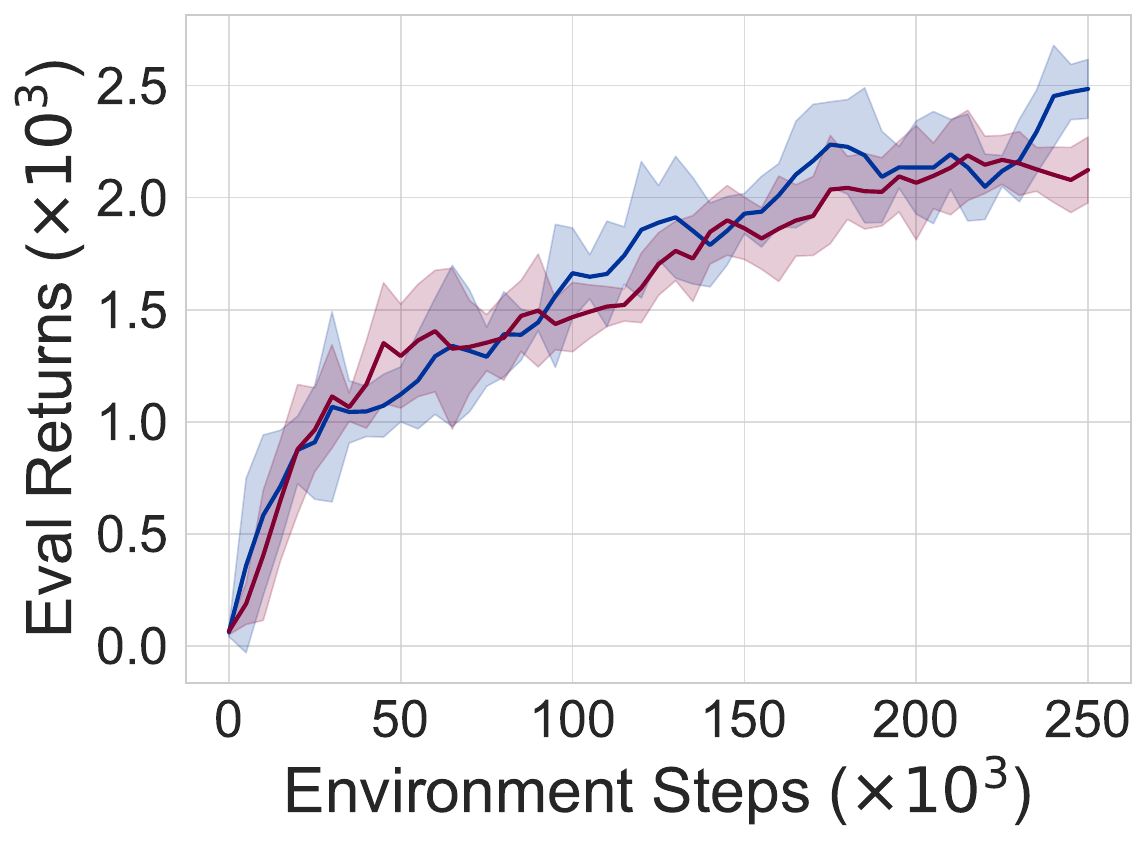}
        \caption{hopper-expert}
    \end{subfigure}\hfil

    \caption{Ablation on lower confidence bound, $\beta = 0$.}
\end{figure*}

\begin{figure*}[h]
    \centering
    \begin{subfigure}{.24\textwidth}
        \centering
        \includegraphics[width=\textwidth,  clip={0,0,0,0}]{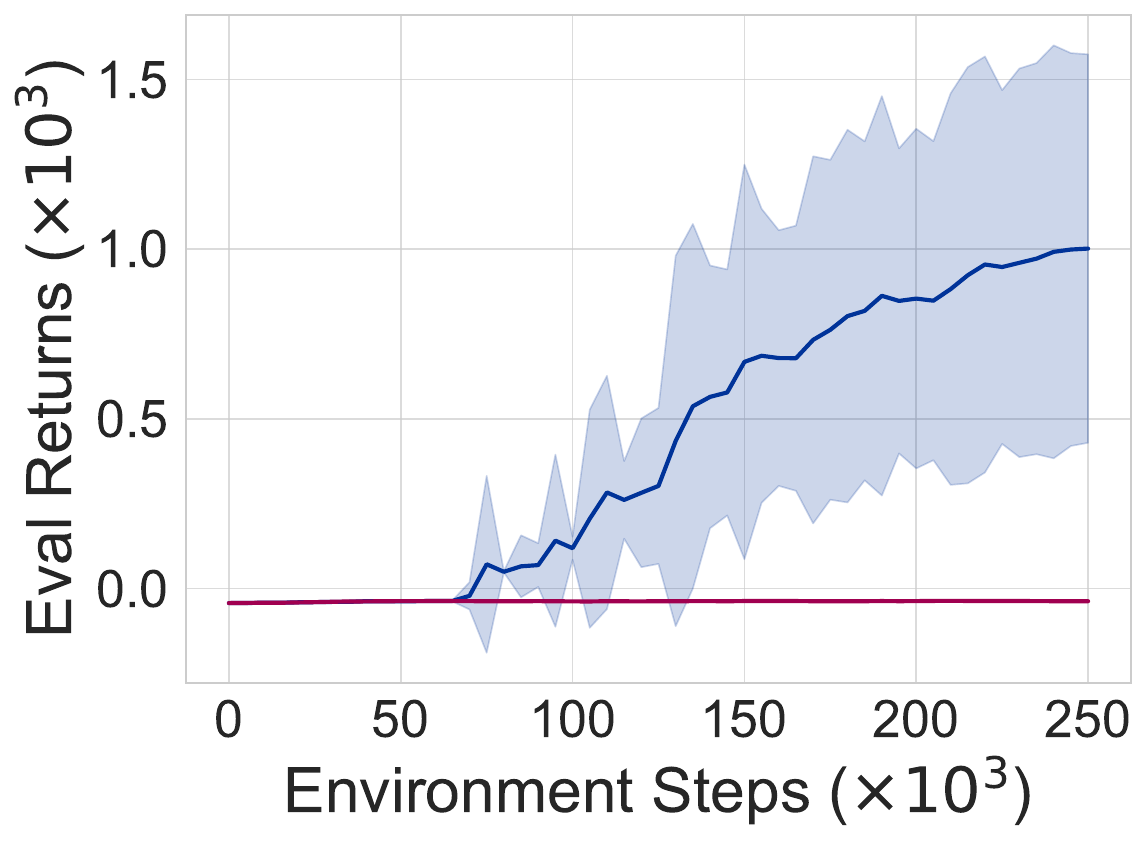}
        \caption{door-human}
    \end{subfigure}\hfil
    \begin{subfigure}{.24\textwidth}
        \centering
        \includegraphics[width=\textwidth,  clip={0,0,0,0}]{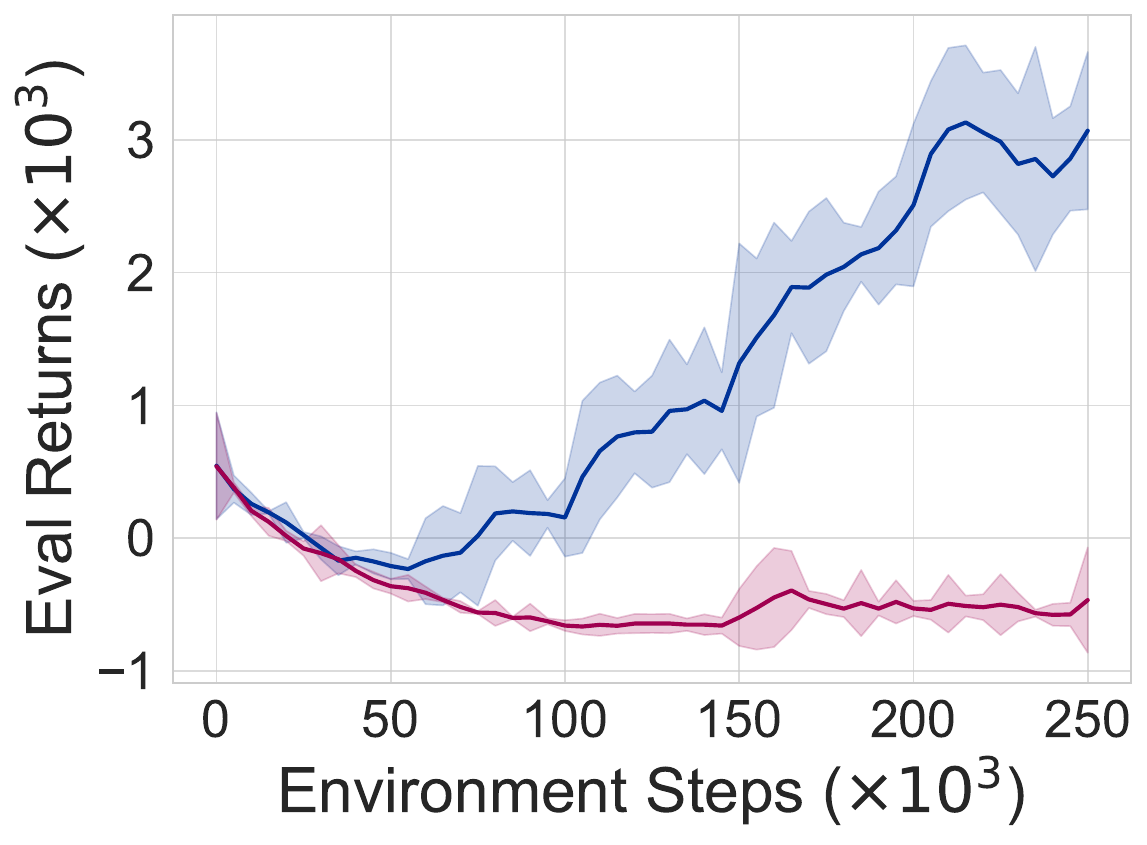}
        \caption{pen-cloned}
    \end{subfigure}\hfil
    \begin{subfigure}{.24\textwidth}
        \centering
        \includegraphics[width=\textwidth,  clip={0,0,0,0}]{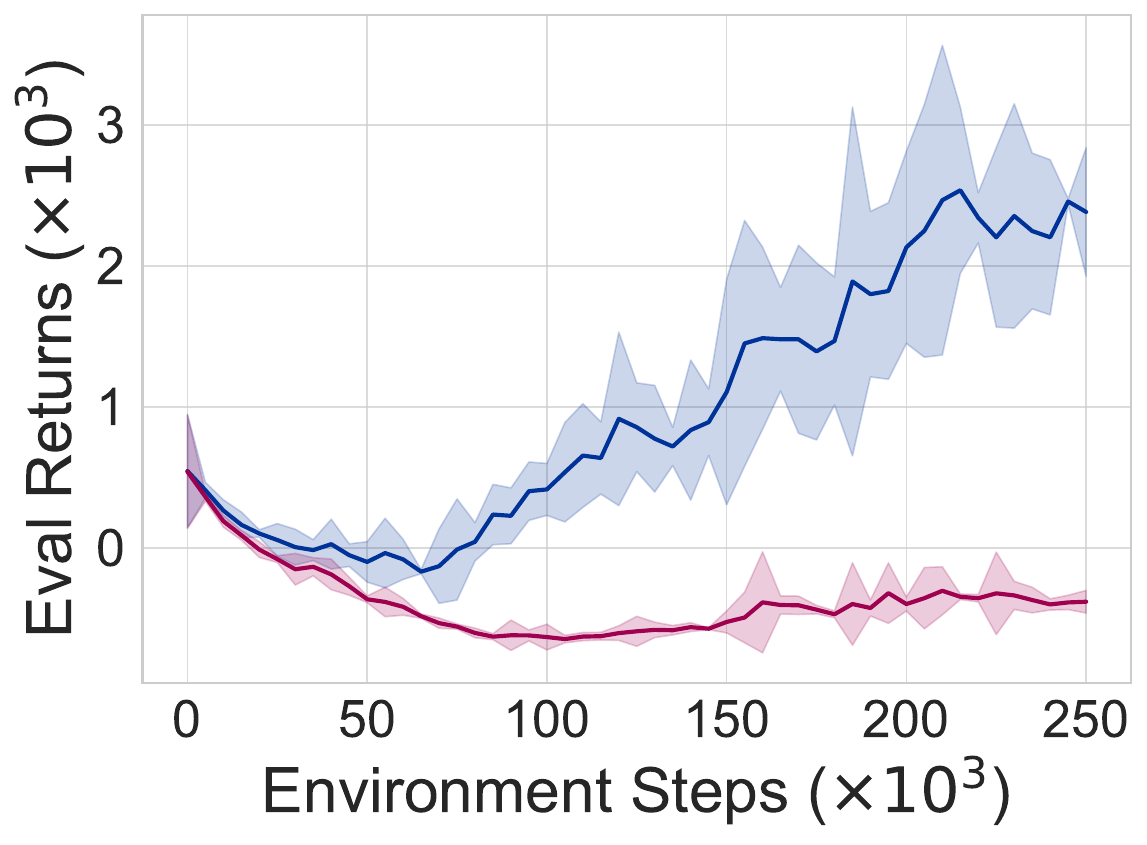}
        \caption{pen-expert}
    \end{subfigure}\hfil
    \begin{subfigure}{.24\textwidth}
        \centering
        \includegraphics[width=\textwidth,  clip={0,0,0,0}]{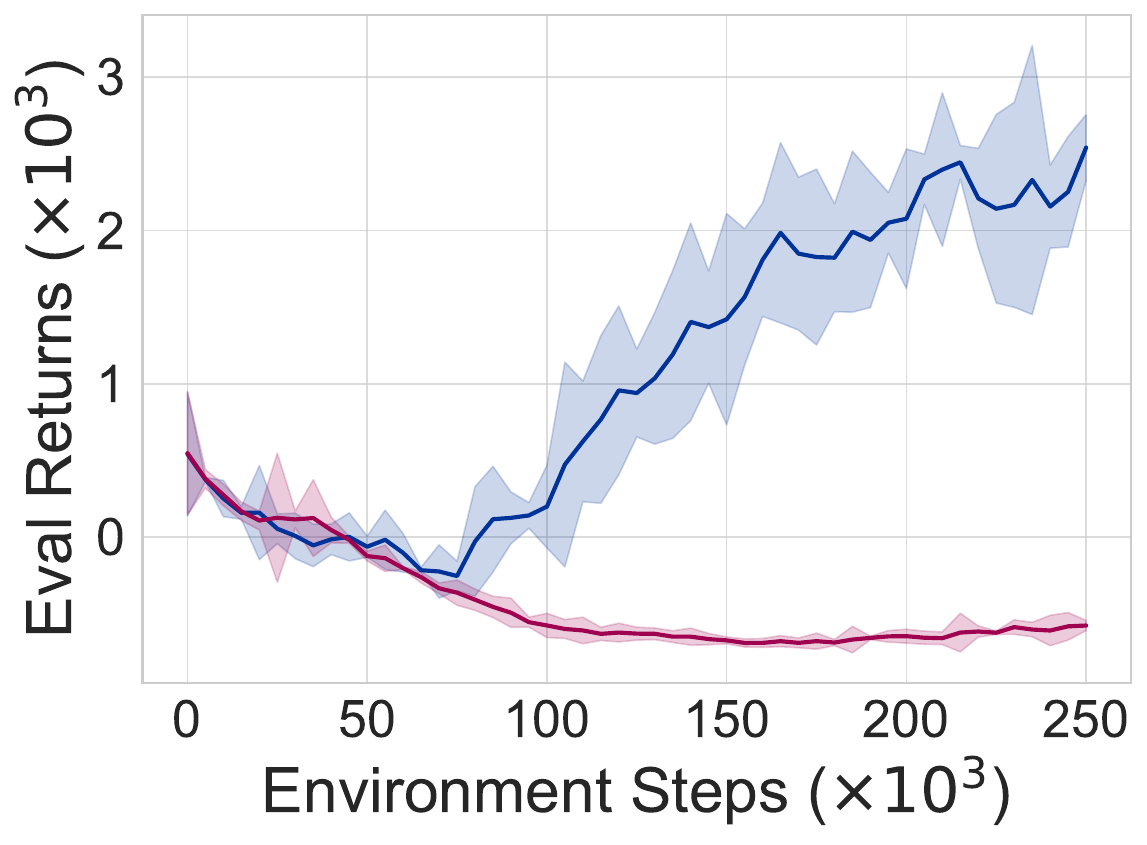}
        \caption{pen-human}
    \end{subfigure}\hfil
    \begin{subfigure}{.24\textwidth}
        \centering
        \includegraphics[width=\textwidth,  clip={0,0,0,0}]{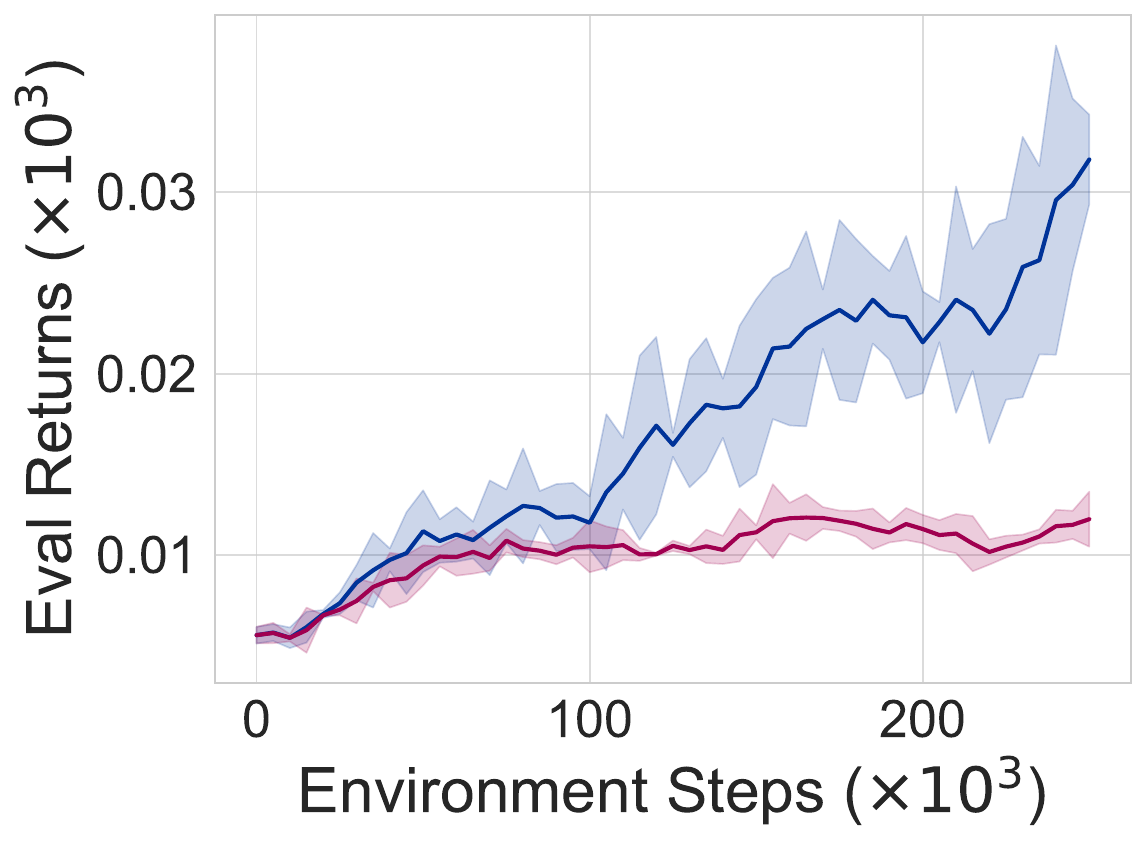}
        \caption{relocate-expert}
    \end{subfigure}\hfil
    \begin{subfigure}{.24\textwidth}
        \centering
        \includegraphics[width=\textwidth,  clip={0,0,0,0}]{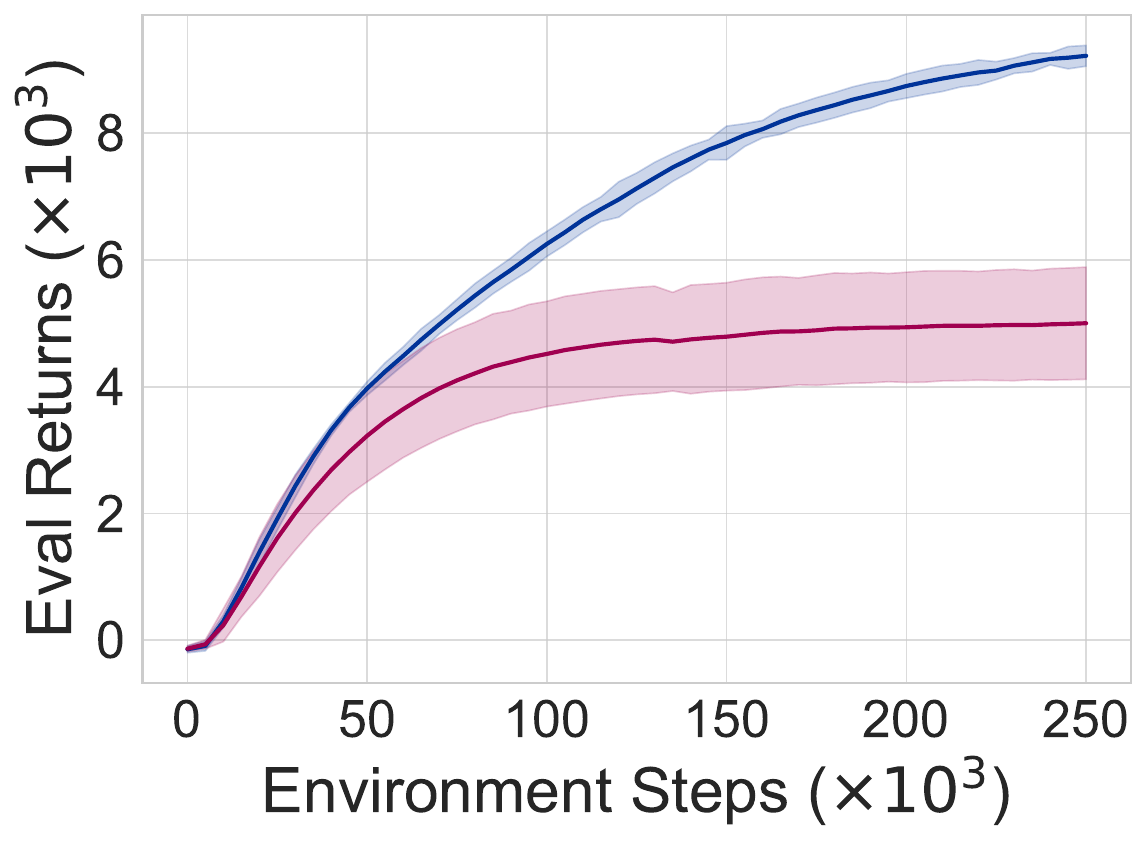}
        \caption{halfcheetah-simple}
    \end{subfigure}\hfil
    \begin{subfigure}{.24\textwidth}
        \centering
        \includegraphics[width=\textwidth,  clip={0,0,0,0}]{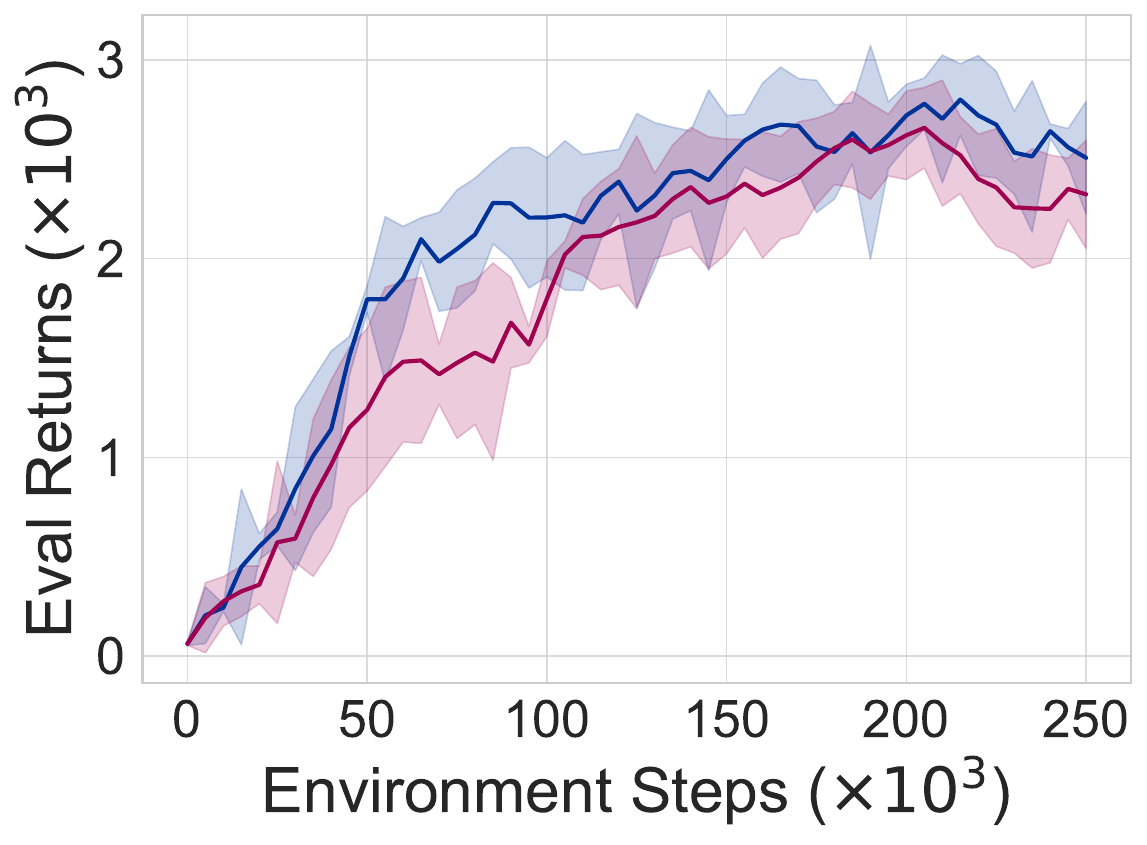}
        \caption{hopper-medium}
    \end{subfigure}\hfil
    \begin{subfigure}{.24\textwidth}
        \centering
        \includegraphics[width=\textwidth,  clip={0,0,0,0}]{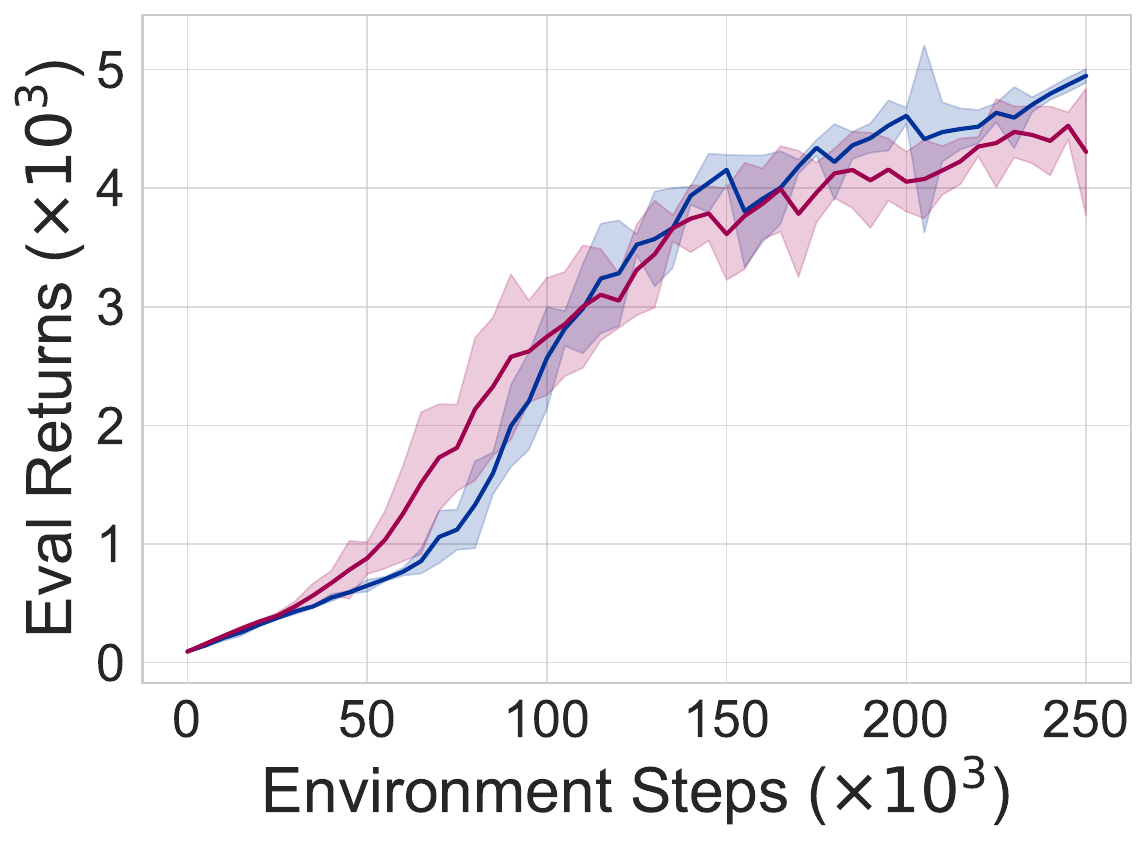}
        \caption{humanoid-medium}
    \end{subfigure}\hfil

    \caption{Ablation on online setting.}
\end{figure*}

\begin{figure*}[h]
    \centering
    \begin{subfigure}{.24\textwidth}
        \centering
        \includegraphics[width=\textwidth,  clip={0,0,0,0}]{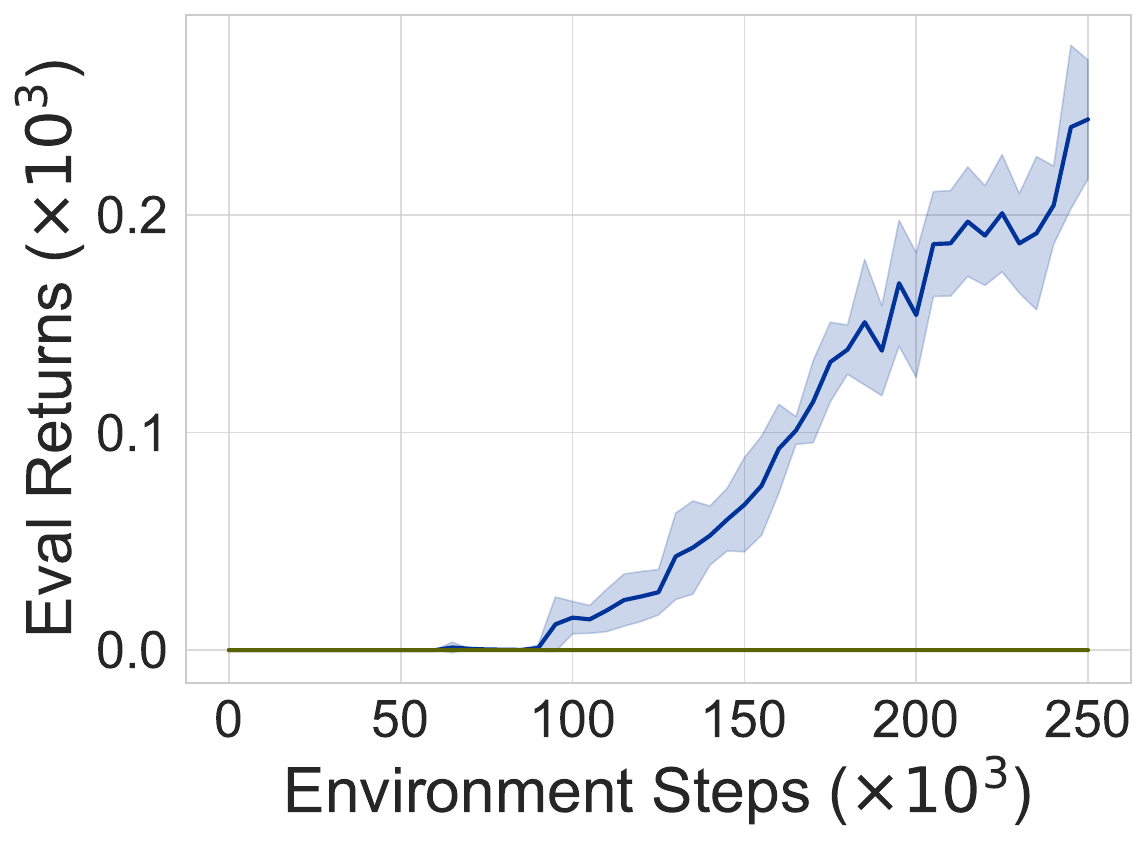}
        \caption{antmaze-umaze}
    \end{subfigure}\hfil
    \begin{subfigure}{.24\textwidth}
        \centering
        \includegraphics[width=\textwidth,  clip={0,0,0,0}]{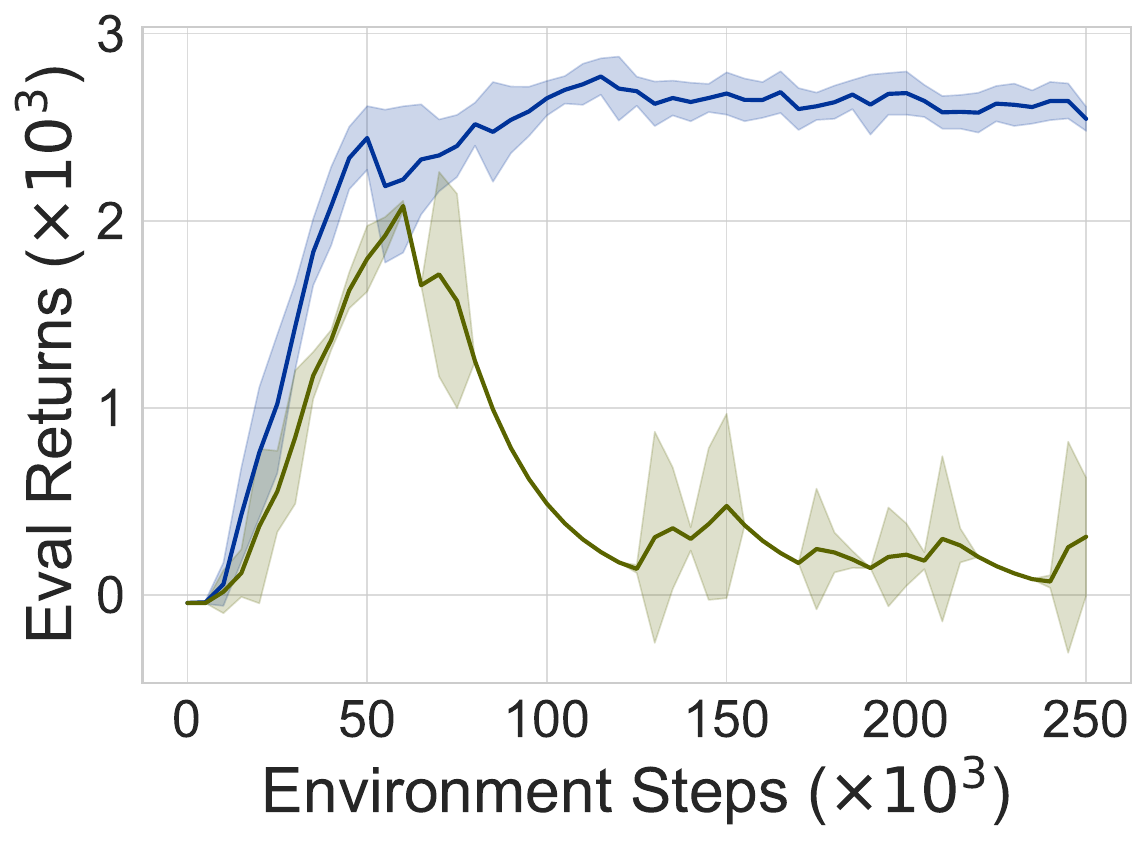}
        \caption{door-expert}
    \end{subfigure}\hfil
    \begin{subfigure}{.24\textwidth}
        \centering
        \includegraphics[width=\textwidth,  clip={0,0,0,0}]{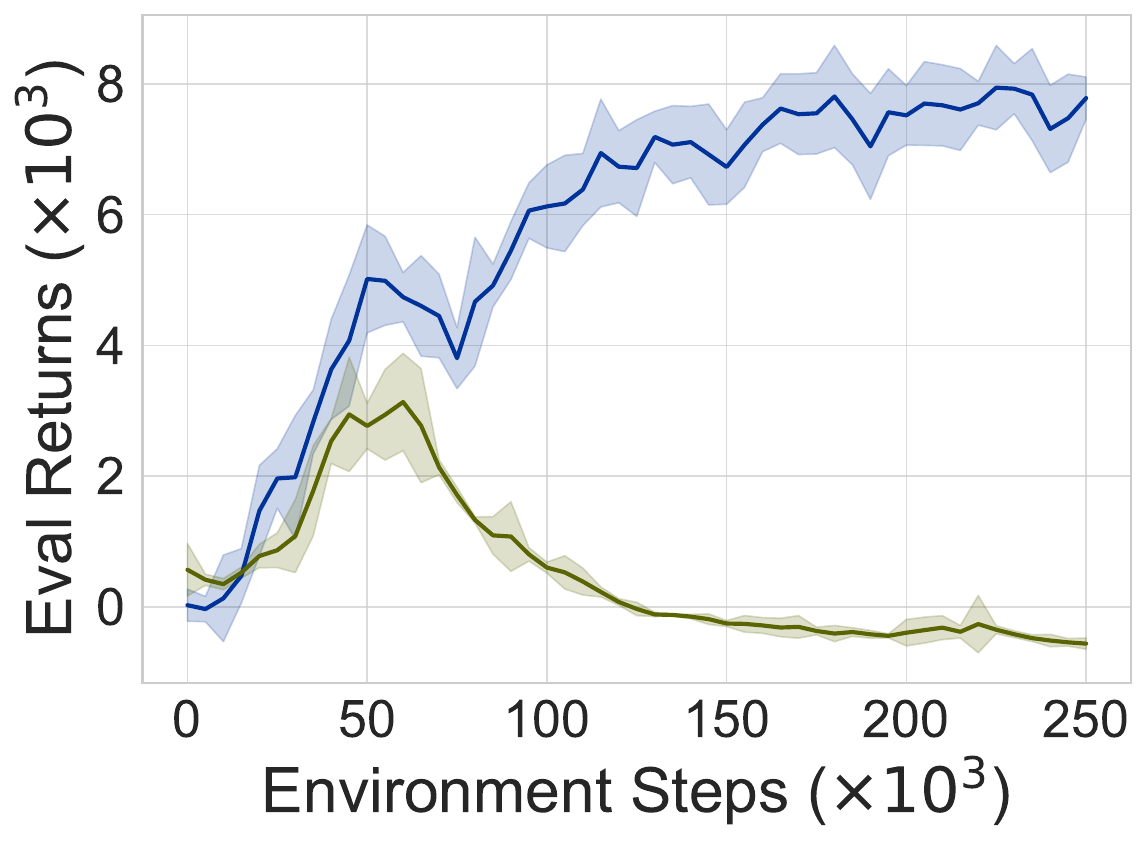}
        \caption{pen-expert}
    \end{subfigure}\hfil
    \begin{subfigure}{.24\textwidth}
        \centering
        \includegraphics[width=\textwidth,  clip={0,0,0,0}]{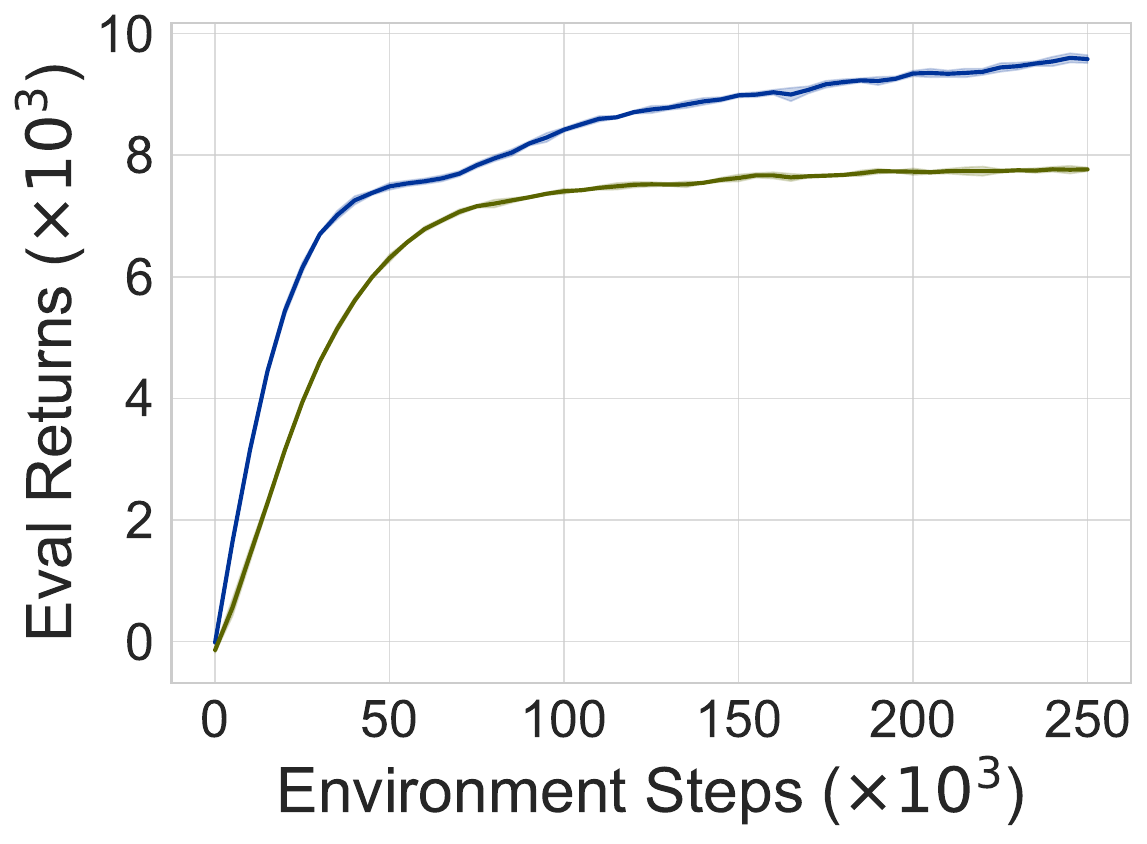}
        \caption{halfcheetah-simple}
    \end{subfigure}\hfil
    \begin{subfigure}{.24\textwidth}
        \centering
        \includegraphics[width=\textwidth,  clip={0,0,0,0}]{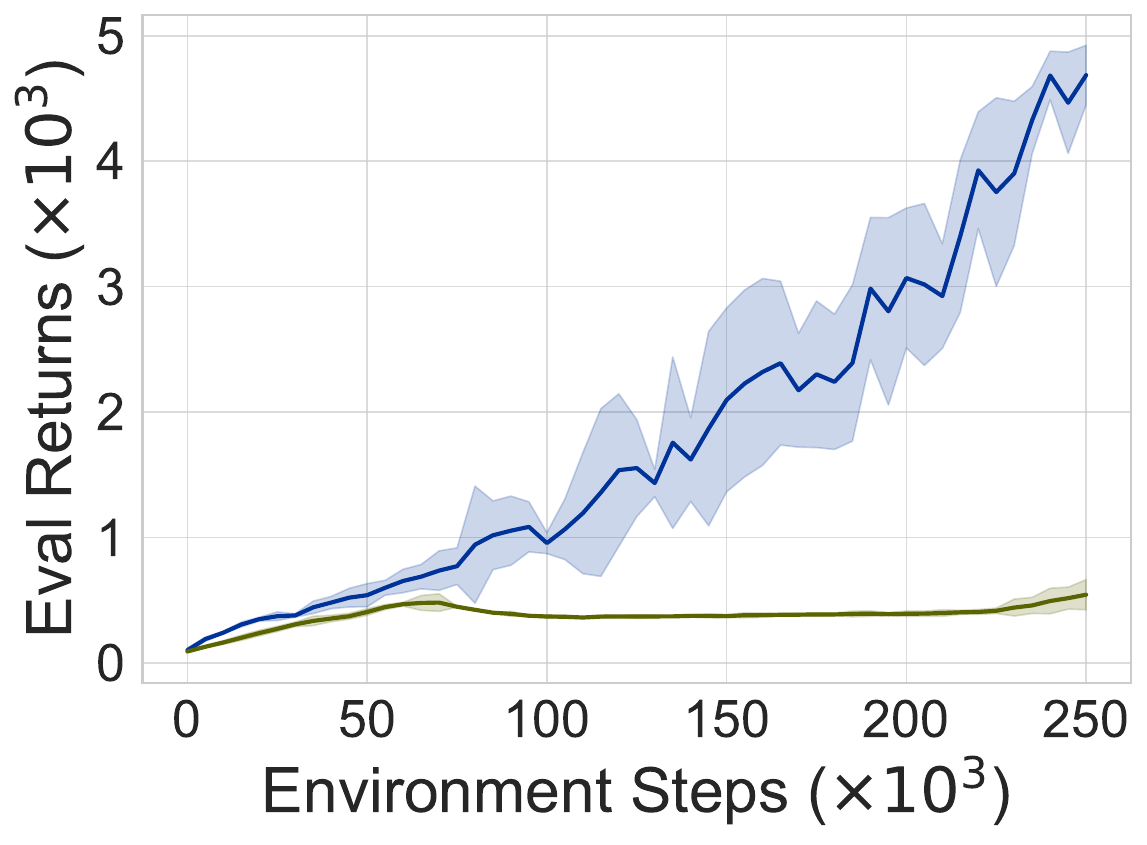}
        \caption{humanoid-expert}
    \end{subfigure}\hfil
    \begin{subfigure}{.24\textwidth}
        \centering
        \includegraphics[width=\textwidth,  clip={0,0,0,0}]{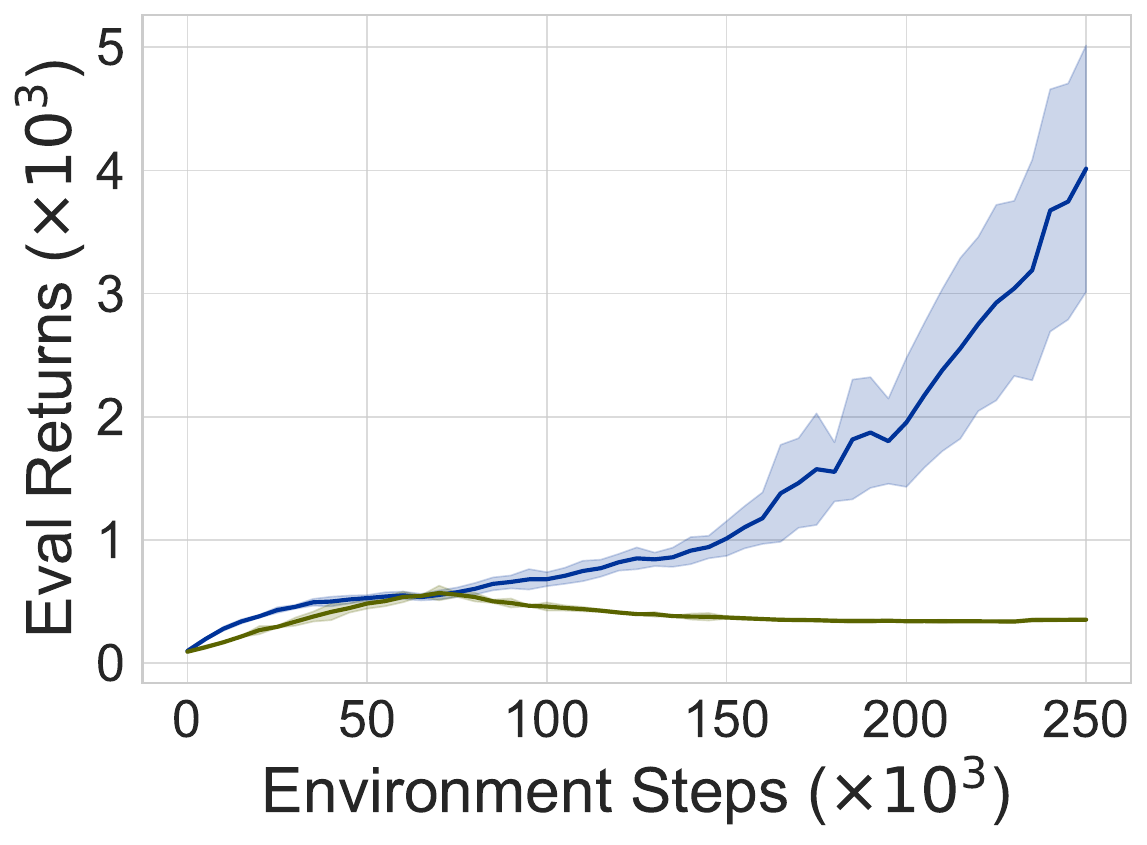}
        \caption{humanoid-medium}
    \end{subfigure}\hfil
    \caption{Ablation against TD error as priority.}
\end{figure*}

\subsection{Environments and datasets.}

\paragraph{Environments.} \figref{fig:exp:env:1} and \figref{fig:exp:env:2} present the tested D4RL tasks: the locomotion ones include halfcheetah, walker2d, ant, antmaze, hopper, humanoid; while the Adroit hand ones include pen, relocate, hammer and door. antmaze has sparse rewards, while all other environments have dense reward. All environments are equipped with continuous state and action spaces.

In the halfcheetah environment, the 2D agent resembles a simplified cheetah model with a torso and lined legs, with the objective of forward locomotion and maintaining balance while maximizing speed. In the walker2d environment, the 2D humanoid agent has 2 legs and multiple joints, with the objective of stable bipedal walking without falling. In the ant environment, the agent is a 3D quadrupedal agent with multiple joints and degrees of freedom, with the objective of moving forward efficiently while maintaining balance. humanoid and hopper are similar. For all of these environments, rewards are given for velocity to encourage the agent to move forward efficiently while maintaining balance, and several offline datasets, per \citep{fu2020d4rl}, with varying characteristics, as detailed below, were tested.

\figref{fig:exp:env:1} (right most) presents a snapshot of relocate. This environment involve a simulated 28-DoF robotic arm interacting with objects in a 3D space and are characterized by sparse rewards and continuous state and action spaces. door, hammer, pen are similar high-dimensional robotic arm manipulation tasks.

In the relocate environment, the arm must pick up a ball and move it to a target position, requiring coordinated grasping and relocation of an object in 3D space. 
For all of these environments, rewards are sparse and typically only given upon task completion, increasing the exploration difficulty.

In the antmaze environment, the aforementioned ant agent is placed in a maze environment and must navigate from a defined start point to a goal. Rewards are binary: $1$ for reaching the goal and $0$ otherwise.

\paragraph{Datasets.} We use offline datasets as provided by Minari \citep{minari}. For each of the environments above, depending on the type of environment (whether it is under D4RL (Adroit) or mujoco (locomotion)), there are several types of offline datasets. These are detailed in \autoref{tab:env_description:1} and \autoref{tab:env_description:2}. 
\begin{table}[H]
    \centering
    \begin{tabular}{l|c}
    \hline
        {\bf Offline dataset type} & {\bf Description} \\
        \hline
        simple & 1M steps of agent trained for 1M steps \\
        medium & 1M steps of agent trained for 4M steps \\
        expert & 1M steps of agent trained for 25M steps \\
        \hline
    \end{tabular}
    \caption{Offline dataset descriptions for hopper.}
    \label{tab:env_description:1}
\end{table}

\begin{table}[H]
    \centering
    \begin{tabular}{l|c}
    \hline
        {\bf Offline dataset type} & {\bf Description} \\
        \hline
        human & 5K steps of human demonstrations \\
        expert & 500K steps of fine-tuned RL policy \\
        cloned & 500K steps of imitation learning agent using human and expert\\
        \hline
    \end{tabular}
    \caption{Offline dataset descriptions for pen.}
    \label{tab:env_description:2}
\end{table}

The exact number of steps may vary per environment, but hopper and pen are representative of the locomotion and Adroit tasks respectively. We refer to \citet{minari} for more granular details for each environment.

\begin{figure}[ht]
    \centering
    \includegraphics[height=0.8in]{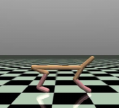} \hfil
    \includegraphics[height=0.8in]{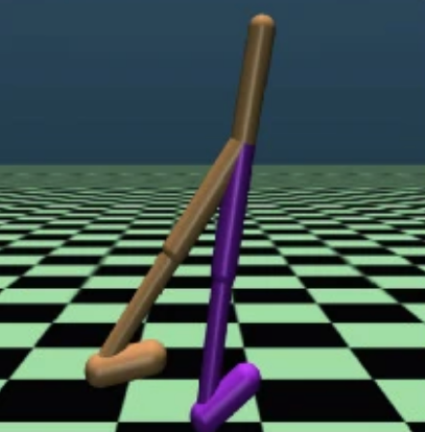}\hfil
    \includegraphics[height=0.8in]{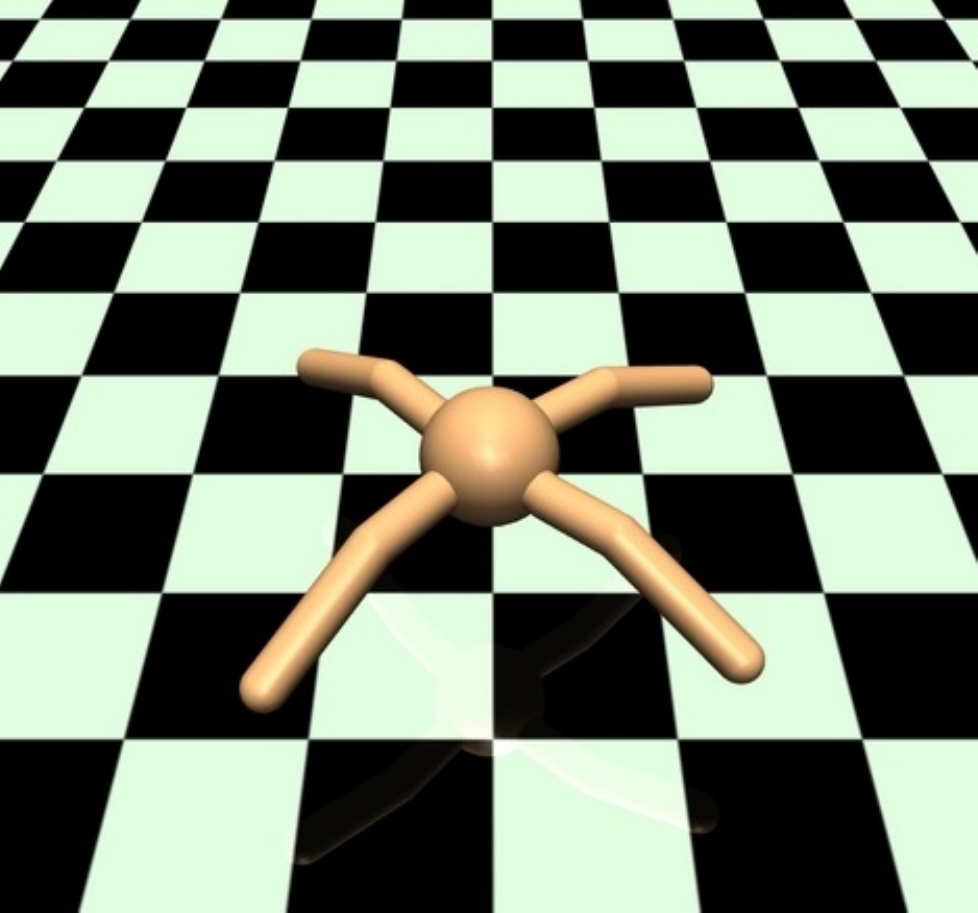}\hfil
    \includegraphics[height=0.8in]{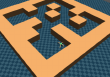}\hfil
     \includegraphics[height=0.8in]{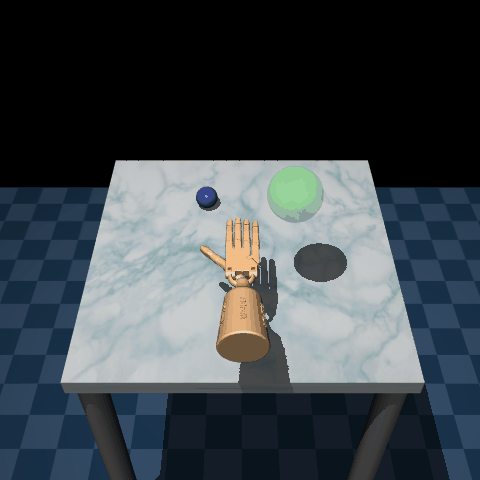} \hfil
    \caption{
    \textbf{Environments}: halfcheetah, walker2d,  ant, antmaze and relocate. 
}\label{fig:exp:env:1}
\end{figure}

\begin{figure}[ht]
    \centering
    \includegraphics[height=0.8in]{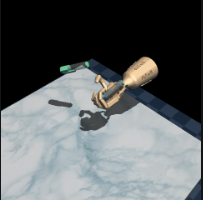} \hfil
    \includegraphics[height=0.8in]{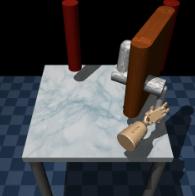}\hfil
    \includegraphics[height=0.8in]{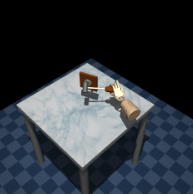}\hfil
    \includegraphics[height=0.8in]{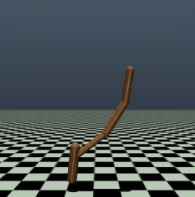}\hfil
     \includegraphics[height=0.8in]{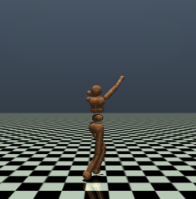} \hfil
    \caption{
    \textbf{Environments}: pen, door, hammer, hopper and humanoid. 
}\label{fig:exp:env:2}
\end{figure}

\subsection{Computing Infrastructure and Wall-time Comparison.}\label{app:computing_infrastructure}
We performed our experiments on a cluster that includes CPU nodes (approximately 280 cores) and GPU nodes (approximately 110 NVIDIA GPUs, ranging from Titan X to A6000, set up mostly in 4- and 8-GPU configurations). On the same cluster, the wall run time of \algname is approximately 1.25 times the runtime of regular RLPD, and is about half the runtime of PEX and BOORL (offline pretraining included). This highlights the superior sample and compute efficiency of \algname. 

\subsection{Hyperparameters}
We list the hyperparameters used for \algname in \tabref{tab:hyperparams}.

\begin{table}[H]
    \centering
    \begin{tabular}{l|c}
    \hline
        {\bf Parameter} & {\bf Value} \\
        \hline
        Batch size & $256$ \\
        Gradient steps $G$ & $10$ \\
        MLP Architecture & 2-Layer, with LayerNorm \\
        Network width & $256$ Units \\
        Discount & $0.99$ \\
        Learning rate & $3 \times 10^{-4}$ \\
        Ensemble size $E$ & $10$ \\
        $\densityTemp$ & $0.2$ \\
        $\advTemperature$ & $1$ \\
        $\beta$ & $0.2$ \\
        $\beta_0$ & $0.4$ \\
        Optimizer & Adam \\
        \hline
    \end{tabular}
    \caption{\algname hyperparameters.}
    \label{tab:hyperparams}
\end{table}